\documentclass{article}

\PassOptionsToPackage{numbers, compress}{natbib}

\usepackage[preprint]{neurips_2019}
\usepackage[utf8]{inputenc} % allow utf-8 input
\usepackage[T1]{fontenc}    % use 8-bit T1 fonts
\usepackage{hyperref}       % hyperlinks
\usepackage{url}            % simple URL typesetting

\usepackage{booktabs}       % professional-quality tables
\usepackage{amsfonts}       % blackboard math symbols
\usepackage{nicefrac}       % compact symbols for 1/2, etc.
\usepackage{microtype}      % microtypography
\usepackage{color,amsthm,amsmath,amsfonts,xspace,caption, graphicx, mathrsfs, bbm}
\newcommand{\cX}{\mathcal{X}}
\newcommand{\Brace}[1]{\left\{#1\right\}}
\newcommand{\Abs}[1]{\left|#1\right|}
\newcommand{\indic}{\mathbbm{1}}
\newcommand{\EE}{\mathbb{E}}
\DeclareMathOperator*{\Exp}{\EE}
\DeclareMathOperator*{\poly}{poly}
\DeclareMathOperator*{\polylog}{polylog}
\newcommand{\Paren}[1]{\left(#1\right)}
\newcommand{\iid}{\text{i.i.d.}}
\newcommand{\cP}{\mathcal{P}}
\newtheorem{Theorem}{Theorem}
\newtheorem{Lemma}{Lemma}
\newtheorem{Corollary}{Corollary}
\newcommand{\Poi}{\mathrm{Poi}}
\newcommand{\Bin}{\mathrm{bin}}
\newcommand{\Var}{\mathrm{Var}}

\usepackage{multicol}

\newcommand{\norm}[1]{\left\lVert#1\right\rVert}

\newcommand{\size}{k}

\newcommand{\addI}{The rest of the paper is organized as follows. Section~\ref{sec:prior} and Section~\ref{sec:exp} illustrate PML's theoretical and practical advantages by comparing it to existing
methods for a variety of learning tasks. With the exception of Section~\ref{sec:TPML}, which proposes TPML and establishes four new results, Section~\ref{sec:Lipproof} to~\ref{sec:uniproof} present the proofs of Theorem~\ref{thm:est} to~\ref{thm:test}. Section~\ref{sec:conclusion} concludes the paper and outlines future directions.}
\newcommand{\addIII}{Section~\ref{sec:Wassproof}, }

\usepackage{xcolor}
\hypersetup{
    colorlinks,
    linkcolor={green!50!black},
    citecolor={blue!50!black},
}

\title{
\mbox{The
Broad
Optimality of Profile Maximum Likelihood}}

\author{
Yi Hao\\
University of California, San Diego\\
\texttt{yih179@eng.ucsd.edu}\\
\And
Alon Orlitsky\\
University of California, San Diego\\
\texttt{alon@ucsd.edu}\\
}

\begin{document}

\maketitle

\begin{abstract}

We study three fundamental statistical-learning problems: distribution estimation, property estimation, and property testing. We establish the profile maximum likelihood (PML) estimator as the first unified sample-optimal approach to a wide range of learning tasks. 
In particular, for every alphabet size $k$ and desired~accuracy~$\varepsilon$: 

\textbf{Distribution estimation} Under $\ell_1$ distance, PML yields optimal $\Theta(k/(\varepsilon^2\log k))$ sample complexity for sorted-distribution estimation, and a PML-based estimator empirically outperforms the Good-Turing estimator on the actual distribution;

\textbf{Additive property estimation}  For a broad class of additive properties, the PML plug-in estimator uses just four times the sample size required by the best estimator to achieve roughly twice its error, with exponentially higher confidence;

\textbf{$\boldsymbol{\alpha}$-R\'enyi entropy estimation} For integer $\alpha>1$, the PML plug-in estimator has optimal $k^{1-1/\alpha}$ sample complexity; for non-integer $\alpha>3/4$, the PML plug-in estimator has sample complexity lower than the state of the art;

\textbf{Identity testing} 
In testing whether an unknown distribution is equal to or at least~$\varepsilon$ far from a given distribution in $\ell_1$ distance,  a PML-based tester achieves the optimal sample complexity up to logarithmic factors of $k$. 

Most of these results also hold for a near-linear-time computable variant of PML. 
Stronger results hold for a different and novel variant called truncated PML (TPML).
\vspace{-0.5em}
\end{abstract}

\section{Introduction}
\subsection{Distributions and their properties}\label{sec:prop}
A distribution $p$ over a discrete alphabet $\cX$ of size $\size$ corresponds to 
an element of the simplex
\[
\Delta_\cX:=\Brace{p\in \mathbb{R}_{\geq 0}^{\size}: \sum_{x\in\cX} p(x)=1}. 
\]
A distribution \emph{property} is a
mapping $f:\Delta_\cX\to \mathbb{R}$
associating a real value with each distribution.
For example its support size. 
A distribution property $f$ is \emph{symmetric} if it is invariant under
domain-symbol permutations. A symmetric property  is \emph{additive}
if it can be written as $f(p):=\sum_x f(p(x))$, 
where for simplicity we use $f$ to denote both the property and the
corresponding real function. 

Many important symmetric properties
are additive. For example, 
\begin{itemize}
\item {\bf Support size} $S(p):=\sum_{x}\indic_{p(x)>0}$, a
fundamental quantity arising in the study of vocabulary size~\citep{
  et76, vocabulary, te87}, population estimation~\citep{
  population,mcxlbg}, and database studies~\citep{database}.
  
\item {\bf Support coverage} $C_m(p):=\sum_{x}(1-(1-p(x))^m)$, where $m$ is a given parameter, the
 expected number of distinct elements observed in a sample of size $m$, arising in 
 biological~\citep{ca84,klr99} and ecological~\citep{ca84, ca14, ca92,ecological} research; 

\item {\bf Shannon entropy} $H(p):=-\sum_{x} p(x)\log p(x)$, the 
primary measure of information~\citep{info, S48} with numerous 
applications to machine learning~\citep{bg,chowliu,qkc} and 
neuroscience~\citep{snm, mzfstj}; 
 
\item {\bf Distance to uniformity} $D(p):=\norm{p-p_u}_1$, where $p_u$ is the uniform distribution over $\cX$,
a property being central to the field of distribution 
property testing~\citep{bfrsw, bffkrw, testingu, rd10}.
 \end{itemize}

Besides being symmetric and additive, these four properties have
yet another attribute in common. Under the appropriate interpretation,
they are also all $1$-Lipschitz. 
Specifically, for two distributions $p,q\in\Delta_\cX$, 
let $\Gamma_{p,q}$ be the collection of distributions over
$\cX\times\cX$ with marginals $p$ and $q$ on the first and second factors respectively. 
The \emph{relative earth-mover distance}~\citep{VV11O}, between $p$ and
$q$ is
\[
R(p,q)
:=
\inf_{\gamma\in\Gamma_{p,q}}\;\Exp_{(X,Y)\sim\gamma}
\Abs{\log\frac{p(X)}{q(Y)}}.
\]
One can verify~\citep{VV11O,instdist} that
$H$, $D$, and $\tilde{C}_m:=C_m/m$ are all $1$-Lipschitz\vspace{-0.1em}
 on the metric
space $(\Delta_\cX, R)$, and
$\tilde{S}:=S/\size$ is $1$-Lipschitz over $(\Delta_{\ge{1}/{\size}},R)$, 
the set of distributions in
$\Delta_\cX$ whose nonzero probabilities are at least $1/{\size}$. 
We will study all such Lipschitz properties in later sections. 

An important symmetric non-additive property is \emph{R\'enyi entropy}, a well-known measure of
 randomness with numerous applications to unsupervised learning~\citep{JK03, XuD99} and image registration~\citep{MaH00,NHA06}. For a distribution
 $p\in \Delta_\cX$ and a non-negative real parameter $\alpha\not=1$, the \emph{$\alpha$-R\'enyi entropy}~\citep{renyientropy} of $p$ is $H_{\alpha} (p):=(1-\alpha)^{-1}\log\Paren{\sum_{x} p(x)^\alpha}$. 
In particular, denoted by $H_{1} (p):=\lim_{\alpha\to 1} H_{\alpha} (p)$, the \emph{$1$-R\'enyi entropy} is exactly Shannon entropy~\citep{renyientropy}. 

\subsection{Problems of interest}
We consider the following three fundamental statistical-learning problems.

\subsection*{Distribution estimation}
A natural learning problem is to estimate an unknown distribution $p\in \Delta_\cX$ from an \iid\ sample $X^n\sim p$. For any two distributions $p, q\in \Delta_\cX$, let $\ell(p, q)$
 be the 
  \emph{loss} when we approximate $p$ by $q$. A~\emph{distribution estimator} $\hat{p}:\cX^*\to \Delta_\cX$ associates every sequence $x^n\in \cX^*$ with a distribution $\hat{p}(x^n)$. We measure the performance of an estimator by its \emph{sample complexity} 
\[
n(\hat{p}, \varepsilon, \delta)
:=\min\{n\!: 
\forall p\in \Delta_\cX,\; \Pr_{X^n\sim p}\Paren{\ell(p, \hat{p}(X^n))\ge\varepsilon}\le\delta\},
\]
the smallest sample size that $\hat{p}$ requires to estimate all distributions in $\Delta_\cX$ to a desired accuracy $\varepsilon>0$, with error probability $\delta\in(0,1)$. The sample complexity of distribution estimation over $\Delta_\cX$ is
\[
n(\varepsilon, \delta)
:=\min \{n({\hat{p}}, \varepsilon,\delta)\!: \hat{p}:\cX^*\to \Delta_\cX\},
\]
the lowest sample complexity of any estimator. For simplicity, we will omit $\delta$ when $\delta=1/3$. 

For a distribution $p\in\Delta_\cX$, we denote by $\{p\}$ the multiset of its probabilities. The \emph{sorted $\ell_1$ distance} between two distributions 
$p,q\in\Delta_\cX$ is
\[
\ell_1^{\text{\tiny{<}}}(p,q):=\min_{p'\in\Delta_\cX: \{p'\}=\{p\}}\norm{p'-q}_1,
\]
the smallest $\ell_1$ distance between $q$ and any sorted version of $p$.
As illustrated in \addIII this is essentially the \emph{$1$-Wasserstein distance} between uniform measures on the probability multisets $\{p\}$ and $\{q\}$. We will consider both the sorted and unsorted $\ell_1$ distances.

\subsection*{Property estimation}
Often we would like to estimate a given property $f$ of an
unknown distribution $p\in \Delta_\cX$ based on a sample $X^n\sim p$.
A \emph{property estimator} is a mapping $\hat{f}:\cX^*\to \mathbb{R}$. Analogously, the
\emph{sample complexity} of $\hat{f}$ in estimating $f$ over a set 
$\cP\subseteq\Delta_\cX$ is 
\[
n_f({\hat{f}}, \cP, \varepsilon, \delta)
:=
\min\{n\!: 
\forall p\in \cP,\; \Pr_{X^n\sim p}(|\hat{f}(X^n)-f(p)|\ge\varepsilon)\le\delta\},
\]\par
 \vspace{-1em}
 the smallest sample size that $\hat{f}$ requires to estimate
$f$ with accuracy $\varepsilon$ and confidence $1-\delta$,
for all distributions in $\cP$. 
The sample complexity of estimating 
$f$ over $\cP$ 
is
\[
n_f(\cP,\varepsilon, \delta)
:=
\min\{n_f({\hat{f}}, \cP, \varepsilon,\delta)\!: \hat{f}:\cX^*\to \mathbb{R}\},
\]
the lowest sample complexity of any estimator. 
For simplicity, we will omit $\cP$ when $\cP=\Delta_\cX$,
and omit $\delta$ when $\delta=1/3$. 
The standard ``median-trick" shows that \vspace{-0.2em}
$\log(1/\delta)\cdot n_f(\cP,\varepsilon)\geq \Omega(n_f(\cP,\varepsilon,\delta))$. 
By convention, we say an estimator $\hat{f}$ is \emph{sample-optimal} if
$n_f({\hat{f}}, \cP, \varepsilon)=\Theta(n_f(\cP,
\varepsilon))$. 

\subsection*{Property testing: Identity testing}
A closely related problem is distribution property testing, of which identity testing is the most fundamental and well-studied~\citep{testingu, G17}. Given an error parameter $\varepsilon$, a distribution $q$, and a sample $X^n$ from an unknown distribution $p$, \emph{identity testing} aims to distinguish between the null hypothesis 
\[
H_0: p = q
\]\par\vspace{-1em}
and the alternative hypothesis 
\[
H_1: \norm{p-q}_1\geq \varepsilon.
\] 
A \emph{property tester} is a mapping $\hat{t}:\cX^*\to \{0, 1\}$,
indicating whether $H_0$ or $H_1$ is accepted.
Analogous to the two formulations above,  the
\emph{sample complexity} of $\hat{t}$ is 
\[
n_q(\hat{t}, \varepsilon, \delta)
:=\min\{n\!: \forall i\in\{0,1\}, 
 \Pr_{X^n\sim p}\Paren{\hat{t}(X^n)\not=i\mid H_i\text{ is true}}\le\delta\},
\]\par
 \vspace{-1em}
and the sample complexity of identity testing with respect to $q$ is 
\[
n_q(\varepsilon, \delta)
:=
\min\{n_q({\hat{t}}, \varepsilon,\delta)\!: \hat{t}:\cX^*\to \{0, 1\}\}.
\]
Again, when $\delta=1/3$, we will omit $\delta$.
For $q=p_u$, the problem is also known as \emph{uniformity testing}. 

\subsection{Profile maximum likelihood}
The \emph{multiplicity} of a symbol $x\in \cX$ 
in a sequence $x^n:=x_1,\ldots, x_n \in \cX^*$ is 
$\mu_x(x^n):=|\{j: x_j=x, 1\leq j\leq n\}|$, the number of times $x$
appears in $x^n$.
These multiplicities
induce an \emph{empirical distribution} 
$p_{\mu}(x^n)$ that associates a probability $\mu_x(x^n)/n$
with each symbol $x\in \cX$. 

The \emph{prevalence} of
an integer $i\geq 0$ in $x^n$ is the number $\varphi_i(x^n)$ of symbols
appearing $i$ times in $x^n$. For known $\cX$, the value of $\varphi_0$ can be
deduced from the remaining multiplicities, hence we define 
the \emph{profile} of $x^n$ to be 
$\varphi(x^n)=(\varphi_1(x^n), \ldots, \varphi_n(x^n))$,
the vector of all positive prevalences. 
For example, $\varphi(\textit{alfalfa})=(0,2,1,0,0,0,0)$. 
Note that the profile of $x^n$ also corresponds to the multiset of
multiplicities of distinct symbols in $x^n$. 

For a distribution $p\in \Delta_\cX$, let\vspace{-0.10em}
\[
p(x^n):=\Pr_{X^n\sim p}(X^n=x^n)
\]
be the probability of observing a sequence $x^n$ under i.i.d. sampling from $p$, and let 
\[
p(\varphi):=\sum_{y^n: \varphi(y^n)=\varphi}p(y^n)
\]
be the probability of observing a profile $\varphi$. 
While the sequence maximum likelihood estimator maps a sequence to its empirical distribution, which maximizes the sequence probability $p(x^n)$, the \emph{profile maximum likelihood (PML)} estimator~\citep{O04} over a set 
$\cP\subseteq\Delta_\cX$ maps each
profile $\varphi$ to a distribution \vspace{-0.15em}
\[
p_{\varphi}:=\arg\max_{p\in \cP} p(\varphi)
\]
that maximizes the profile probability. 
Relaxing the optimization objective, 
for any $\beta\in(0,1)$, a \emph{$\beta$-approximate PML}
estimator~\citep{mmcover} maps each profile $\varphi$ to a distribution
$p_{\varphi}^\beta$ such that
${p}^\beta_{\varphi}({\varphi})\ge \beta\cdot p_{\varphi}({\varphi})$. 

Originating from the principle of maximum likelihood,
PML was proved~\citep{A12,mmcover,A17,D12,O04} 
to possess 
a number of
useful attributes, such as 
existence over finite discrete domains, 
majorization by empirical distributions,
consistency for distribution estimation under both sorted and unsorted $\ell_1$ distances,
and competitiveness to other profile-based estimators.

Let $\varepsilon$ be an error parameter and \emph{$f$ be one of the four properties} in Section~\ref{sec:prop}. Set $n:=n_f(\varepsilon)$. 
Recent work of~\citet{mmcover} showed that for some absolute constant $c'>0$,  
if $c<c'$ and $\varepsilon\geq {n}^{-c}$, then a plug-in estimator for $f$,
using an $\exp(-{n}^{\text{\tiny{$1\!-\!\Theta(c)$}}})$-approximate PML, is sample-optimal.
Motivated by this result,~\citet{ChaAr} constructed an explicit $\exp(-\mathcal{O}({n}^{2/3}\log^3 n))$-approximate PML (APML) whose computation time is near-linear in $n$. Combined, these two results provide a unified, sample-optimal, and near-linear-time computable plug-in estimator for the four properties.

\section{New results and implications}\vspace{-0.5em}
\subsection{New results}
\subsubsection*{Additive property estimation}\vspace{-0.25em}
Let $f$ be an additive symmetric 
property that is 
$1$-Lipschitz on $(\Delta_\cX, R)$. Let $\varepsilon$ be an error parameter and $n\geq n_f(\varepsilon)$, the smallest sample size required by any estimator to achieve accuracy $\varepsilon$ with confidence $2/3$, for all distributions in $\Delta_\cX$. 
For an absolute constant $c\in(10^{-2}, 10^{-1})$, if $\varepsilon\geq n^{-c}$, 
\begin{Theorem}\label{thm:est}
The PML plug-in estimator, when given a sample of size $4n$ from any distribution $p\in \Delta_\cX$,  will estimate $f(p)$ up to an error of $(2+o(1)) \varepsilon$, with probability at least $1-\exp\Paren{-4\sqrt{n}}$.  
\end{Theorem}
For a different $c>0$, Theorem~\ref{thm:est} also holds for APML, which is near-linear-time computable~\citep{ChaAr}. 
In~Section~\ref{sec:TPML}, we propose a PML variation called \emph{truncated PML} (TPML) for which the lower bound on $\varepsilon$ can be improved to the near-optimal $n^{-0.49}$, for symmetric properties such as Shannon entropy, support coverage, and support size. See Theorem~\ref{thm:entro},~\ref{thm:suppc}, and~\ref{thm:supp} for detail.\vspace{-0.25em}

\subsubsection*{R\'enyi entropy estimation}
For $\cX$ of finite size $\size$ and any $p\in\Delta_\cX$, it is well-known that $H_{\alpha} (p)\in [0, \log \size]$. 
The following theorems characterize the performance of the PML plug-in estimator in estimating R\'enyi entropy.

For any distribution $p\in \Delta_\cX$, error parameter $\varepsilon\in(0,1)$, absolute constant $\lambda\in(0,0.1)$, 
and sampling parameter $n$, draw a sample $X^n\sim p$ and denote its profile by $\varphi$. Then for sufficiently large $\size$,
\begin{Theorem}\label{thm:renyi1}
For $\alpha\in(3/4, 1)$, if $n=\Omega_\alpha(\size^{1/\alpha}/(\varepsilon^{1/\alpha} \log \size))$, 
\[
\Pr\Paren{|H_{\alpha} (p_\varphi)-H_{\alpha} (p)|\geq \varepsilon}\leq \exp(-\sqrt{n}).
\]
\end{Theorem}

\begin{Theorem}\label{thm:renyi2}
For non-integer $\alpha>1$, if $n=\Omega_\alpha(\size/(\varepsilon^{1/\alpha} \log \size))$, 
\[
\Pr\Paren{|H_{\alpha} (p_\varphi)-H_{\alpha} (p)|\geq \varepsilon}\leq \exp(-n^{1-\lambda}).
\]
\end{Theorem}

\begin{Theorem}\label{thm:renyi3}
For integer $\alpha>1$, if $n=\Omega_\alpha(\size^{1-1/\alpha} (\varepsilon^{2}|\log{\varepsilon}|)^{-(1+\alpha)})$ and $H_{\alpha} (p)\le (\log n)/4$, 
\[
\Pr(|H_{\alpha} (p_\varphi)-H_{\alpha} (p)|\geq \varepsilon)\leq {1}/{3}.
\]
\end{Theorem}\vspace{-0.5em}
Replacing $3/4$ by $5/6$, Theorem~\ref{thm:renyi1} also holds for APML with a better probability bound $\exp(-{n}^{2/3})$. 
In addition, Theorem~\ref{thm:renyi2} holds for APML without any modifications. \vspace{-0.25em}

\subsubsection*{Distribution estimation}
Let $c$ be the absolute constant 
defined just prior to
Theorem~\ref{thm:est}. 
For any distribution $p\in \Delta_\cX$, error parameter $\varepsilon\in(0,1)$, and sampling parameter $n$, draw a sample $X^n\sim p$ and denote its profile by $\varphi$. 
\begin{Theorem}\label{thm:dist}
If $n=\Omega( n(\varepsilon))=\Omega\Paren{{\size}/{(\varepsilon^2\log \size)}}$ and $\varepsilon\geq n^{-c}$, 
\[
\Pr(\ell_1^{\text{\tiny{<}}}(p_\varphi,p)\geq \varepsilon)\leq \exp(-\Omega(\sqrt{n})).
\] 
\end{Theorem}\vspace{-0.5em}
For a different $c>0$, Theorem~\ref{thm:dist} also holds for APML with a better probability bound $\exp(-{n}^{2/3})$. 
In Section~\ref{sec:TPMLdist}, we show that a simple combination of the TPML and empirical estimators accurately recovers the actual distribution under a variant of the relative earth-mover distance, regardless of $k$. \vspace{-0.25em}

\subsubsection*{Identity testing}
The recent works of~\citet{DK16} and~\citet{Gol16} provided a procedure reducing identity testing to uniformity testing, while modifying the desired accuracy and alphabet size by only absolute constant factors. Hence below we consider uniformity testing. 

The uniformity tester $T_{\text{\tiny PML}}$ shown in Figure~\ref{fig:1} is purely based on PML and satisfies
\begin{Theorem}\label{thm:test}
If $\varepsilon= \tilde{\Omega}(\size^{-1/4})$ and $n=\tilde{\Omega}(\sqrt{\size}/\varepsilon^2)$, then the tester $T_{\text{\tiny PML}}(X^n)$ will be correct with probability at least $1-\size^{-2}$. The tester also distinguishes between $p=p_u$ and $\norm{p-p_u}_2\geq \varepsilon/\sqrt{k}$. 

\end{Theorem}
The $\tilde{\Omega}(\cdot)$ notation only hides logarithmic factors of $k$. 
The tester $T_{\text{\tiny PML}}$ is near-optimal as for uniform distribution $p_u$, 
the results in~\citep{DGT18} yield an ${\Omega}(\sqrt{\size\log \size}/\varepsilon^2)$ lower bound on
 $n_{p_u}(\varepsilon, k^{-2})$.

\addI

\begin{figure}[ht]
\begin{center}
\boxed{
\begin{aligned}
&\text{\bf Input:}\ \ \text{parameters }\size, \varepsilon, \text{and a sample }X^n\sim p \text{ with profile } \varphi.\\
& \text{\bf if } {\max}_x \mu_x(X^n)\geq 3\max\{1, n/\size\}\log \size \text{ \bf  then return } 1\text{;}\\
& \text{\bf elif } \norm{p_{\varphi}-p_u}_2\geq  3\varepsilon/ (4\sqrt{\size}) \text{ \bf  then return }1\text{;}\\ 
& \text{\bf else}\text{ \bf return }0
\end{aligned}
}
\end{center}
\caption{Uniformity tester $T_{\text{\tiny PML}}$}
\label{fig:1}\vspace{-1em}
\end{figure}

\subsection{Implications}
Several immediate implications are in order.

{\bf Theorem~\ref{thm:est}} makes PML the first plug-in estimator that is universally \emph{sample-optimal}  for a broad class of distribution properties. 
In particular, Theorem~\ref{thm:est} also covers the four properties considered in~\citep{mmcover}. To see this, as mentioned in Section~\ref{sec:prop}, $\tilde{C}_m$, $H$, and $D$ are $1$-Lipschitz on $(\Delta_\cX, R)$; as for $\tilde{S}$, the following result~\citep{mmcover} relates it to $\tilde{C}_m$ for distributions in $\Delta_{\geq {1}/{\size}}$, and proves PML's optimality.
\begin{Lemma}\label{lem:sc}
For any $\varepsilon>0$, $m=k \log(1/\varepsilon)$, and $p\in \Delta_{\geq {1}/{\size}}$, 
\[
|\tilde{S}(p)-  \tilde{C}_m(p)\log(1/\varepsilon)|\leq \varepsilon.
\]
\end{Lemma}\vspace{-0.5em}
The theorem also applies to many other properties. As an example~\citep{VV11O}, given an integer $s>0$, let $f_s(x) := \min\{x, |x-1/s|\}$. Then to within a factor of two,  $f_s(p) :=\sum_x f_s(p(x))$ approximates the $\ell_1$ distance between any distribution $p$ and the closest uniform distribution in $\Delta_\cX$ of support~size~$s$. 

In Section~\ref{sec:com_add} we compare Theorem~\ref{thm:est} with existing results and present more of its implications.

{\bf Theorem~\ref{thm:renyi1} and~\ref{thm:renyi2}} imply that for all
non-integer $\alpha>3/4$ (resp. $\alpha>5/6$), the PML  (resp. APML) plug-in estimator achieves a sample complexity better than the best currently known~\citep{AO17}. This makes both the PML and APML plug-in estimators the state-of-the-art algorithms for estimating non-integer order R\'enyi entropy.  
See Section~\ref{sec:rel_ren} for an introduction of known results, and see Section~\ref{sec:comp_renyi} for a detailed comparison between existing methods and ours.   

{\bf Theorem~\ref{thm:renyi3}} shows that for all integer $\alpha>1$, the sample complexity of the PML plug-in estimator has optimal $k^{1-1/\alpha}$ dependence~\citep{AO17,OS17} on the alphabet size $k$.

{\bf Theorem~\ref{thm:dist}} makes APML the first distribution
estimator under sorted $\ell_1$ distance that is both near-linear-time
computable and sample-optimal for a range of desired accuracy
$\varepsilon$ beyond inverse polylogarithmic of $n$. In comparison, existing
algorithms~\citep{A12,jnew,ventro} either run in polynomial time in
the sample sizes, or are only known to achieve optimal sample complexity for $\varepsilon=\Omega(1/\sqrt{\log n})$, which is essentially different from the applicable range of $\varepsilon\geq n^{-\Theta(1)}$ in Theorem~\ref{thm:dist}. We provide a more detailed comparison in Section~\ref{sec:com_dist}.

{\bf Theorem~\ref{thm:test}}  
provides the first PML-based uniformity tester with near-optimal sample complexity. As stated, the tester also distinguishes between $p=p_u$ and $\norm{p-p_u}_2\geq \varepsilon/\sqrt{k}$. 
This is a stronger guarantee since by the Cauchy-Schwarz inequality,
$\norm{p-p_u}_1\geq \varepsilon$ implies $\norm{p-p_u}_2\geq \varepsilon/\sqrt{k}$.\vspace{-0.25em}

\section{Related work and comparisons}\label{sec:prior}
\subsection{Additive property estimation}\label{sec:rel_add}
The study of additive property estimation dates back at least half a
century~\citep{emiller, population, G56} and has steadily grown over the
years. For any additive symmetric property $f$ and sequence $x^n$, the
simplest and most widely-used approach uses the \emph{empirical
  (plug-in)} estimator $\hat{f}^E(x^n):=f(p_{\mu}(x^n))$
that evaluates $f$ at the empirical distribution.
While the empirical estimator
performs well in the large-sample regime, modern data science applications often
concern high-dimensional data, for which more involved methods have yielded property 
estimators that are more sample-efficient.
For example, for relatively large $\size$ and for $f$ being $\tilde{S}$, $\tilde{C}_m$,
$H$, or $D$, recent research~\citep{jvhw, O16, VV11, VV11O, 
  mmentro,W19} showed that 
the empirical estimator is optimal up to logarithmic factors, namely
$
n_f(\cP,\varepsilon)=
\Theta_{\varepsilon}({n_f(\hat{f}^E,\cP, \varepsilon)}/{\log
n_f(\hat{f}^E,\cP, \varepsilon)}),
$
where $\cP$ is $\Delta_{\geq {1}/{\size}}$ for $\tilde{S}$,
and is $\Delta_\cX$ for the other properties. 

Below we classify
the methods for deriving the corresponding sample-optimal estimators into two
categories: plug-in and approximation, and provide a high-level
description. For simplicity of illustration, we assume
that $\varepsilon\in (0,1]$.

The \emph{plug-in} approach essentially estimates the unknown distribution multiset, 
which suffices for computing any
symmetric properties. Besides the empirical and PML
estimators,
\citet{et76} proposed a linear-programming approach that finds a multiset estimate consistent with the sample's profile. This approach was then adapted and analyzed by~\citet{VV11, ventro}, yielding plug-in estimators that achieve near-optimal sample complexities for $H$ and $\tilde{S}$, and optimal sample complexity for $D$, when $\varepsilon$ is relatively large. 

The \emph{approximation} approach modifies non-smooth segments of the probability function to correct the bias of empirical estimators. A popular modification is to replace those non-smooth segments by their low-degree polynomial approximations and then estimate the modified function. 
For several properties including the above four and \emph{power sum} $P_{\alpha}(p):=\sum_{x} p(x)^\alpha$, where $\alpha$ is a given parameter, 
this approach yields property-dependent estimators~\citep{jvhw, O16,
  mmentro, W19} that are sample-optimal for all
$\varepsilon$.  

More recently,  \citet{mmcover} proved the aforementioned results on PML estimator and made it the first unified, sample-optimal plug-in estimator for  $\tilde{S}$, $\tilde{C}_m$, $H$ and $D$ and relatively large $\varepsilon$. 
Following these advances, \citet{jnew} refined the linear-programming approach and designed a plug-in estimator that implicitly performs polynomial approximation and is sample-optimal for $H$, $\tilde{S}$, and $ P_{\alpha}$ with $\alpha<1$, when $\varepsilon$ is relatively large. 

\subsection{Comparison I: Theorem~\ref{thm:est} and related
  property-estimation work}\label{sec:com_add}

In terms of the estimator's theoretical guarantee,
Theorem~\ref{thm:est} is essentially the same
as~\citet{VV11O}. However, 
for each property, $k$, and $n$, \citep{VV11O} solves a different linear program and constructs a new estimator, which takes polynomial time. 
On the other hand, both the PML estimator and its near-linear-time computable variant, once computed, can be used to accurately estimate exponentially  many properties that are $1$-Lipschitz on $(\Delta_\cX, R)$. A similar comparison holds between the PML method and the approximation approach, while the latter is provably sample-optimal for only a few properties. In addition, Theorem~\ref{thm:est} shows that the PML estimator often achieves the optimal sample complexity up to a small constant factor, which is a desired estimator attribute shared by some, but not all approximation-based estimators~\citep{jvhw, O16, mmentro, W19}. 

In term of the method and proof technique, Theorem~\ref{thm:est} is
closest to~\citet{mmcover}. On the other hand,~\citep{mmcover} 
establishes the optimality of PML for only 
four properties, while our result covers a much broader property class. 
In addition, both the above mentioned ``small constant factor'' attribute, and the confidence boost from $2/3$ to $1-\exp(-4\sqrt{n})$ are unique contributions of this work. 
The PML plug-in approach is also close in flavor to the plug-in estimators in~\citet{VV11, ventro} and their refinement in~\citet{jnew}.  
On the other hand, as pointed out previously, these plug-in estimators are provably sample-optimal for only a few properties. 
More specifically, for estimating $H$, $\tilde{S}$, and $\tilde{C}_m$, 
the plug-in estimators in~\citep{VV11, ventro} achieve sub-optimal
sample complexities with regard to the desired accuracy $\varepsilon$;
and the estimation guarantee  in~\citep{jnew} is provided in terms of
the approximation errors of $\tilde{\mathcal{O}}(\sqrt{n})$
polynomials that  are not directly related to the optimal sample
complexities.

\subsection{R\'enyi entropy estimation}\label{sec:rel_ren}
Motivated by the wide applications of R\'enyi entropy, heuristic estimators were proposed and studied 
 in the physics literature following~\cite{GP88}, and asymptotically consistent estimators were presented and 
 analyzed in the statistical-learning literature~\cite{KLS12,XE10}.  
For the special case of 1-R\'enyi (or Shannon) entropy, the works of~\citep{VV11,VV11O} determined the sample complexity to be $n_f(\varepsilon)=\Theta(\size/(\varepsilon\log \size))$. 

For general $\alpha$-R\'enyi entropy, the best-known results in~\citet{AO17} state that 
for integer and non-integer $\alpha$ values, the corresponding sample complexities $n_f(\varepsilon, \delta)$ are $\mathcal{O}_\alpha(\size^{1-1/\alpha}\log(1/\delta)/\varepsilon^{2})$ and $\mathcal{O}_\alpha(\size^{\min\{1/\alpha, 1\}}\log(1/\delta)/(\varepsilon^{1/\alpha} \log \size))$, respectively. 
The upper bounds for integer $\alpha$ are achieved by an estimator that corrects the bias of the empirical plug-in estimator. 
To achieve the upper bounds for non-integer $\alpha$ values, one needs to compute some best polynomial approximation~of~$z^\alpha$, whose degree and domain both depend on $n$, and 
construct a more involved estimator using the approximation approach~\citep{jvhw,
  mmentro} mentioned in Section~\ref{sec:rel_add}. 

\subsection{Comparison II: Theorem~\ref{thm:renyi1} to~\ref{thm:renyi3} and related R\'enyi-entropy-estimation work}\label{sec:comp_renyi}
Our result shows that a single PML estimate suffices to estimate the R\'enyi entropy of different orders $\alpha$. Such adaptiveness to the order parameter is a significant advantage of PML over existing methods.
For example, by Theorem~\ref{thm:renyi2} and the union bound, one can use a \emph{single} APML or PML to accurately approximate exponentially many non-integer order R\'enyi entropy values, yet still maintains an overall confidence of $1-\exp(-\size^{0.9})$. 
By comparison, the estimation heuristic in~\citep{AO17} requires different polynomial-based estimators for different $\alpha$ values. In particular, to construct each estimator, one needs to compute some best polynomial approximation of $z^\alpha$, which is not known to admit a closed-form formula for $\alpha\not\in\mathbb{Z}$. Furthermore, even for a single $\alpha$ and with a sample size $\sqrt{\size}$ times larger, such estimator is not known to achieve the same level of confidence as PML or APML.  

As for the theoretical guarantees, the sample-complexity upper bounds in both Theorem~\ref{thm:renyi1} and~\ref{thm:renyi2} are better than those  mentioned in the previous section. More specifically,
 for any $\alpha\in(3/4, 1)$ and $\delta\ge\exp(-\size^{0.5})$, Theorem~\ref{thm:renyi1} shows that $n_{f}(\varepsilon, \delta)=\mathcal{O}_\alpha(\size^{1/\alpha}/(\varepsilon^{1/\alpha} \log \size))$. Analogously, for any non-integer $\alpha>1$ and $\delta\ge\exp(-\size^{0.9})$, Theorem~\ref{thm:renyi2} shows that $n_{f}(\varepsilon, \delta)=\mathcal{O}_\alpha(\size/(\varepsilon^{1/\alpha} \log \size))$. Both bounds are better than the best currently known by a $\log (1/\delta)$ factor. \vspace{-0.25em}

\subsection{Distribution estimation}
Estimating large-alphabet distributions from their samples
is a fundamental statistical-learning tenet. 
Over the past few decades, distribution estimation has found numerous
applications, ranging from natural language modeling~\citep{language} to biological research~\citep{bioinfo}, and has been studied extensively. 
Under the classical $\ell_1$ and KL losses, existing research~\citep{Braess04,learning} showed that the corresponding sample complexities $n(\varepsilon)$ are $\Theta(\size/\varepsilon^2)$ and $\Theta(\size/\varepsilon)$, respectively. 
Several recent works have investigated the analogous formulation under sorted $\ell_1$ distance, and revealed a lower sample complexity of $n(\varepsilon)=\Theta(k/(\varepsilon^2\log k))$. 
Specifically, under certain conditions, \citet{ventro} and \citet{jnew} derived sample-optimal estimators using linear programming, and \citet{A12} showed that PML achieves a sub-optimal $\mathcal{O}(\size/(\varepsilon^{2.1} \log \size))$ sample complexity for relatively large $\varepsilon$. 

\subsection{Comparison III: Theorem~\ref{thm:dist} and related distribution-estimation work}\label{sec:com_dist}
We compare our results with existing ones from three different
perspectives.

{\bf Applicable parameter ranges:} As shown by~\citep{jnew}, for $\varepsilon\ll n^{-1/3}$, the 
simple empirical estimator is already sample-optimal. Hence we 
consider the parammeter range $\varepsilon= \Omega(n^{-1/3})$.
For the results in~\citep{ventro} and~\citep{A12} to hold, we would need $\varepsilon$ to be at least $\Omega(1/\sqrt{\log n})$. 
On the other hand, Theorem~\ref{thm:dist} shows that PML and APML are sample-optimal for $\varepsilon$ larger than $n^{-\Theta(1)}$. Here, the gap is exponentially large. The result in~\citep{jnew} applies to the whole range $\varepsilon=\Omega(n^{-1/3})$, which is larger than the applicable range of our results. 

{\bf Time complexity:} Both the APML and the estimator in~\citep{ventro} are near-linear-time computable in the sample sizes, while the estimator in~\citep{jnew} would require polynomial time to be computed. 

{\bf Statistical confidence:} The PML and APML achieve the desired
accuracy with an error probability at most $\exp(-\Omega(\sqrt{n}))$. On the contrary, the estimator
in~\citep{jnew} is known to achieve an error probability that
decreases only as $\mathcal{O}(n^{-3})$. The gap is again exponentially large. 
The estimator in~\citep{ventro} admits an error probability bound of $\exp(-n^{0.02})$, which is still far from ours. 

\subsection{Identity testing}
Initiated by the work of~\citep{GO00}, identity testing is arguably one of the most important and widely-studied problems in distribution property testing. Over the past two decades, a sequence of works~\citep{ CompUni13,Ach15,bffkrw, ChanD14, DK16, DKane15,  DGT18,GO00, Pan08, Val17} have addressed the sample complexity of this problem and proposed testers with a variety of  guarantees. In particular, applying a coincidence-based tester, \citet{Pan08} determined the sample complexity of uniformity testing up to constant factors; utilizing a variant of the Pearson's chi-squared statistic, \citet{Val17} resolved the general identity testing problem. 
For an overview of related results, we refer interested readers
to~\citep{testingu} and~\citep{G17}.
The contribution of this work is mainly showing that PML, is a unified
sample-optimal approach for several related problems, and as shown in
Theorem~\ref{thm:test}, also provides a near-optimal tester for this important testing problem.

\section{Numerical experiments}\label{sec:exp}
A number of different approaches have been taken to computing the PML and its approximations. 
Among the existing works, \citet{ExactPML} considered exact algebraic computation, \citet{EMMCMC04,O04} designed an EM algorithm with MCMC acceleration, \citet{Von12,Von14I} proposed a Bethe approximation heuristic, \citet{A17} introduced a sieved PML estimator and a stochastic approximation of the associated EM algorithm, and \citet{Pav17} derived a dynamic programming approach. Notably and recently, for a sample size $n$, \citet{ChaAr} constructed an explicit $\exp(-\mathcal{O}({n}^{2/3}\log^3 n))$-approximate PML whose computation time is near-linear in $n$. 

In this section, we first introduce a variant of the MCMC-EM algorithm in~\citep{EMMCMC04, O04, SPan} and then demonstrate the 
efficacy of PML on a variety of learning tasks through experiments.  

\subsection{MCMC-EM algorithm variant}\label{sec:EMv}
To approximate PML, the work~\citep{EMMCMC04} proposed an MCMC-EM algorithm, where MCMC and EM stand for Markov chain Monte Carlo and expectation maximization, respectively. A sketch of the original MCMC-EM algorithm can be found in~\citep{EMMCMC04}, and a detailed description is available in Chapter~6 of~\citep{SPan}.  The EM part  uses a simple iteration procedure to update the distribution estimates. One can show~\citep{SPan} that it is equivalent to the conventional \emph{generalized gradient ascent method}. 
The MCMC part exploits local properties of the update process and accelerates the EM computation. Below we present a variant of this algorithm that often runs faster and is more accurate.

\paragraph{Step 1:} We separate the large and small  multiplicities. Define a threshold parameter $\tau:=1.5\log^2 n$ and suppress $X^n$ in $p_\mu(X^n)$ for simplicity. For symbols $x$ with $\mu_x(X^n)\ge \tau$, estimate their probabilities by $p_{\mu}(x) = \mu_x(X^n)/n$ and remove them from the sample.  Denote the collection of removed symbols by $R$ and the remaining sample sequence by $X^r$. In the subsequent steps, we apply the EM-MCMC algorithm to $X^r$. 

The idea is simple: By the Chernoff-type bound for binomial random variables, with high probability, the
empirical frequency $\mu_x(X^n)/n$ of a large-multiplicity symbol $x$
is very close to its mean value~$p(x)$.
Hence for large-multiplicity symbols we can simply use the empirical
estimates and focus on estimating the probabilities of
small-multiplicity 
symbols. This is similar to initializing the EM
algorithm by the empirical distribution and fixing the large
probability estimates through the iterations. However, the approach
described here is more efficient. 

\paragraph{Step 2:} We determine a proper alphabet size for the output
distribution of the EM algorithm. If the true value $k$ is provided,
then we simply use $k-|R|$. Otherwise, we apply the following support
size estimator~\cite{mmcover} to $X^r$: 
\[
\hat{S}(X^r):= \sum_{j\geq 1}(1-(-(t-1))^j \Pr (L\geq j))\cdot \varphi_j(X^r),
\]
where $t=\log r$ and $L$ is an independent binomial random variable with support size $\lceil\frac{1}{2}\log_2(\frac{rt^2}{t-1})\rceil$ and success probability ${(t+1)^{-1}}$. For any $\varepsilon$ larger than an absolute constant, estimator $\hat{S}$ achieves the optimal sample complexity $n_f(\Delta_{\geq 1/k},\varepsilon)$ in estimating support size, up to constant factors~\citep{mmcover}.

\paragraph{Step 3:} Apply the MCMC-EM algorithm in~\citep{EMMCMC04, SPan} to $\varphi(X^r)$ with the output alphabet size determined in the previous step, and denote the resulting distribution estimate by $p_{r}$. (In the experiments, we perform the EM iteration for $30$ times.) 
Intuitively, this estimate corresponds to the conditional distribution given that the next observation is a symbol with small probability.

\paragraph{Step 4:}  Let $T_\mu:=\sum_{x\in R} p_\mu(x)$ be the total
probability of the large-multiplicity symbols. Treat $p_r$ as a vector
and let $p_r':=(1-T_r)\cdot p_r$. For every symbol $x\in R$, append
$p_\mu(x)$ to $p_r'$, and return the resulting vector.  Note that this 
vector corresponds to a valid discrete distribution.

\vfill
\pagebreak

\paragraph{Algorithm code}\

The implementation of our algorithm is available at~\url{https://github.com/ucsdyi/PML}.

For computational efficiency, the program code for the original MCMC-EM algorithm in~\citep{EMMCMC04, SPan} is written in C++, with a file name ``MCMCEM.cpp''. The program code for other functions is written in Python3. Note that to execute the program, one should have a 64-bit Windows/Linux system with Python3 installed (64-bit version). In addition, we also use functions provided by ``NumPy'' and ``SciPy'', while the latter is not crucial and can be removed by modifying the code slightly.  

Our implementation also makes use of ``ctypes'', a \emph{built-in}
foreign language library for Python that allows us to call C++ functions directly. Note that before calling C++ functions in Python, we need to compile the corresponding C++ source files into DLLs or shared libraries. 
We have compiled and included two such files, one is ``MCMCEM.so'', the other is ``MCMCEM.dll''.

Functions in ``MCMCEM.cpp'' can be used separately. To compute a PML estimate, simply call the function ``int PML(int MAXSZ=10000, int maximum\_EM=20, int EM\_n=100)'', where the first parameter specifies an upper bound on the support size of the
output distribution,  the second provides the maximum number of EM iteration, and the last corresponds to the sample size $n$.
This function takes as input a local file called ``proFile'', which contains the profile vector $\varphi(X^n)$ in the format of ``1 4 7 10 \ldots''. Specifically, the file ``proFile'' consists of only space-separated non-negative integers, and the $i$-th integer represents the value of $\varphi_i(X^n)$.  The output is a vector of length at most MAXSZ, and is stored in another local file called ``PMLFile''. Each line of the file ``PMLFile'' contains a non-negative real number, corresponding to a probability estimate. 

To perform experiments and save the plots to the directory containing
the code, simply execute the file ``Main.py''. To avoid further
complication, the code compares our estimator with only three other
estimators: empirical, empirical with a larger $n\log n$ sample size, and improved Good-Turing~\citep{compdist} (for distribution estimation under unsorted $\ell_1$ distance). The implementation covers all the distributions described in the next section. One can test any of these distributions by including it in ``D\_List'' of the ``main()'' function.
The implementation also covers a variety of learning tasks, such as distribution estimation under sorted and unsorted $\ell_1$ distances, and property estimation for Shannon entropy, $\alpha$-R\'enyi entropy, support coverage, and support size.

Finally, functions related to distribution and sample generation are available in file ``Samples.py''. Others including the property computation functions, the sorted and unsorted $\ell_1$ distance functions, and the previously-described support size estimator, are contained in file ``Functions.py''.

\subsection{Experiment distributions}
In the following experiments, samples are generated according to six distributions with the same support size $k=5{,}000$. 

Three of them have finite support by definition: 
uniform distribution, 
 two-step distribution with half the symbols having probability $2/(5k)$ and the other half have probability $8/(5k)$,
and a three-step distribution with one third the symbols having
probability $3/(13k)$,  another third having probability $9/(13k)$,
and the remaining having probability $27/(13k)$. 

The other three distributions are over $\{i\in\mathbb{Z}: i\geq 1\}$, and are truncated 
at $i=5{,}000$ and re-normalized: geometric distribution with parameter $g=1/k$ satisfying $p_i \propto (1-g)^i$,
 Zipf distribution with parameter $1/2$ satisfying $p_i \propto i^{-1/2}$,
and log-series distribution with parameter $\gamma = 2/k$  satisfying $p_i\propto (1-\gamma)^i/i$. 

\subsection{Experiment results and details}
As shown below, the proposed PML approximation algorithm has exceptional performance. 

\paragraph{Distribution estimation under $\boldsymbol{\ell_1}$ distance}\

We derive a new distribution estimator under the (unsorted) $\ell_1$ distance by combining the proposed PML computation algorithm with the denoising procedure in~\citep{instdist} and a missing mass estimator~\citep{compdist}. 
 
First we describe this distribution estimator, which takes a sample $X^n$ from some unknown distribution $p$. An optional input is $\cX$, the underlying alphabet. 
 
\paragraph{Step 1:} Apply the PML computation algorithm described in Section~\ref{sec:EMv} to $X^n$, 
and denote the returned vector,
consisting of non-negative real numbers that sum to $1$, by $V$.

\paragraph{Step 2:} Employ the following variant of the denoising procedure in~\citep{instdist}. 
Arbitrarily remove a total probability mass of $\log^{-2} n$ from
entries of the vector $V$ without making any entry negative. Then
for each $j\leq \log^2 n$, 
augment the vector by 
$n/(j\log^4 n)$ entries of probability $j/n$.
For every multiplicity $\mu\geq 1$ appearing in the sample,
assign to all symbols appearing $\mu$ times the following probability value.
If $\mu\geq \log^2 n$, simply assign to each of these symbols
the empirical estimate $\mu/n$;
otherwise, temporally associate a weight of $\Bin(n,v, \mu):=\binom{n}{\mu}(1-v)^{n-\mu}v^\mu$ with each entry $v$ in $V$,
 and assign to each of these symbols the current weighted median of  $V$.
 \paragraph{Step 3:} If $\cX$ is available, we can estimate the total probability mass $M(X^n):= \sum_{x\in \cX} \indic_{x\not\in X^n}$ of the unseen symbols (a.k.a., the \emph{missing mass}) by the following estimator:
 \[
 \hat{M}(X^n) := \frac{\varphi_1(X^n)}{\sum_{j} (j\varphi_j(X^n)\indic_{j> \varphi_{j+1}}+(j+1)\varphi_{j+1}(X^n)\indic_{j\le \varphi_{j+1}})}.
 \]
We equally distribute this probability mass estimate among symbols that do not appear in the sample. 

As shown below, the proposed distribution estimator achieves the state-of-the-art performance. 

In Figures~\ref{figure1}, the horizontal axis reflects the sample size
$n$, ranging from $10{,}000$ to $100{,}000$, and the vertical axis
reflects the (unsorted) $\ell_1$ distance between the true
distribution and the estimates, averaged over $30$ independent
trials. We compare our estimator with three others:
the improved Good-Turing estimator~\citep{compdist}, 
the empirical estimator, serving as a baseline, and
the empirical estimator with a larger $n\log n$ sample size.
Note that $\log n$ is roughly $11$. 
As shown in~\citep{compdist}, 
the improved Good-Turing estimator
is provably instance-by-instance near-optimal and
substantially outperforms other estimators such as the Laplace
(add-$1$) estimator, the Braess-Sauer estimator~\citep{Braess04}, and
the Krichevsky-Trofimov estimator~\citep{Krich81}.  Hence we do not
include those estimators in our comparisons. 

As the following plots show, our proposed estimator outperformed the
improved Good-Turing estimator in all experiments. 

\begin{figure*}[ht]
\begin{multicols}{3}
    \includegraphics[width=\linewidth]{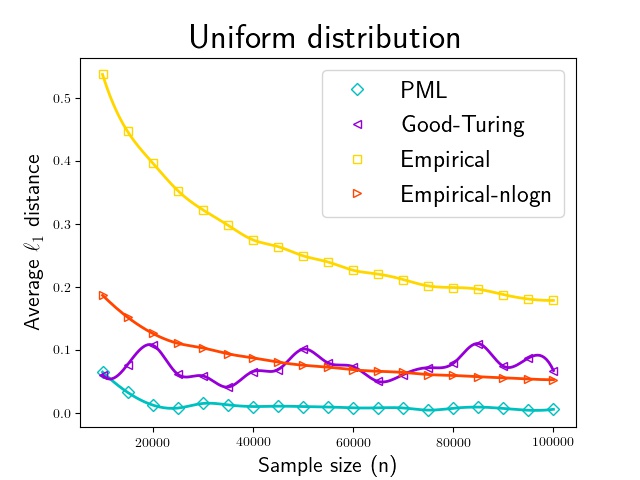}\par 
    \includegraphics[width=\linewidth]{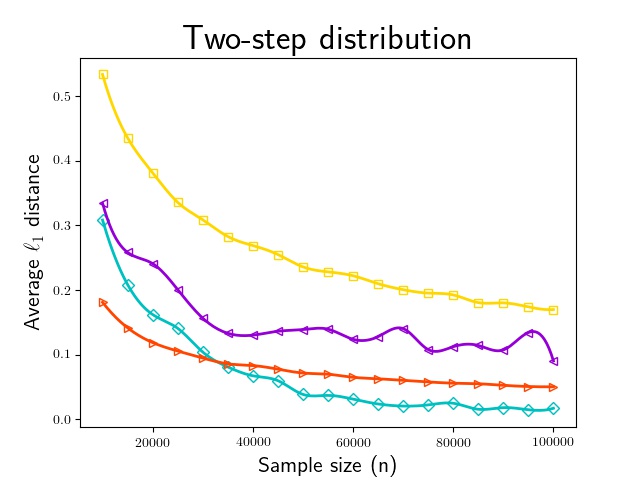}\par 
    \includegraphics[width=\linewidth]{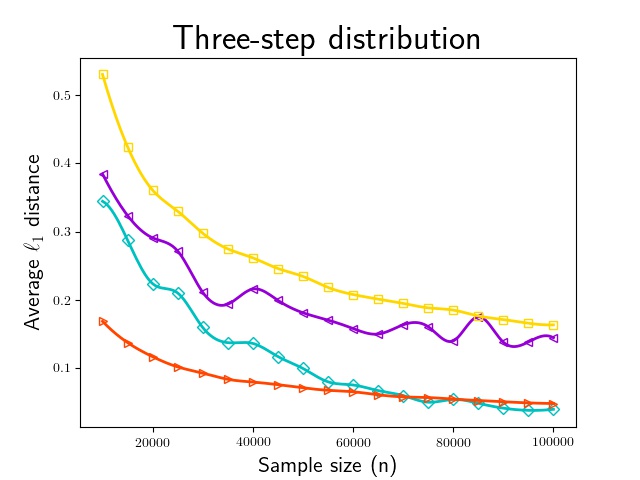}\par
     \end{multicols}\vspace{-1.5em}
\begin{multicols}{3}
    \includegraphics[width=\linewidth]{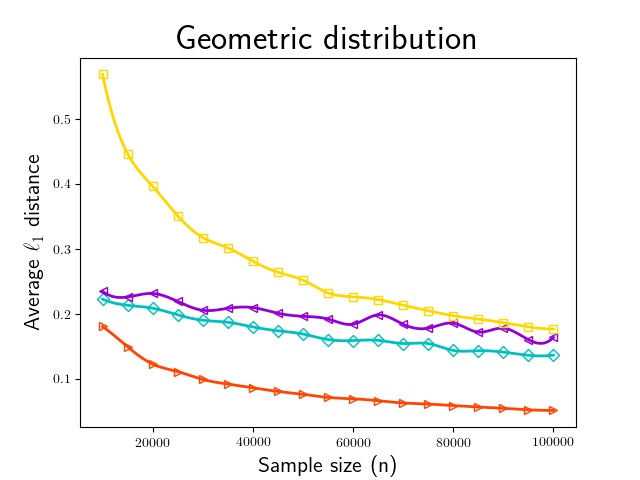}\par
   \includegraphics[width=\linewidth]{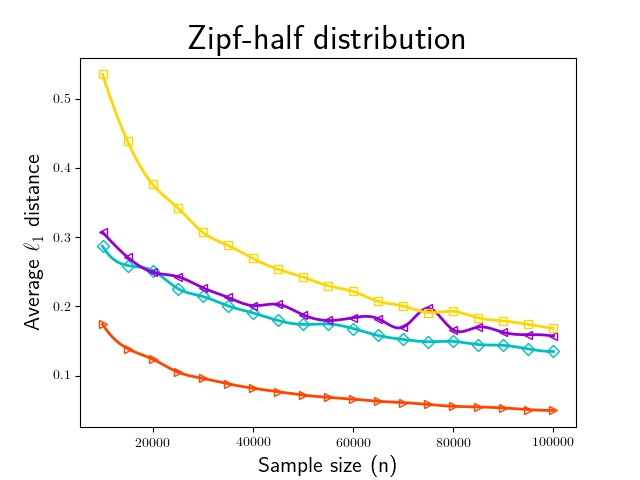}\par
   \includegraphics[width=\linewidth]{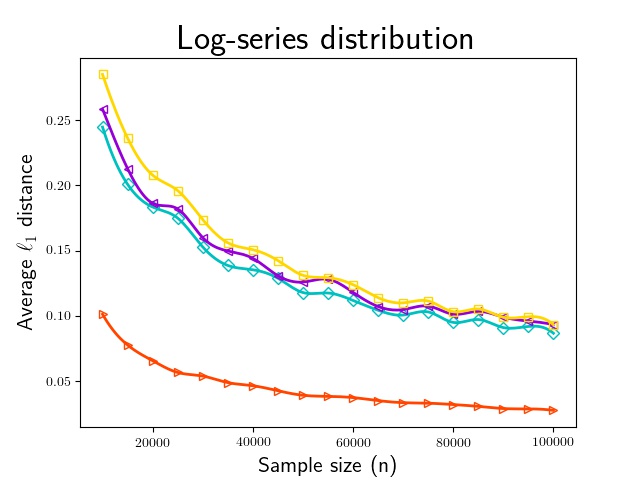}\par
\end{multicols}
\vspace{-1em}
\caption{Distribution estimation under $\protect{\ell_1}$ distance}
\label{figure1}
\end{figure*}

\vfill
\pagebreak

\paragraph{Distribution estimation under sorted $\boldsymbol{\ell_1}$ distance}\

In Figure~\ref{figure2}, the sample size $n$ ranges from $2{,}000$ to
$20{,}000$, and the vertical axis reflects the sorted $\ell_1$
distance between the true distribution and the estimates, averaged
over $30$ independent trials. We compare our estimator with that
proposed by~\citet{ventro} that utilizes linear programming, with
the empirical estimator, and with the empirical estimator with a larger
$n\log n$ sample size.

We do not include the estimator in~\citep{jnew} since there is no
implementation available, and as pointed out by the recent work of~\citep{VKVK19} (page 7), the approach in~\citep{jnew} ``is quite unwieldy. It involves significant parameter tuning and special treatment for the edge cases.'' and ``Some techniques \ldots are quite crude and likely lose large constant factors both in theory and in practice.''

As shown in Figure~\ref{figure2}, with the exception of uniform
distribution, where the estimator in~\citet{ventro} (VV-LP) is the
best and PML is the closest second, the PML estimator outperforms
VV-LP for all other tested distributions. As the underlying distribution becomes more skewed, the improvement of PML over VV-LP grows. For the log-series distribution, the performance of VV-LP is even worse than the empirical estimator. 

Additionally, the plots also demonstrate that PML has a more stable performance than VV-LP. 

\begin{figure*}[ht]
\begin{multicols}{3}
    \includegraphics[width=\linewidth]{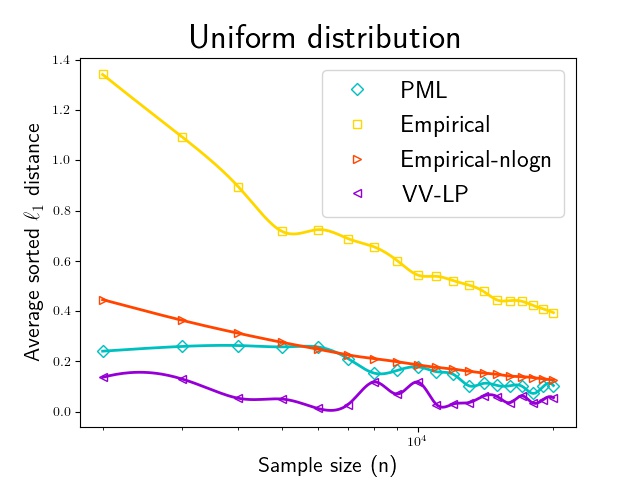}\par 
    \includegraphics[width=\linewidth]{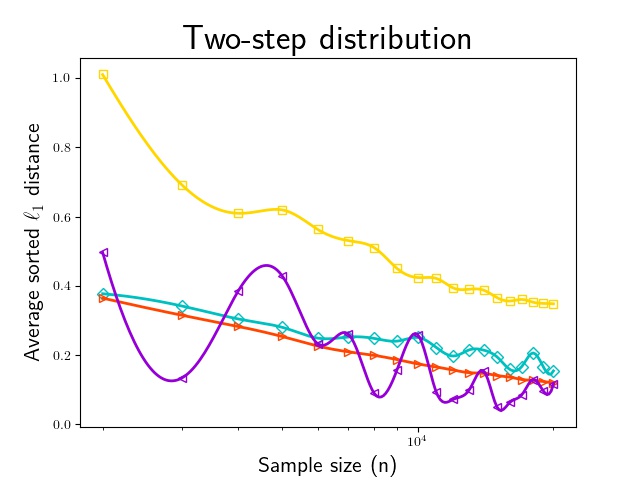}\par 
    \includegraphics[width=\linewidth]{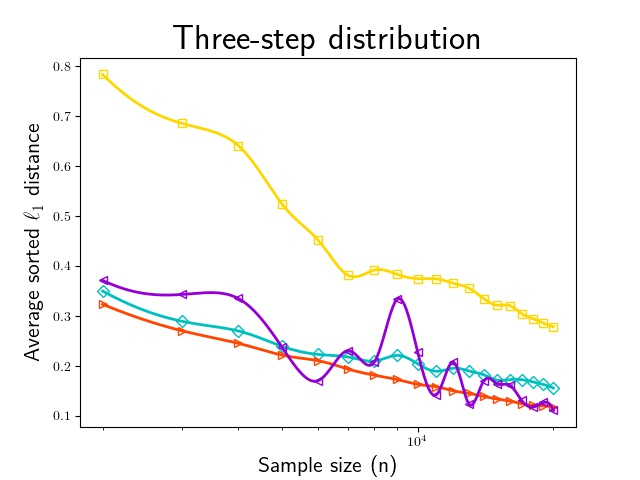}\par
     \end{multicols}\vspace{-1em}
\begin{multicols}{3}
    \includegraphics[width=\linewidth]{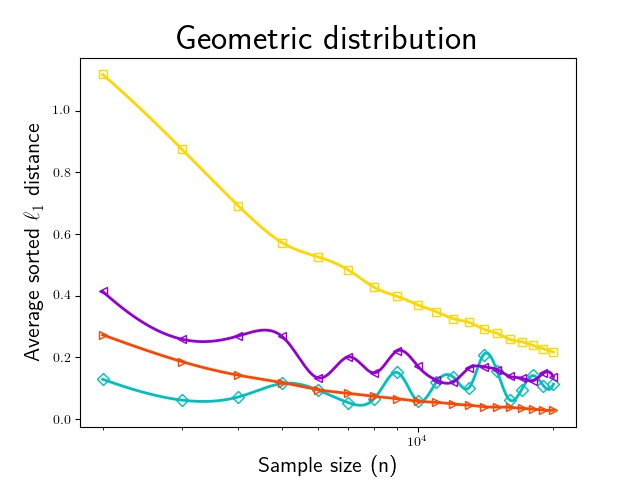}\par
   \includegraphics[width=\linewidth]{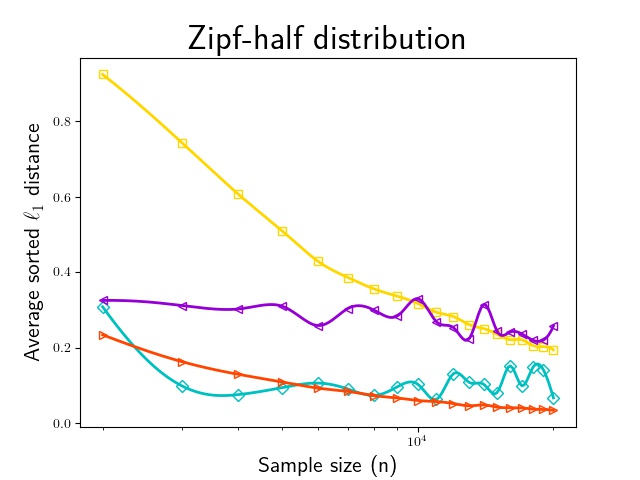}\par
   \includegraphics[width=\linewidth]{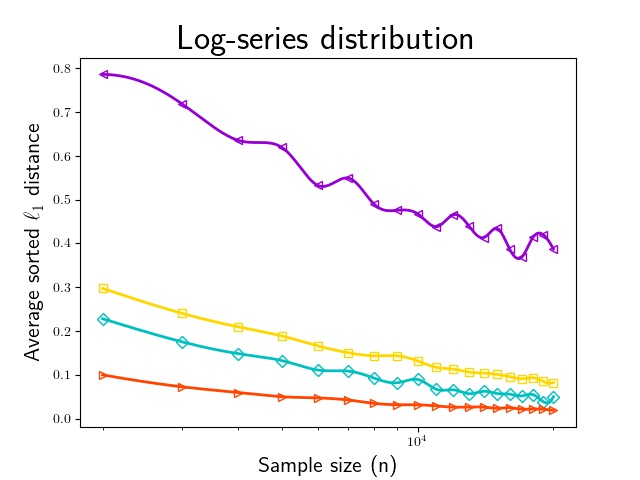}\par
\end{multicols}
\vspace{-0.5em}
\caption{Distribution estimation under sorted $\protect{\ell_1}$ distance}
\label{figure2}
\end{figure*}

\vspace{1em}
\paragraph{Shannon entropy estimation under absolute error}\

In Figure~\ref{figure3}, the sample size $n$ ranges from $1{,}000$ to
$1{,}000{,}000$, and the vertical axis reflects the absolute
difference between the true entropy values and the estimates, averaged
over $30$ independent trials. We compare our estimator with two
state-of-the-art estimators, 
\emph{WY}  in~\citep{mmentro}, 
and \emph{JVHW} in \citep{jvhw},
as well as the empirical estimator,
and the empirical estimator with a larger $n\log
n$ sample size.  Additional entropy estimators such as the
Miller-Mallow estimator~\citep{emiller}, the best upper bound (BUB)
estimator~\citep{pentro}, and the Valiant-Valiant
estimator~\citep{ventro} were compared in~\citep{mmentro,jvhw} and
found to perform similarly to or worse than the two estimators that we
compared with, therefore we do not include them
here. Also, considering~\cite{ventro}, page 50 in~\citep{PYang}
notes that ``the performance of linear programming
estimator starts to deteriorate when the sample size is very large.''

Note that the alphabet size $k$ is a crucial input to WY, but is not
required by either JVHW or our PML algorithm. In the experiments, we
provide WY with the true value of $k=5{,}000$.

As shown in the plots, our estimator performs as well as these state-of-the-art estimators. 
\vfill

\begin{figure*}[ht]
\begin{multicols}{3}
    \includegraphics[width=\linewidth]{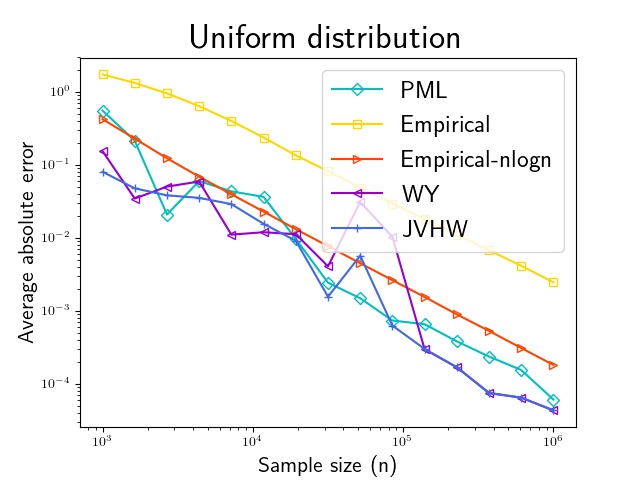}\par 
    \includegraphics[width=\linewidth]{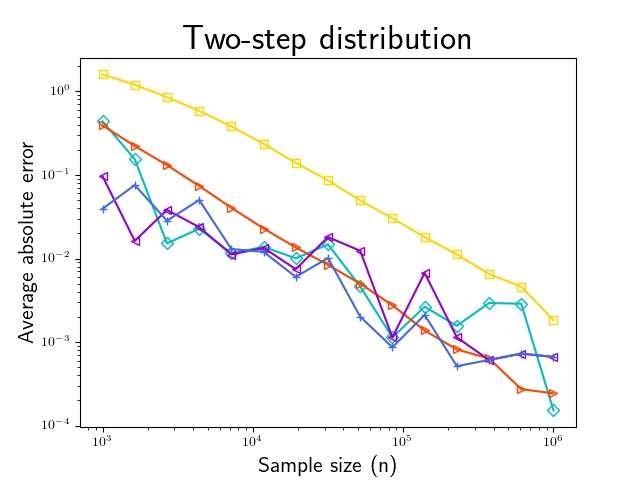}\par 
    \includegraphics[width=\linewidth]{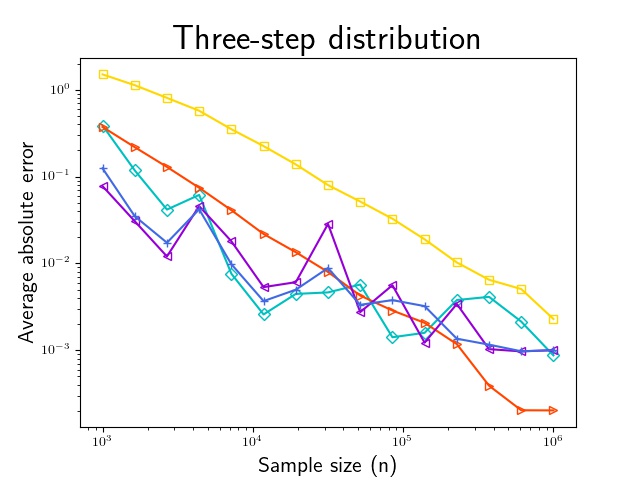}\par
     \end{multicols}\vspace{-2em}
\begin{multicols}{3}
    \includegraphics[width=\linewidth]{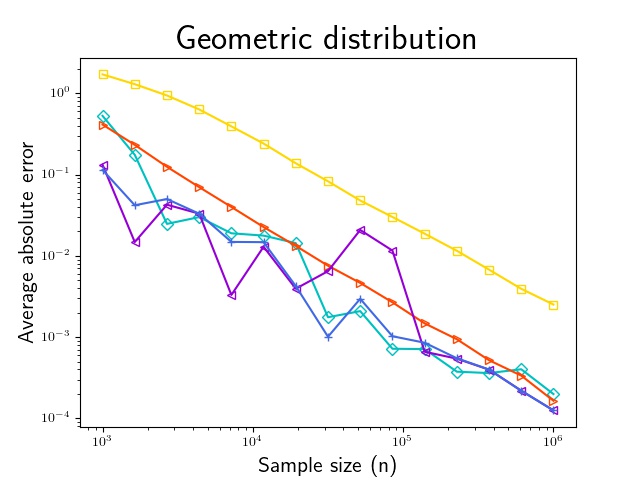}\par
   \includegraphics[width=\linewidth]{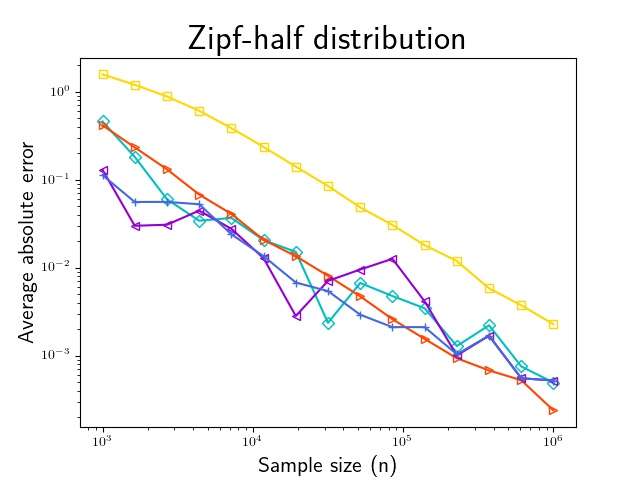}\par
   \includegraphics[width=\linewidth]{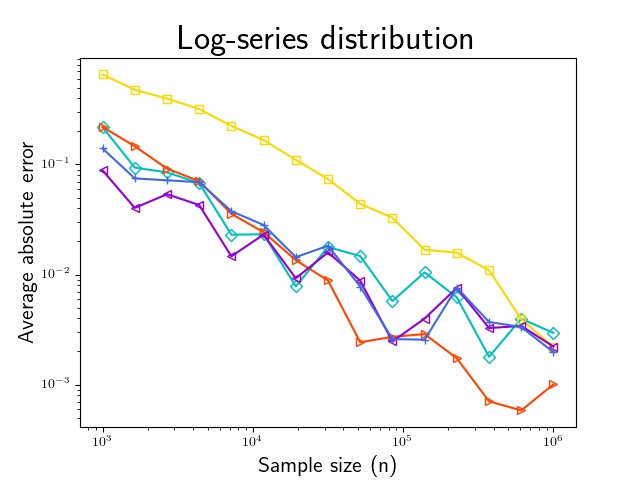}\par
\end{multicols}\vspace{-1em}
\caption{Shannon entropy estimation under absolute error}
\label{figure3}\vspace{-0.5em}
\end{figure*}

\pagebreak
\paragraph{$\boldsymbol{\alpha}$-R\'enyi entropy estimation under absolute error}\

For a distribution $p\in\Delta_\cX$, recall that the $\alpha$-power sum of $p$ is $P_{\alpha}(p)=\sum_{x} p(x)^\alpha$, implying $H_\alpha(p)=(1-\alpha)^{-1}\log(P_{\alpha}(p))$. To establish the sample-complexity upper bounds mentioned in Section~\ref{sec:rel_ren} for non-integer $\alpha$ values, \citet{AO17} first estimate the $P_{\alpha}(p)$ using the $\alpha$-power-sum estimator proposed in~\citep{jvhw}, and then substitute the estimate into the previous equation. 
 The authors of~\citep{jvhw} have implemented this two-step R\'enyi entropy estimation algorithm. In the experiments, we take a sample of size $n$, ranging from $10{,}000$ to $100{,}000$, and compare our estimator with this implementation, referred to as \emph{JVHW}, the empirical estimator, and the empirical estimator with a larger $n\log n$ sample size. Note that $\log n$ ranges from $9.2$ to $11.5$.
According to the results in~\citep{AO17}, the sample complexities for estimating $\alpha$-R\'enyi entropy are quite different for $\alpha<1$ and $\alpha>1$, hence we consider two cases: $\alpha=0.5$ and $\alpha=1.5$. 
 
 As shown in Figure~\ref{figure4} and~\ref{figure5}, our estimator clearly outperformed the one proposed by~\citep{AO17, jvhw}. 
 
We further note that for small sample sizes and several distributions, the 
 estimator in~\citep{AO17, jvhw} performs significantly worse than ours. Also, for large sample sizes, the estimators in~\citep{AO17, jvhw} degenerates to the simple empirical plug-in estimator. In comparison, our proposed estimator tracks the performance of the empirical estimator with a larger $n\log n$ sample size for nearly all the tested distributions. \vspace{-0.5em}

\begin{figure*}[h]
\begin{multicols}{3}
    \includegraphics[width=\linewidth]{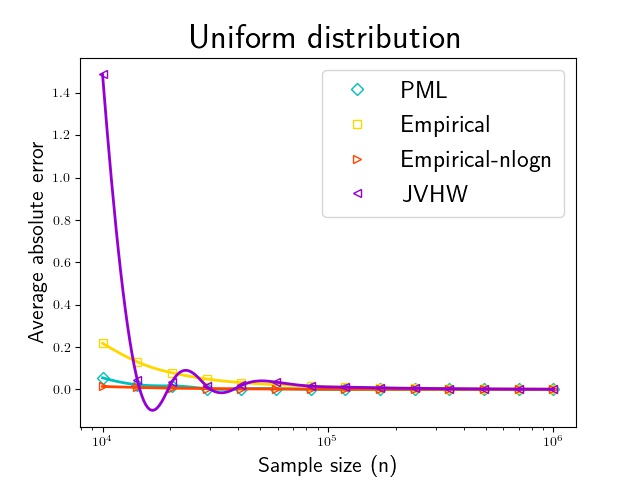}\par 
    \includegraphics[width=\linewidth]{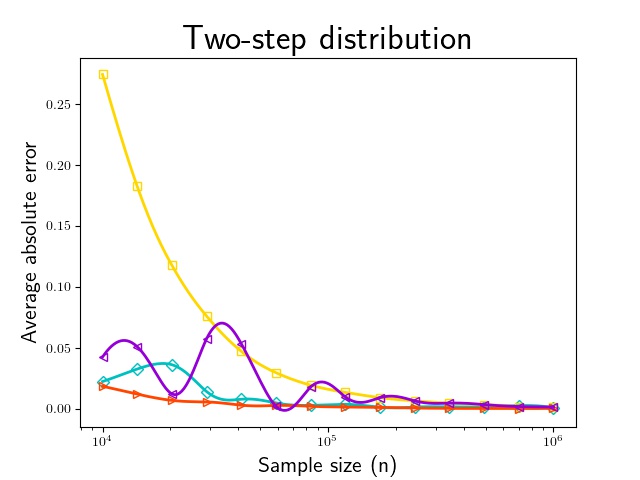}\par 
    \includegraphics[width=\linewidth]{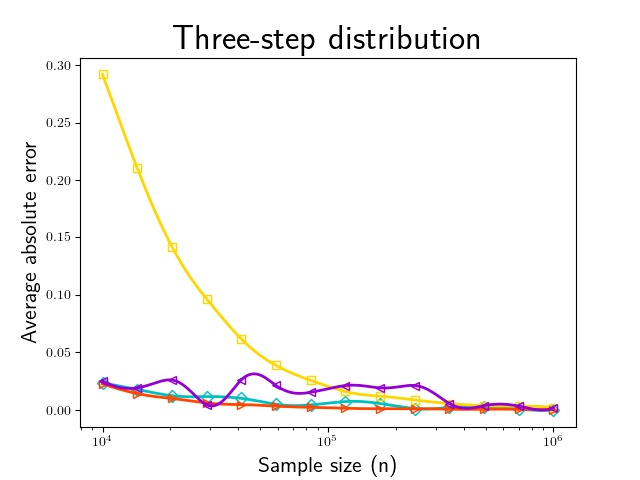}\par
     \end{multicols}\vspace{-1.5em}
\begin{multicols}{3}
    \includegraphics[width=\linewidth]{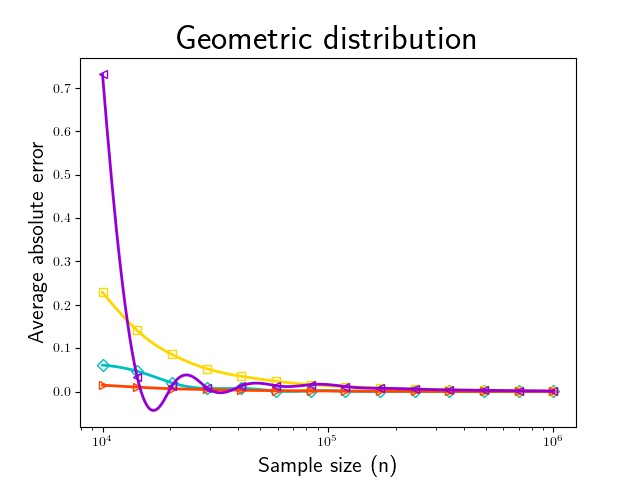}\par
   \includegraphics[width=\linewidth]{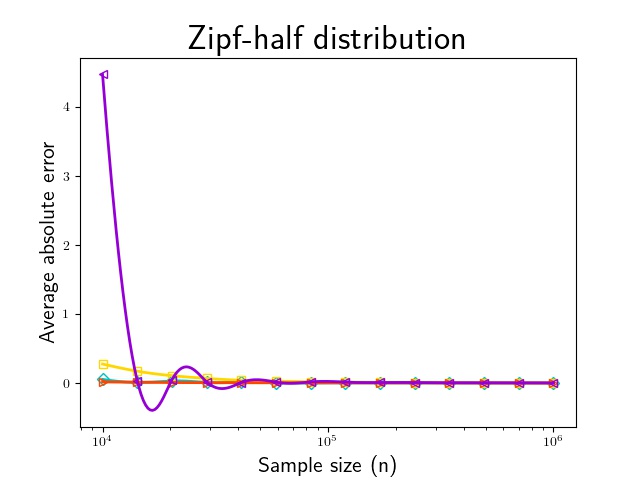}\par
   \includegraphics[width=\linewidth]{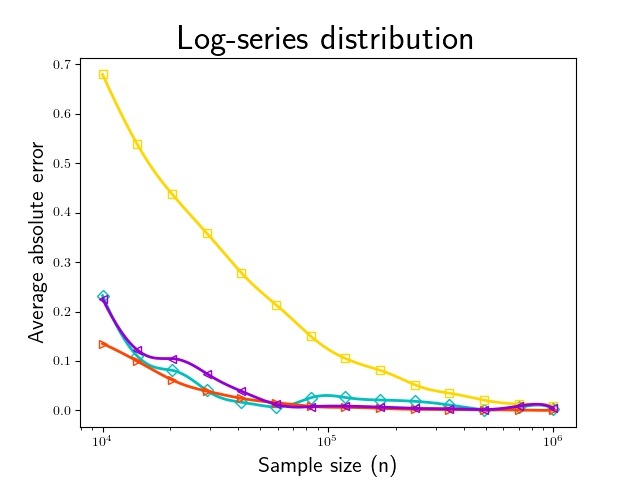}\par
\end{multicols}\vspace{-1em}
\caption{0.5-R\'enyi entropy estimation under absolute error}
\label{figure4}
\end{figure*}

\begin{figure*}[h]
\begin{multicols}{3}
    \includegraphics[width=\linewidth]{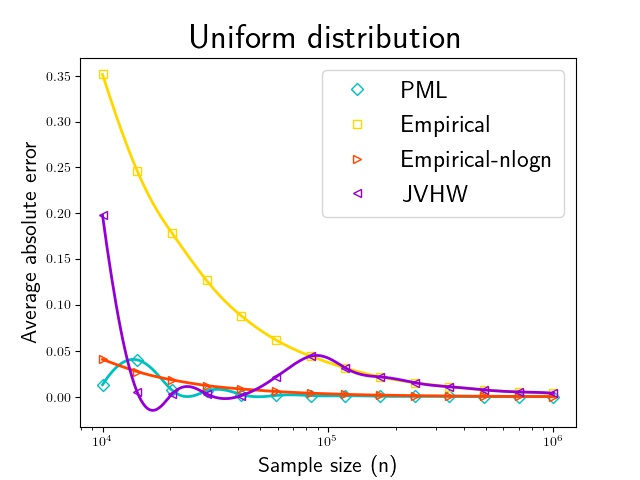}\par 
    \includegraphics[width=\linewidth]{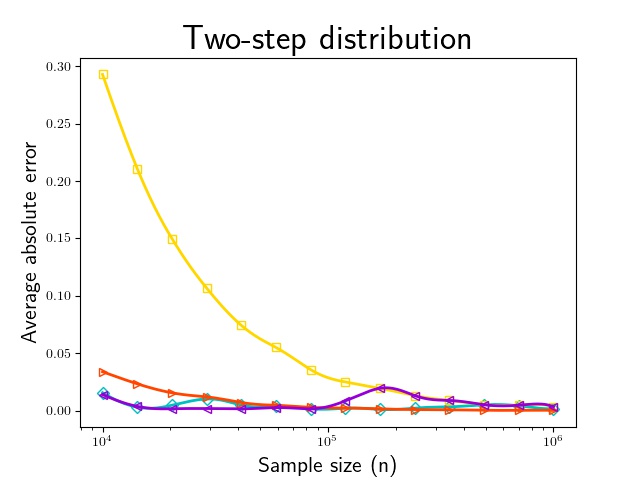}\par 
    \includegraphics[width=\linewidth]{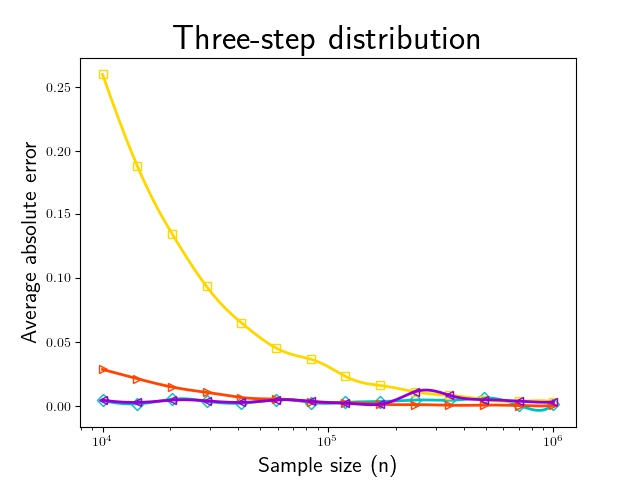}\par
     \end{multicols}\vspace{-1em}
\begin{multicols}{3}
    \includegraphics[width=\linewidth]{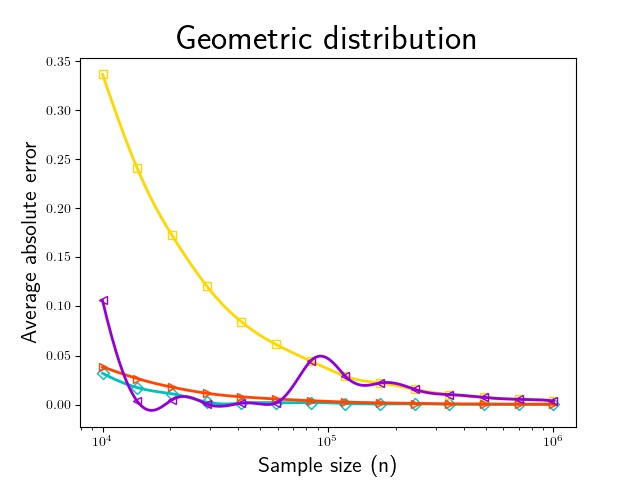}\par
   \includegraphics[width=\linewidth]{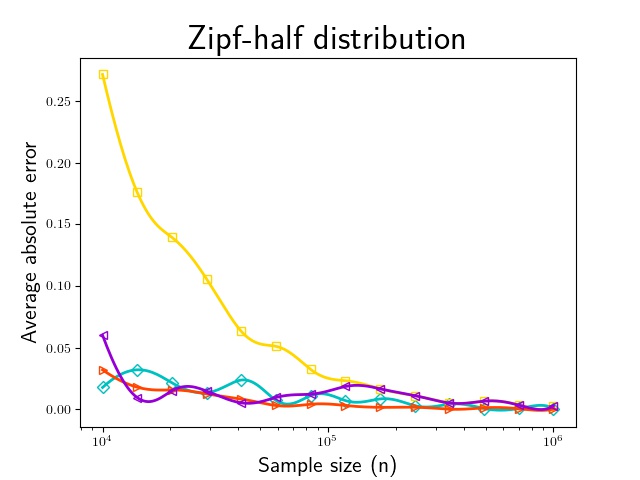}\par
   \includegraphics[width=\linewidth]{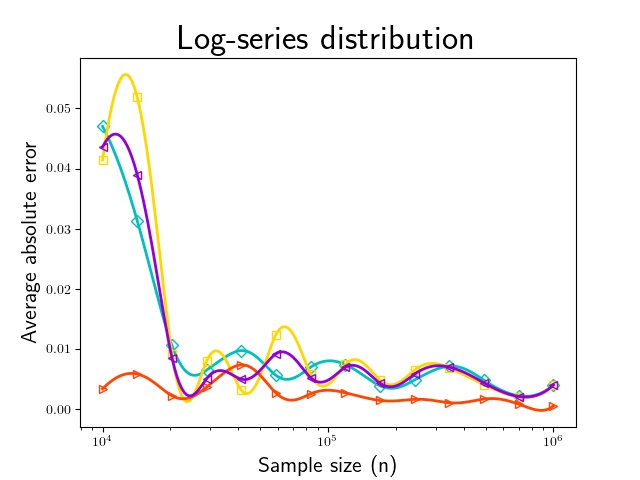}\par
\end{multicols}
\caption{1.5-R\'enyi entropy estimation under absolute error}
\label{figure5}
\end{figure*}\par

\section{ Lipschitz-property estimation}\label{sec:Lipproof}
\subsection{Proof outline of Theorem~\ref{thm:est}}
The proof proceeds as follows. 
First, fixing $n$, $\cX$, and a symmetric additive property $f$ that is $1$-Lipschitz on $(\Delta_{\cX}, R)$, we consider a related linear program defined in~\citep{VDoc12}, and lower bound the worst-case error of any estimators using the linear program's objective value, say $v$. 
Second, following the construction in~\citep{VDoc12}, we find an explicit estimator $\hat{f}^\star$ that is \emph{linear}, i.e., can be expressed as a linear combination of $\varphi_i$'s, and show optimality by upper bounding its worst-case error in terms of $v$. 
Third, we study the concentration of a general linear estimator, and through the McDiarmid's inequality~\citep{M89}, relate the tail probability of its estimate to the estimator's sensitivity to the input changes. 
Fourth, we bound the sensitivity of $\hat{f}^\star$ by the maximum difference between its consecutive coefficients, and further bound this difference by a function of $n$, showing that the estimate induced by $\hat{f}^\star$ highly concentrates around its expectation.
Finally, we invoke the result in~\citep{mmcover} that the PML-plug-in
estimator is competitive to all profile-based estimators whose
estimates are highly concentrated, concluding that PML shares the
optimality of $\hat{f}^\star$, thereby establishing Theorem~\ref{thm:est}. 
\subsection{Technical details}\label{sec:trueproof}
Let $f$ be a symmetric additive property that is $1$-Lipschitz on $(\Delta_{\cX}, R)$. Without loss of generality, we assume that $f(p)=0$ if $p(x)=1$ for some $x\in\cX$. 
\paragraph{Lower bound}%
First, fixing $n$, $\cX$, and $f$, we lower bound the worst-case error of any estimators. 

Let $u\in(0,1/2)$ be a small absolute constant. If there is an estimator $\hat{f}$ that, when given a length-$n$ sample from any distribution $p\in \Delta_\cX$, will estimate $f(p)$ up to an error of $\varepsilon$ with probability at least $1/2+u$. 
Then for any two distributions $p_1, p_2\in \Delta_\cX$ satisfying $|f(p_1)-f(p_2)|>\varepsilon$, we can use $\hat{f}$ to distinguish  $X^n\sim p_1$ from $X^n\sim p_2$, and will be correct with probability at least $1/2+u$. 

On the other hand, for any parameter $c_1\in(1/100,1/25]$ and $c_2=1/2+6c_1$, consider the corresponding linear program defined in Linear Program 6.7 in~\citep{VDoc12}, and denote by $v$ the objective value of any of its solutions. Then, Proposition 6.8 in~\citep{VDoc12} implies that we can find two distributions $p_1, p_2\in \Delta_\cX$ such that $|f(p_1)-f(p_2)|>v\cdot (1-o(1))-O(n^{-c_1} \log n)$, and no algorithm can use $\Poi(n)$ sample points to distinguish these two distributions with probability at least $1/2+u$. 

The previous reasoning yields that $v< (1+o(1))\varepsilon +O(n^{-c_1} \log n)$. By construction, $v$ is a function of $\cX, n,$ and $f$, and essentially serves as a lower bound for $\varepsilon$. 

\paragraph{Upper Bound}%
 Second, fixing $n$, $\cX$, and $f$, we construct an explicit estimator based on the previously mentioned linear program, and show optimality by upper bounding its worst-case error in terms of $v$, the linear program's objective value. 

A property estimator $\hat{f}$ is \emph{linear} if there exist real coefficients $\{\ell_i\}_{i\geq 1}$ such that the identity $\hat{f}(x^n)=\sum_{i\geq 1} \ell_i\cdot \varphi_i(x^n)$ holds for all $x^n$. The following lemma (Proposition 6.10 in~\citep{VDoc12}) bounds the worst-case error of a linear estimator when its coefficients satisfy certain conditions. 
 \begin{Lemma}\label{lem:1}
Given any positive integer $m$, and real coefficients $\{\beta_i\}_{i\geq 0}$, 
define  $\varepsilon(y):={f(y)}/{y}-e^{-my}\sum_{i\geq 0} \beta_i \cdot (my)^i/i!$. 
Let $\beta_i^\star:=\beta_{i-1}\cdot {i} /m, \forall i\geq 1$, and $\beta_0^\star:=0$. If for some $a',b',c'> 0$,
\begin{enumerate}
\item $|\varepsilon(y)|\leq a'+{b'}/{y}$,
\item $|\beta_{j}^\star-\beta_{\ell}^\star|\leq c' \sqrt{{j}/{m}}$ for any $j$ and $\ell$ such that $|j- \ell|\leq \sqrt{j}\log m$, 
\end{enumerate}
then given a sample $X^m$ from any $p\in D_\cX$, the estimator defined by $\sum_{i\geq 1}\beta_{i}^\star\cdot \varphi_i$ will estimate $f(p)$ with an accuracy of $a'+b'\cdot\size+c'\cdot\log m$ and a failure probability at most $o(1/\poly(m))$.
\end{Lemma}

Following the construction in~\citep{VDoc12} (page 124), let $z:=(z_0, z_1, \ldots)$ be the vector of coefficients induced by any solution of the dual program of the previously mentioned linear program. For our purpose, the way in which these coefficients are derived is largely irrelevant. One can show that $|z_\ell|\leq v \cdot n^{c_2},\forall \ell\ge 0$.
 Let $t_n:=2n^{-c_1}\log n$ and $\alpha\in(0, 1)$, and define 
 \[
\beta_i:=(1-e^{-t_n\alpha i})f\Paren{\frac{(i+1)\alpha}{n}}\frac{n}{(i+1)\alpha}+\sum_{\ell=0}^i z_\ell (1-t_n)^\ell\alpha^\ell (1-\alpha)^{i-\ell}\binom{i}{\ell}.
\]
 for any $i\leq n$, and $\beta_i:=\beta_n$ for $i>n$. The next lemma shows that we
 can find proper parameters $a, b,$ and $c$ to apply Lemma~\ref{lem:1} to the above construction. Specifically,
 \begin{Lemma}\label{lem:bound}
 For any $\alpha\in[1/100,1)$ and some $a'',b''\geq 0$ such that $a''+b''\size\leq v$, if $v\leq \log^2 n$ and $c_1, c_2$ satisfy $\alpha c_2+(3/2-\alpha) c_1\leq 1/4$, the two conditions in Lemma~\ref{lem:1} hold for the above construction with $m=n/\alpha$, $a'=a''+\mathcal{O}(n^{-c_1/2}\log^2 n)$, $b'=b''(1+\mathcal{O}(t_n))$, and $c'=\mathcal{O}(n^{-1/4}\log^3 n)$. 
Furthermore, for any $i\geq 0$,  we have $|\beta_i|\leq \mathcal{O}(n^{\alpha c_2+ (1-\alpha) c_1}\log^3 n)$.
 \end{Lemma}
This lemma differs from the results established in the proof of Proposition 6.19 in~\citep{VDoc12} only in the applicable range of $\alpha$, where the latter assumes that $\alpha\in[1/2, 1)$. 
For completeness, we will present a proof of Lemma~\ref{lem:bound} in Appendix~\ref{sec:lemproof}.

By Lemma~\ref{lem:1} and~\ref{lem:bound}, if $v\leq \log^2 n$, given a sample $X^{n/\alpha}$ from any $p\in \Delta_\cX$, the linear estimator $\sum_{i\geq 1}\beta_{i}^\star\cdot \varphi_i$ will estimate $f(p)$ with an accuracy of $a'+b'\size+c'\log (n/\alpha)=a''+\mathcal{O}(n^{-c_1/2}\log^2 n)+b''\size(1+\mathcal{O}(t_n))+\mathcal{O}(n^{-1/4}\log^4 n)\leq v (1+\mathcal{O}(t_n))+\mathcal{O}(n^{-c_1/2}\log^2 n)$ and a failure probability at most $o(1/\poly(n))$. 
Recall that for fixed $\cX, n,$ and $f$, the value of $v$ is a constant, thus can be computed without samples. 
 Furthermore according to the last claim in Proposition 6.19 in~\citep{VDoc12}, for $v> \log^2 n$, the estimator that always returns $0$ has an error of at most $(1+o(1))v$. Hence with high probability, the estimator $\hat{f}^{\star}:=\sum_{i\geq 1}(\beta_{i}^\star\cdot \indic_{v\leq \log^2 n})\cdot \varphi_i$ will estimate $f(p)$ up to an error of $v (1+o(1))+\mathcal{O}(t_n\log n)$, for any possible values of $v$.

 \paragraph{Concentration of linear estimators} Third, we slightly diverge from the previous discussion and study the concentration of general linear estimators. 
 
 The \emph{sensitivity} of a property estimator $\hat{f}:\cX^*\to \mathbb{R}$ for a given input size $n$ is 
 \[
 s_n(\hat{f}):= \max\Brace{f(x^n)-f(y^n): x^n \text{ and } y^n \text{ differ in one element}},
 \]
 the maximum change in its value when the input sequence is modified at exactly one location. 
 For any $p\in \Delta_\cX$ and $X^n\sim p$, the following corollary of the McDiarmid's inequality~\citep{M89} 
 relates the two-side tail probability of $\hat{f}(X^n)$ to $s_n(\hat{f})$.
 \begin{Lemma}\label{lem:sensitivitybound}
For all $t\geq 0$, we have
$
 \Pr\Paren{|\hat{f}(X^n)-\EE[\hat{f}(X^n)]|\geq t}\leq 2\exp(-2t^2\cdot(\sqrt{n}s_n(\hat{f}))^{-2}).
$
\end{Lemma}
Define $\ell_0:=0$. The next lemma bounds the sensitivity of a linear estimator $\hat{f}:=\sum_{i\geq 1} \ell_i\cdot \varphi_i$ in terms of $\max_{i\geq 1}|\ell_i-\ell_{i-1}|$,
the maximum absolute difference between its consecutive coefficients.
  \begin{Lemma}\label{lem:sen}
For any $n$ and linear estimator $\hat{f}:=\sum_{i\geq 1} \ell_i\cdot \varphi_i$, we have
$
s_n(\hat{f})\leq 2\max_{i\geq 1}|\ell_i-\ell_{i-1}|
$.
\end{Lemma}
\begin{proof}
Let $x^n$ and $y^n$ be two arbitrary sequences over $\cX$ that differ in one element. 
Let $i$ be the index where $x_i\not=y_i$. 
Then by definition, the following multiplicity equalities hold: $\mu_{x_i}(x^n)=\mu_{x_i}(y^n)+1$, $\mu_{y_i}(y^n)=\mu_{y_i}(x^n)+1$, and 
$\mu_{x}(x^n)=\mu_{x}(y^n)$ for $x\in \cX$ satisfying $x\not=x_i, y_i$. 
For simplicity of notation, let $\mu_0:= \mu_{x_i}(x^n)$, $\mu_1:= \mu_{y_i}(y^n)$, and for any $i\geq 1$, let
$\hat{f}_{i}:=\ell_{i-1}\cdot\varphi_{i-1}+\ell_{i}\cdot\varphi_{i}$. 

The first multiplicity equality implies 
$\varphi_{\mu_0}(x^n)=\varphi_{\mu_0}(y^n)+1$ and $\varphi_{\mu_0-1}(x^n)=\varphi_{\mu_0-1}(y^n)-1$.
Therefore, we have $\hat{f}_{\mu_0}(x^n)-\hat{f}_{\mu_0}(y^n)=\ell_{\mu_0}-\ell_{\mu_0-1}$. 
Similarly, the second equality implies
$\hat{f}_{\mu_1}(x^n)-\hat{f}_{\mu_1}(y^n)=-\ell_{\mu_1}+\ell_{\mu_1-1}$. 
The third equality combines these two results and yields
\[
\hat{f}(x^n)-\hat{f}(y^n) = \ell_{\mu_0}-\ell_{\mu_0-1}+(-\ell_{\mu_1}+\ell_{\mu_1-1}).
\]
Applying the triangle inequality to the right-hand side completes the proof. 
\end{proof}

By these two lemmas, we have the following result for the concentration of linear estimators. 
\begin{Corollary}\label{cor: conc}
For any $t\geq 0$, $p\in \Delta_\cX$, and $\hat{f}:=\sum_{i\geq 1} \ell_i\cdot \varphi_i$, if $X^n \sim p$, then
\[
\Pr\Paren{|\hat{f}(X^n)-\EE[\hat{f}(X^n)]|\geq t}\leq 2\min_{i\geq 1}\exp(-t^2\cdot(\sqrt{2n}(\ell_i-\ell_{i-1}))^{-2}).
\] 
\end{Corollary}

 \paragraph{Sensitivity bound}
Fourth, we bound the sensitivity of $\hat{f}^{\star}=\sum_{i\geq 1}(\beta_{i}^\star\cdot \indic_{v\leq \log^2 n})\cdot \varphi_i$. 
By Lemma~\ref{lem:sen}, it suffices to consider the absolute difference between consecutive $\beta_i^\star$'s. 
We assume $v\leq \log^2n$ and $\alpha\in [1/100, 1)$, and analyze two cases below, depending on whether $i$ is greater than $400 n^{c_1}$ or not. 

By Lemma~\ref{lem:bound}, for $i\leq 400 n^{c_1}$, we have $|\beta_i|\leq \mathcal{O}(n^{\alpha c_2+ (1-\alpha) c_1}\log^3 n)$. Define $\beta_{-1}:=0$. Then,
\begin{align*}
|\beta_{i+1}^\star-\beta_{i}^\star|
= \Abs{\frac{i+1}{n/\alpha} \beta_i- \frac{i}{n/\alpha} \beta_{i-1}}
\leq \Abs{\frac{400 n^{c_1}+1}{n/\alpha} \beta_i}+ \Abs{\frac{400 n^{c_1}}{n/\alpha} \beta_{i-1}}
\leq \mathcal{O}\Paren{n^{\alpha c_2+ (2-\alpha) c_1-1}\log^3 n}.
\end{align*}
For $i> 400 n^{c_1}$, we only need to consider $i<n$ since $\beta_{i+1}^\star=\beta_i^\star$ for all $i\geq n$. Then,
\begin{align*}
|\beta_{i+1}^\star-\beta_{i}^\star|
\overset{(a)}\leq &
\Abs{\sum_{\ell=0}^i z_\ell (1-t_n)^\ell\alpha^\ell (1-\alpha)^{i-\ell}\binom{i}{\ell}\frac{(i+1)\alpha}{n}}
+\Abs{\sum_{\ell=0}^{i-1} z_\ell (1-t_n)^\ell\alpha^\ell (1-\alpha)^{i-1-\ell}\binom{i-1}{\ell}\frac{i\alpha}{n}}\\
&+\Abs{f\Paren{\frac{(i+1)\alpha}{n}}-f\Paren{\frac{i\alpha}{n}}}+\Abs{e^{-t_n\alpha i}f\Paren{\frac{(i+1)\alpha}{n}}}
+\Abs{e^{-t_n\alpha (i-1)}f\Paren{\frac{i\alpha}{n}}}\\
\overset{(b)}\leq&
(n^{c_2}\log^2 n)\Paren{\Abs{\sum_{\ell=0}^i (1-t_n)^\ell\alpha^\ell (1-\alpha)^{i-\ell}\binom{i}{\ell}}
+ \Abs{\sum_{\ell=0}^{i-1} (1-t_n)^\ell\alpha^\ell (1-\alpha)^{i-1-\ell}\binom{i-1}{\ell}}}\\
&+\Abs{f\Paren{\frac{(i+1)\alpha}{n}}-f\Paren{\frac{i\alpha}{n}}}+\Abs{e^{-t_n\alpha i}f\Paren{\frac{(i+1)\alpha}{n}}}
+\Abs{e^{-t_n\alpha (i-1)}f\Paren{\frac{i\alpha}{n}}}\\
\overset{(c)}\leq &
(n^{c_2}\log^2 n)((1-t_n \alpha)^i+(1-t_n \alpha)^{i-1})
+\Abs{f\Paren{\frac{(i+1)\alpha}{n}}-f\Paren{\frac{i\alpha}{n}}}+2e^{-t_n\alpha (i-1)}/e\\
\overset{(d)}\leq &
2(n^{c_2}\log^2 n)\Paren{1-\frac{\log n}{50n^{c_1}}}^{400n^{c_1}}
+\Abs{f\Paren{\frac{(i+1)\alpha}{n}}-f\Paren{\frac{i\alpha}{n}}}+2n^{-2}/e\\
\overset{(e)}=&
2(n^{c_2}\log^2 n)\Paren{\Paren{1-\frac{\log n}{50n^{c_1}}}^{\frac{50n^{c_1}}{\log n}}}^{8\log n}
+\Abs{f\Paren{\frac{(i+1)\alpha}{n}}-f\Paren{\frac{i\alpha}{n}}}+2n^{-2}/e\\
\overset{(f)}\leq&
2n^{-2} +\Abs{f\Paren{\frac{(i+1)\alpha}{n}}-f\Paren{\frac{i\alpha}{n}}},
\end{align*}
where (a) follows from the triangle inequality; (b) follows from $i\leq n$, $v\leq \log^2 n$, and $|z_\ell|\leq v \cdot n^{c_2}$ for all $\ell\ge 0$; (c) follows from the binomial theorem and $|f(x)|\leq x|\log x|\leq 1/e$ for $x\in(0,1]$; (d) follows from $\alpha\geq 1/100$, $i> 400n^{c_1}$, and $t_n = 2n^{-c_1}\log n$; (e) follows from simple algebra; and (f) follows from $c_2=1/2+6c_1< 1$ and $(1-1/x)^{x}\leq e^{-1}$ for $x>1$. 

It remains to analyze the second term on the right-hand side. 
 \begin{align*}
\Abs{f\Paren{\frac{(i+1)\alpha}{n}}-f\Paren{\frac{i\alpha}{n}}}
&\overset{(a)}
= \frac{(i+1)\alpha}{n}\Abs{f\Paren{\frac{(i+1)\alpha}{n}}\frac{n}{(i+1)\alpha}-f\Paren{\frac{i\alpha}{n}}\frac{n}{(i+1)\alpha}}\\
&\overset{(b)}
= \frac{(i+1)\alpha}{n}\Abs{f\Paren{\frac{(i+1)\alpha}{n}}\frac{n}{(i+1)\alpha}-f\Paren{\frac{i\alpha}{n}}\frac{n}{i\alpha}
+f\Paren{\frac{i\alpha}{n}}\frac{n}{i(i+1)\alpha}}\\
&\overset{(c)}
\leq \frac{(i+1)\alpha}{n}\Abs{\log \frac{i+1}{i}}
+\frac{(i+1)\alpha}{n}\Abs{\frac{i\alpha}{n}\Paren{\log\Paren{\frac{i\alpha}{n}}}\frac{n}{i(i+1)\alpha}}\\
&\overset{(d)}\leq \frac{(i+1)\alpha}{n}\frac{1}{i}
+\mathcal{O}\Paren{\frac{\log n}{n}}\overset{(e)}\leq \mathcal{O}\Paren{\frac{\log n}{n}},
 \end{align*} 
where (a), (b) and (e) follows from simple algebra; (c) follows from $|f(x)/x-f(y)/y|\leq |\log(x/y)|$ for all $x,y\in (0,1]$; 
(d) follows from $\log(1+x)\leq x$ for $x\geq 0$ and $x|\log x|\leq 1/e$ for $x\in (0,1]$.

Consolidating the above inequalities and applying Lemma~\ref{lem:sen}, we get the sensitivity bound
\[
s_n(f^\star)\leq \mathcal{O}\Paren{n^{\alpha c_2+ (2-\alpha) c_1-1}\log^3 n}.
\]

\paragraph{Competitiveness of PML}
A property estimator $\hat{f}$ is \emph{profile-based} if there exists a mapping $\hat{g}$ such that $\hat{f}(x^n) = \hat{g}(\varphi(x^n))$ for all $x^n\in \cX^*$. The following lemma~\citep{A12, mmcover, D12} states that the PML estimator is competitive to other profile-based estimators.
\begin{Lemma}\label{lem:pml}
For any positive real numbers $\varepsilon$ and $\delta$, additive symmetric property $f$, and profile-based estimator $\hat{f}$,  the PML-plug-in estimator $f(p_\varphi)$ satisfies
\[
n_f(f(p_\varphi), 2\varepsilon, \delta\cdot \exp(3\sqrt{n}))\leq n_f(\hat{f}, \varepsilon, \delta).
\]
For any $\beta$-approximate PML, a similar result holds with $\delta\cdot \exp(3\sqrt{n})$ replaced by $\delta\cdot \exp(3\sqrt{n})/\beta$.
\end{Lemma}
The factor $\exp(3\sqrt{n})$  directly comes from the well-known result of~\citet{HardyRamanujan} on integer partitions, 
since there is a bijective mapping from profiles of size $n$ to partitions of integer $n$. 

\paragraph{Final analysis}
Finally, we combine the above results and establish Theorem~\ref{thm:est}.

Denote by $\tau(n)$ the previous upper bound on $s_n(f^{\star})$. 
 Let $p$ be a distribution in $\Delta_\cX$ and $X^n\sim p$. 
Let~$\gamma$ be an absolute constant in $(0,1/4)$. Then by Lemma~\ref{lem:sensitivitybound}, 
 \[
 \Pr\Paren{|\hat{f}^\star(X^n)-\EE[\hat{f}^\star(X^n)]|\geq 2 n^{1-\gamma} \tau(n)}\leq 2\exp(-8n^{1-2\gamma}). 
 \]
Let $\varepsilon>0$ be an error parameter. Assume there exists an estimator $\hat{f}$ that, when given a length-$\alpha n$ sample from any distribution $p'\in \Delta_\cX$, estimates $f(p')$ up to an absolute error $\varepsilon$ with probability at least $2/3$. Then according to the results in the upper- and lower-bound sections, with probability at most $o(1/\poly(n))$, the estimate $\hat{f}^\star(X^{n})$ will differ from  $f(p)$ by more than $v (1+o(1))+\mathcal{O}(n^{-c_1/2}\log^2 n)\leq \varepsilon (1+o(1))+\mathcal{O}(n^{-c_1/2}\log^2 n)$. \vspace{-0.3em}
In addition, by the equality $\sum_{i\geq 1}i\cdot \varphi_i(X^n)=n$ and Lemma~\ref{lem:bound}, we surely have $|\hat{f}^\star(X^n)|\leq |\sum_{i\geq 1}(i/m)\beta_{i-1}\cdot \varphi_i (X^n)|\leq \max_{i\geq 0} |\beta_i|\leq \mathcal{O}(n^{\alpha c_2+ (1-\alpha) c_1}\log^3 n)$. Multiplying this bound by $o(1/\poly(n))$ yields a quantity that is negligible comparing to $\mathcal{O}(n^{-c_1/2}\log^2 n)$. 
Therefore, the absolute bias $|\EE[\hat{f}^\star(X^n)]-f(p)|$ is at most $\varepsilon (1+o(1))+\mathcal{O}(n^{-c_1/2}\log^2 n)$. 
The triangle inequality combines this with the tail bound above:
\[
\Pr\Paren{|\hat{f}(X^{n})-f(p)|\geq \varepsilon \Paren{1+o(1)}+\mathcal{O}(n^{-c_1/2}\log^2 n)+2 n^{1-\gamma} \tau(n)}\leq 2\exp\Paren{-8n^{1-2\gamma}}.
\]
Let $\alpha=1/4$. For PML and APML estimators, set $(\gamma,c_1)$ to be $(1/4, 1/31)$ and $(0.166, 1/91)$,~respectively. Combined, the last inequality and Lemma~\ref{lem:pml} imply Theorem~\ref{thm:est}. There is a simple trade-off between $\alpha$ and $c_1$  induced by our proof technique. Specifically, if we increase the value of $c_1$ to achieve a better lower bound on $\varepsilon$, the value of $\alpha$ may need to be reduced accordingly, which enlarges the sample complexity gap between our estimators and the optimal one. For example, reducing $\alpha$ to $1/12$ and $1/22$, we can improve $c_1$ to $1/25$ and $1/20$, respectively, 
 for both PML and APML.

\section{$\boldsymbol{\alpha}$-R\'enyi entropy estimation}\label{sec:renyiproof}

For any $p\in \Delta_\cX$ and non-negative $\alpha\not=1$, the \emph{$\alpha$-R\'enyi entropy}~\citep{renyientropy} of $p$ is 
\[
H_{\alpha} (p):=\frac{1}{1-\alpha}\log P_{\alpha}(p) = \frac{1}{1-\alpha}\log\Paren{\sum_{x} p(x)^\alpha}.
\]
For $\cX$ of finite size $\size$ and any $p\in\Delta_\cX$, it is well-known that $H_{\alpha} (p)\in [0, \log \size]$.

\subsection{Proof of Theorem~\ref{thm:renyi1}: $\alpha\in(3/4, 1)$}
For $\alpha\in(3/4, 1)$, the following theorem characterizes the performance of the PML-plug-in estimator. 

For any distribution $p\in \Delta_\cX$, error parameter $\varepsilon\in(0,1)$, and sampling parameter $n$, draw a sample $X^n\sim p$ and denote its profile by $\varphi$. Then for sufficiently large $\size$,
\setcounter{Theorem}{1} 
\begin{Theorem}\label{thm:renyi1a}
For an $\alpha\in(3/4, 1)$, if $n=\Omega_\alpha(\size^{1/\alpha}/(\varepsilon^{1/\alpha} \log \size))$, 
\[
\Pr\Paren{|H_{\alpha} (p_\varphi)-H_{\alpha} (p)|\geq \varepsilon}\leq \exp(-\sqrt{n}).
\]
\end{Theorem}
We establish both this theorem and an analogous result for APML in the remaining section. Let $n$ be a sampling parameter and $p\in \Delta_\cX$ be an unknown distribution. For some $\alpha$-dependent positive constants $c_{\alpha, 1}$ and $c_{\alpha, 2}$ to be determined later, let $\tau:=c_{\alpha, 1} \log n$ and $d:=c_{\alpha, 2} \log n$ be threshold and degree parameters, respectively. Let $N, N'$ be independent Poisson random variables with mean $n$. Consider Poisson sampling with two samples drawn from $p$, first of size $N$ and the second $N'$. Suppressing the sample representations, for each $x\in \cX$, we denote{\vspace{-0.20em}} by $\mu_x$ and $\mu_x'$ the multiplicities of symbol $x$ in the first and second samples, respectively. Denote by $q(z):=\sum_{m=0}^d a_m z^m$ be the degree-$d$ min-max polynomial approximation of $z^a$ over $[0,1]$. 
We consider the following variant of the polynomial-based estimator proposed in~\citep{AO17}. 
\[
\hat{P}_\alpha := \sum_{x} \Paren{\sum_{m=0}^d \frac{a_m (2\tau)^{\alpha-m}\mu_x^{\underline{m}}}{n^\alpha}}\indic_{\mu_x\leq 4\tau}\cdot  \indic_{\mu_x'\leq \tau}+\sum_{x} \Paren{\frac{\mu_x}{n}}^\alpha \indic_{\mu_x'> \tau}. 
\]
The smaller the value of $\mu_x'$ is, the smaller we expect the value of $p(x)$ to be.{\vspace{-0.25em}} In view of this, we denote the first and second components of $\hat{P}_\alpha$ by $\hat{P}_\alpha^{(s)}$ and $\hat{P}_\alpha^{(\ell)}$, and refer to them as small- and large-probability estimators, respectively.{\vspace{-0.1em}} Note that our estimator differs from that in~\citep{AO17} only by the additional $\indic_{\mu_x\leq 4\tau}$ term, which for sufficiently large $c_{\alpha, 1}$, only modifies $\EE[\hat{P}_\alpha^{(s)}]$ by at most $n^{-2\alpha}$. 

Note that $\mu'$ naturally induces a partition over $\cX$. For symbols $x$ with $\mu_x\leq 4\tau$, we denote by 
\[
P_{a, \mu'}^{(s)}(p):=\sum_{x: \mu_x\leq 4\tau} p(x)^{\alpha} 
\]
the small-probability power sum. Analogously, for symbols $x$ with $\mu_x> 4\tau$, we denote by 
\[
P_{a, \mu'}^{(\ell)}(p):=\sum_{x: \mu_x> 4\tau} p(x)^{\alpha} 
\]
the large-probability power sum. These are random properties with non-trivial variances and are hard to be analyzed.  To address this, we apply an ``expectation trick'' and denote by $P_{a}^{(s)}(p):=\EE[P_{a, \mu'}^{(s)}(p)]$ and $P_{a}^{(\ell)}(p):=\EE[P_{a, \mu'}^{(\ell)}(p)]$ their expected values, both of which are additive symmetric properties. 

Let $\varepsilon$ be a given error parameter and $n=\Omega_\alpha(\size^{1/\alpha}/(\varepsilon^{1/\alpha} \log \size))$  be a sampling parameter. First we consider the small probability estimator. By the results in~\cite{AO17}, for sufficiently large $c_{\alpha,1}$, the bias of $\hat{P}_\alpha^{(s)}$ in estimating ${P}_\alpha^{(s)}(p)$ satisfies
\[
|\EE[\hat{P}_\alpha^{(s)}]-{P}_\alpha^{(s)}(p)|
\leq \mathcal{O}_\alpha(1)\cdot  P_\alpha(p) \Paren{\frac{\size}{n\log n}}^\alpha+n^{-\alpha}
\leq  \varepsilon P_\alpha(p),
\]
where we have used $n^{-\alpha}=  \mathcal{O}_\alpha(\varepsilon \size^{-1} (\log \size)^{\alpha}) \leq \varepsilon P_\alpha(p)$.
To show concentration, we bound the sensitivity of estimator $\hat{P}_\alpha^{(s)}$. 
For $m\geq 0$, we can bound the coefficients of $q(x)$ as follows.
\[
|a_m|= \mathcal{O}_\alpha((\sqrt{2}+1)^d)= \mathcal{O}_\alpha(n^{c_{\alpha, 2}}). 
\]

Therefore by definition, changing one point
in the sample changes the value of  $\hat{P}_\alpha^{(s)}$ by at most 
\[
2\Paren{\sum_{m=0}^d \frac{|a_m| (2\tau)^{\alpha-m}(4\tau)^{\underline{m}}}{n^\alpha}}
\leq \sum_{m=0}^d \frac{|a_m| (2\tau)^{\alpha}2^{m+1}}{n^\alpha}
=  \mathcal{O}_\alpha\Paren{n^{2c_{\alpha, 2}-\alpha}(\log n)^{\alpha}}.
\]
Let $\lambda\in(0,1/4)$ be an arbitrary absolute constant. For sufficiently small $c_{\alpha, 2}$, the right-hand side is at most $\mathcal{O}_\alpha\Paren{n^{\lambda-\alpha}}$. The McDiarmid's inequality together with the concentration of Poisson random variables implies that 
for all $\varepsilon\geq 0$,
\[
 \Pr\Paren{|\hat{P}_\alpha^{(s)}-\EE[\hat{P}_\alpha^{(s)}]|\geq \varepsilon P_\alpha(p)}\leq 2\exp(-\Omega_\alpha(\varepsilon^{2} P_\alpha^{2}(p) n^{2\alpha-1-2\lambda})).
\]
Note that $n=\Omega_\alpha(\size^{1/\alpha}/(\varepsilon^{1/\alpha} \log \size))$ and $P_\alpha(p)\geq 1$, which follows from the fact that $z^\alpha$ is a concave function over $[0,1]$ for $\alpha\in(0,1)$. Hence we obtain 
\[
\Pr\Paren{|\hat{P}_\alpha^{(s)}-\EE[\hat{P}_\alpha^{(s)}]|\geq \varepsilon P_\alpha(p)}\leq 3
\exp\Paren{-\Omega_\alpha\Paren{\varepsilon^2 n^{2\alpha-1-2\lambda}}}. 
\]
For $\alpha>3/4$, we can set $\lambda=(4\alpha-3)/8$. Direct calculation shows that for sufficiently large $\size$, the right-hand side is no more than $\exp(-8\sqrt{n})$. Analogously, we can show that for $\alpha>5/6$, the probability bound can be improved to $\exp(-\Theta({n}^{2/3}))$.

Second, we consider the large probability estimator. To begin with, we set $n=\Theta_\alpha(\size^{1/3})$. By the results in~\cite{AO17},  for sufficiently large $c_{\alpha,1}$, the bias of $\hat{P}_\alpha^{(\ell)}$ in estimating ${P}_\alpha^{(\ell)}(p)$ satisfies
\[
|\EE[\hat{P}_\alpha^{(\ell)}]-{P}_\alpha^{(\ell)}(p)|
\leq  \mathcal{O}_\alpha\Paren{\frac{P_\alpha(p)}{\tau}}+\frac{1}{n^{\alpha}},
\]
which, for sufficiently large $\size$, is at most $\varepsilon P_\alpha(p)$.
Under the same conditions, the variance of $\hat{P}_\alpha^{(\ell)}$ is at most
\[
\Var(\hat{P}_\alpha^{(\ell)})
\leq  \mathcal{O}_\alpha\Paren{\sum_x \frac{p(x)^{2\alpha}}{\tau}}+\frac{1}{n^{2\alpha}}\leq \frac{(\varepsilon P_\alpha(p))^2}{3}.
\]
Then, the Chebyshev's inequality yields
\[
\Pr\Paren{|\EE[\hat{P}_\alpha^{(\ell)}]-\hat{P}_\alpha^{(\ell)}|\geq \varepsilon P_\alpha(p)}
\leq  \frac{1}{3}.
\]
The triangle inequality combines this tail bound with the above bias bound and implies
\[
\Pr\Paren{|{P}_\alpha^{(\ell)}(p)-\hat{P}_\alpha^{(\ell)}|\geq 2\varepsilon P_\alpha(p)}
\leq  \frac{1}{3}.
\]
Therefore, utilizing the median trick and $\alpha<1$, we can construct another estimator $\hat{P}_\alpha^{(\ell, 1)}$ 
that takes a sample of size $n=\Omega_\alpha(\size^{1/\alpha}/(\varepsilon^{1/\alpha} \log \size))$, and satisfies
\[
\Pr\Paren{|{P}_\alpha^{(\ell)}(p)-\hat{P}_\alpha^{(\ell, 1)}|\geq 2\varepsilon P_\alpha(p)}
\leq 2\exp(-\Omega_\alpha(n/\size^{1/3})))
\leq 2\exp(-\Theta({n}^{2/3})).
\]
Recall that ${P}_\alpha(p)={P}_\alpha^{(s)}(p)+{P}_\alpha^{(\ell)}(p)$. By the union bound and the triangle inequality, under Poisson sampling with parameter $n=\Theta_\alpha(\size^{1/\alpha}/(\varepsilon^{1/\alpha} \log \size))$, 
\[
\Pr\Paren{|{P}_\alpha(p)-(\hat{P}_\alpha^{(s)}+\hat{P}_\alpha^{(\ell, 1)})|\geq 4\varepsilon P_\alpha(p)}
\leq \exp(-8\sqrt{n}).
\]
Since both $N$ and $N'$ are Poisson random variables with mean $n$, we must have $N+N'\sim \Poi(2n)$, implying that $\Pr(N+N'=2n)=e^{-2n} (2n)^{2n}/(2n)!$. A variant of the well-known Stirling's formula states that $m!\geq e m^{m+1/2} e^{-m}$ for all positive integers $m$. We obtain $\Pr(N+N'=2n)\geq e^{-2n} (2n)^{2n}\cdot(e (2n)^{2n+1/2} e^{-2n})^{-1}\geq 1/(e\sqrt{2n})> 1/(4n)$. Hence, under fixed sampling with a sample size of $2n$, the estimator $\hat{P}_{\alpha}^{(1)}:=(\hat{P}_\alpha^{(s)}+\hat{P}_\alpha^{(\ell, 1)})$ satisfies
\[
\Pr\Paren{|{P}_\alpha(p)-\hat{P}_{\alpha}^{(1)}|\geq 4\varepsilon P_\alpha(p)}
\leq 4n\exp(-8\sqrt{n}).
\]
Replacing $n$ with $n/2$ and $\varepsilon$ with $\varepsilon/4$, the \emph{sufficiency of profiles}~\citep{AO17} implies the existence of a profile-based estimator $\hat{P}_{\alpha}^{\star}$ such that for any $p\in \Delta_\cX$,
\[
\Pr_{X^n \sim p}\Paren{|{P}_\alpha(p)-\hat{P}_{\alpha}^{\star}(X^n)|\geq \varepsilon P_\alpha(p)}
\leq 2n\exp(-4\sqrt{2n})<\exp(-4\sqrt{n}).
\]
Let $\delta$ denote the quantity on the right-hand side. 
For any $x^n$ with profile $\varphi$ satisfying both $p(\varphi)>\delta$, we must have
$|\hat{P}_\alpha^\star(x^n)-P_\alpha(p)|\leq \varepsilon P_\alpha(p)$. By definition, we also have $p_\varphi(\varphi)\geq p(\varphi)>\delta$ and hence $|\hat{P}_\alpha^\star(x^n)-P_\alpha(p_\varphi)|\leq \varepsilon P_\alpha(p_\varphi)$. 
For any $\varepsilon\in (0, 1/2)$, simple algebra combines the two property inequalities and yields
\[
|P_\alpha(p)-P_\alpha(p_\varphi)|\leq 2\varepsilon P_\alpha(p).
\]
On the other hand, for a sample $X^n\sim p$ with profile $\varphi'$, the probability that we have $p(\varphi')\leq \delta$ is at most $\delta$ times the cardinality of the set $\Phi^n:=\{\varphi(x^n): x^n\in \cX^n\}$. The latter quantity corresponds to the number of integer partitions of $n$, which, by the well-known result of~\citet{HardyRamanujan}, is at most $\exp(3\sqrt{n})$. Hence, the probability that $p(\varphi')\leq \delta$ is upper bounded by $\exp(-\sqrt{n})$. 
To conclude, we have shown that
\[
\Pr\Paren{|P_\alpha(p)-P_\alpha(p_\varphi)|\geq 2\varepsilon P_\alpha(p)}\leq \exp(-\sqrt{n}).
\]
In terms of R\'enyi entropy values, applying the inequality $e^{z}-1\geq 1-e^{-z} \geq z/2$ for all $z\geq 0$, we establish that for $\alpha>3/4$ and $n=\Omega_\alpha(\size/(\varepsilon^{1/\alpha} \log \size))$,
\[
\Pr\Paren{|H_\alpha(p)-H_\alpha(p_\varphi)|\geq \varepsilon}=\Pr\Paren{P_\alpha(p_\varphi)e^{-(\alpha-1)\varepsilon}\le P_\alpha(p)\leq P_\alpha(p_\varphi)e^{(\alpha-1)\varepsilon}}
\le \exp(-\sqrt{n}).
\]

\subsection{Proof of Theorem~\ref{thm:renyi2}: Non-integer $\alpha>1$}
The proof of the following theorem is essentially the same as that shown in the previous section. However, for completeness, we still include a full-length proof. 

For any distribution $p\in \Delta_\cX$, error parameter $\varepsilon\in(0,1)$, absolute constant $\lambda\in(0,0.1)$, and sampling parameter $n$, draw a sample $X^n\sim p$ and denote its profile by $\varphi$. Then for sufficiently large $\size$,
\setcounter{Theorem}{2} 
\begin{Theorem}\label{thm:renyi2a}
For a non-integer $\alpha>1$, if $n=\Omega_\alpha(\size/(\varepsilon^{1/\alpha} \log \size))$, 
\[
\Pr\Paren{|H_{\alpha} (p_\varphi)-H_{\alpha} (p)|\geq \varepsilon}\leq \exp(-n^{1-\lambda}).
\]
\end{Theorem}

We establish this theorem in the remaining section. Let $n$ be a sampling parameter and $p\in \Delta_\cX$ be an unknown distribution. For some $\alpha$-dependent positive constants $c_{\alpha, 1}$ and $c_{\alpha, 2}$ to be determined later, let $\tau:=c_{\alpha, 1} \log n$ and $d:=c_{\alpha, 2} \log n$ be threshold and degree parameters, respectively. Let $N, N'$ be independent Poisson random variables with mean $n$. Consider Poisson sampling with two samples drawn from $p$, first of size $N$ and the second $N'$. Suppressing the sample representations, for each $x\in \cX$, we denote{\vspace{-0.20em}} by $\mu_x$ and $\mu_x'$ the multiplicities of symbol $x$ in the first and second samples, respectively. Denote by $q(z):=\sum_{m=0}^d a_m z^m$ be the degree-$d$ min-max polynomial approximation of $z^a$ over $[0,1]$. 
We consider the following variant of the estimator proposed in~\citep{AO17}. 
\[
\hat{P}_\alpha := \sum_{x} \Paren{\sum_{m=0}^d \frac{a_m (2\tau)^{\alpha-m}\mu_x^{\underline{m}}}{n^\alpha}}\indic_{\mu_x\leq 4\tau}\cdot  \indic_{\mu_x'\leq \tau}+\sum_{x} \Paren{\frac{\mu_x}{n}}^\alpha \indic_{\mu_x'> \tau}. 
\]
The smaller the value of $\mu_x'$ is, the smaller we expect the value of $p(x)$ to be.{\vspace{-0.25em}} In view of this, we denote the first and second components of $\hat{P}_\alpha$ by $\hat{P}_\alpha^{(s)}$ and $\hat{P}_\alpha^{(\ell)}$, and refer to them as small- and large-probability estimators, respectively.{\vspace{-0.1em}} Note that our estimator differs from that in~\citep{AO17} only by the additional $\indic_{\mu_y\leq 4\tau}$ term, which for sufficiently large $c_{\alpha, 1}$, only modifies $\EE[\hat{P}_\alpha^{(s)}]$ by at most $n^{-2\alpha}$. 

Note that $\mu'$ naturally induces a partition over $\cX$. For symbols $x$ with $\mu_x\leq 4\tau$, we denote by 
\[
P_{a, \mu'}^{(s)}(p):=\sum_{x: \mu_x\leq 4\tau} p(x)^{\alpha} 
\]
the small-probability power sum. Analogously, for symbols $x$ with $\mu_x> 4\tau$, we denote by 
\[
P_{a, \mu'}^{(\ell)}(p):=\sum_{x: \mu_x> 4\tau} p(x)^{\alpha} 
\]
the large-probability power sum. These are random properties with non-trivial variances and are hard to be analyzed.  To address this, we apply an ``expectation trick'' and denote by $P_{a}^{(s)}(p):=\EE[P_{a, \mu'}^{(s)}(p)]$ and $P_{a}^{(\ell)}(p):=\EE[P_{a, \mu'}^{(\ell)}(p)]$ their expected values, both of which are additive symmetric properties. 

Let $\varepsilon$ be a given error parameter and $n=\Omega_\alpha(\size/(\varepsilon^{1/\alpha} \log \size))$  be a sampling parameter. First we consider the small probability estimator. By the results in~\cite{AO17}, for sufficiently large $c_{\alpha,1}$, the bias of $\hat{P}_\alpha^{(s)}$ in estimating ${P}_\alpha^{(s)}(p)$ satisfies
\[
|\EE[\hat{P}_\alpha^{(s)}]-{P}_\alpha^{(s)}(p)|
\leq \mathcal{O}_\alpha(1)\cdot  P_\alpha(p) \Paren{\frac{\size}{n\log n}}^\alpha+n^{-\alpha}
\leq  \varepsilon P_\alpha(p),
\]
where we have used $n^{-\alpha}=  \mathcal{O}_\alpha(\varepsilon \size^{-\alpha} (\log \size)^{\alpha}) \leq \varepsilon P_\alpha(p)$. 
To show concentration, we bound the sensitivity of estimator $\hat{P}_\alpha^{(s)}$. 
For $m\geq 0$, we can bound the coefficients of $q(x)$ as follows.
\[
|a_m|\leq \mathcal{O}_\alpha((\sqrt{2}+1)^d)= \mathcal{O}_\alpha(n^{c_{\alpha, 2}}). 
\]

Therefore by definition, changing one point
in the sample changes the value of  $\hat{P}_\alpha^{(s)}$ by at most 
\[
2\Paren{\sum_{m=0}^d \frac{|a_m| (2\tau)^{\alpha-m}(4\tau)^{\underline{m}}}{n^\alpha}}
\leq \sum_{m=0}^d \frac{|a_m| (2\tau)^{\alpha}2^{m+1}}{n^\alpha}
\leq  \mathcal{O}_\alpha\Paren{n^{2c_{\alpha, 2}-\alpha}(\log n)^{\alpha}}.
\]
Let $\lambda\in(0,1/4)$ be an arbitrary absolute constant. For sufficiently small $c_{\alpha, 2}$, the right-hand side is at most $\mathcal{O}_\alpha\Paren{n^{\lambda-\alpha}}$. The McDiarmid's inequality together with the concentration of Poisson random variables implies that 
for all $\varepsilon\geq 0$,
\[
 \Pr\Paren{|\hat{P}_\alpha^{(s)}-\EE[\hat{P}_\alpha^{(s)}]|\geq \varepsilon P_\alpha(p)}\leq 2\exp(-\Omega_\alpha(\varepsilon^{2} P_\alpha^{2}(p) n^{2\alpha-1-2\lambda})).
\]
Note that $n=\Omega_\alpha(\size/(\varepsilon^{1/\alpha} \log \size))$ and $P_\alpha(p)\geq \size^{1-\alpha}$. Hence we obtain 
\[
\Pr\Paren{|\hat{P}_\alpha^{(s)}-\EE[\hat{P}_\alpha^{(s)}]|\geq \varepsilon P_\alpha(p)}\leq 3\exp\Paren{-\Omega_\alpha(\varepsilon^{2} \size^{2-2\alpha} n^{2\alpha-1-2\lambda})}. 
\]
By simple algebra, for sufficiently large $\size$, the right-hand side is at most $\exp(-n^{1-3\lambda})$. 

Second, we consider the large probability estimator. To begin with, we set $n=\Theta_\alpha(\size^{\lambda})$. By the results in~\cite{AO17},  for sufficiently large $c_{\alpha,1}$, the bias of $\hat{P}_\alpha^{(\ell)}$ in estimating ${P}_\alpha^{(\ell)}(p)$ satisfies
\[
|\EE[\hat{P}_\alpha^{(\ell)}]-{P}_\alpha^{(\ell)}(p)|
\leq  \mathcal{O}_\alpha\Paren{\frac{P_\alpha(p)}{\tau}}+\frac{1}{n^{4\alpha}},
\]
which, for sufficiently large $\size$, is at most $\varepsilon P_\alpha(p)$.
Under the same conditions, the variance of $\hat{P}_\alpha^{(\ell)}$ is at most
\[
\Var(\hat{P}_\alpha^{(\ell)})
\leq  \mathcal{O}_\alpha\Paren{\sum_x \frac{p(x)^{2\alpha}}{\tau}}+\frac{1}{n^{8\alpha}}\leq \frac{(\varepsilon P_\alpha(p))^2}{3}.
\]
Then, the Chebyshev's inequality yields
\[
\Pr\Paren{|\EE[\hat{P}_\alpha^{(\ell)}]-\hat{P}_\alpha^{(\ell)}|\geq \varepsilon P_\alpha(p)}
\leq  \frac{1}{3}.
\]
The triangle inequality combines this tail bound with the above bias bound and implies
\[
\Pr\Paren{|{P}_\alpha^{(\ell)}(p)-\hat{P}_\alpha^{(\ell)}|\geq 2\varepsilon P_\alpha(p)}
\leq  \frac{1}{3}.
\]
Therefore, utilizing the median trick, we can construct another estimator $\hat{P}_\alpha^{(\ell, 1)}$ 
that takes a sample of size $n=\Omega_\alpha(\size/(\varepsilon^{1/\alpha} \log \size))$, and for sufficiently large $\size$, satisfies
\[
\Pr\Paren{|{P}_\alpha^{(\ell)}(p)-\hat{P}_\alpha^{(\ell, 1)}|\geq 2\varepsilon P_\alpha(p)}
\leq 2\exp(-\Omega_\alpha(n/\size^\lambda))
\leq \exp(-n^{1-2\lambda}).
\]
Recall that ${P}_\alpha(p)={P}_\alpha^{(s)}(p)+{P}_\alpha^{(\ell)}(p)$. By the union bound and the triangle inequality, under Poisson sampling with parameter $n=\Omega_\alpha(\size/(\varepsilon^{1/\alpha} \log \size))$, 
\[
\Pr\Paren{|{P}_\alpha(p)-(\hat{P}_\alpha^{(s)}+\hat{P}_\alpha^{(\ell, 1)})|\geq 4\varepsilon P_\alpha(p)}
\leq \exp(-n^{1-3\lambda}).
\]
Since both $N$ and $N'$ are Poisson random variables with mean $n$, we must have $N+N'\sim \Poi(2n)$, implying that $\Pr(N+N'=2n)=e^{-2n} (2n)^{2n}/(2n)!$. A variant of the well-known Stirling's formula states that $m!\geq e m^{m+1/2} e^{-m}$ for all positive integers $m$. We obtain $\Pr(N+N'=2n)\geq e^{-2n} (2n)^{2n}\cdot(e (2n)^{2n+1/2} e^{-2n})^{-1}\geq 1/(e\sqrt{2n})> 1/(4n)$. Hence, under fixed sampling with a sample size of $2n$, the estimator $\hat{P}_{\alpha}^{(1)}:=(\hat{P}_\alpha^{(s)}+\hat{P}_\alpha^{(\ell, 1)})$ satisfies
\[
\Pr\Paren{|{P}_\alpha(p)-\hat{P}_{\alpha}^{(1)}|\geq 4\varepsilon P_\alpha(p)}
\leq 4n\exp(-n^{1-3\lambda}).
\]
Replacing $\varepsilon$ with $\varepsilon/4$ and $\lambda$ with $\lambda/5$, the sufficiency of profiles implies the existence of a profile-based estimator $\hat{P}_{\alpha}^{\star}$ such that for sufficiently large $\size$ and any $p\in \Delta_\cX$,
\[
\Pr_{X^n \sim p}\Paren{|{P}_\alpha(p)-\hat{P}_{\alpha}^{\star}(X^n)|\geq \varepsilon P_\alpha(p)}
\leq 4n\exp(-n^{1-3\lambda/5})<\exp(-n^{1-4\lambda/5}).
\]
Let $\delta$ denote the quantity on the right-hand side. 
For any $x^n$ with profile $\varphi$ satisfying both $p(\varphi)>\delta$, we must have
$|\hat{P}_\alpha^\star(x^n)-P_\alpha(p)|\leq \varepsilon P_\alpha(p)$. By definition, we also have $p_\varphi(\varphi)\geq p(\varphi)>\delta$ and hence $|\hat{P}_\alpha^\star(x^n)-P_\alpha(p_\varphi)|\leq \varepsilon P_\alpha(p_\varphi)$. 
For any $\varepsilon\in (0, 1/2)$, simple algebra combines the two property inequalities and yields
\[
|P_\alpha(p)-P_\alpha(p_\varphi)|\leq 2\varepsilon P_\alpha(p).
\]
On the other hand, for a sample $X^n\sim p$ with profile $\varphi'$, the probability that we have $p(\varphi')\leq \delta$ is at most $\delta$ times the cardinality of the set $\Phi^n:=\{\varphi(x^n): x^n\in \cX^n\}$. The latter quantity corresponds to the number of integer partitions of $n$, which, by the well-known result of~\citet{HardyRamanujan}, is at most $\exp(3\sqrt{n})$. Hence, the probability that $p(\varphi')\leq \delta$ is upper bounded by $\exp(-n^{1-\lambda})$. 
To conclude, we have shown that
\[
\Pr\Paren{|P_\alpha(p)-P_\alpha(p_\varphi)|\geq 2\varepsilon P_\alpha(p)}\leq \exp(-n^{1-\lambda}).
\]
In terms of R\'enyi entropy values, applying the inequality $e^{z}-1\geq 1-e^{-z} \geq z/2$ for all $z\geq 0$, we establish that for $n=\Omega_\alpha(\size/(\varepsilon^{1/\alpha} \log \size))$,
\[
\Pr\Paren{|H_\alpha(p)-H_\alpha(p_\varphi)|\geq \varepsilon}=\Pr\Paren{P_\alpha(p_\varphi)e^{-(\alpha-1)\varepsilon}\le P_\alpha(p)\leq P_\alpha(p_\varphi)e^{(\alpha-1)\varepsilon}}
\le \exp(-n^{1-\lambda}).
\]

\subsection{Proof of Theorem~\ref{thm:renyi3}: Integer $\alpha>1$}\label{sec:thm4proof}
For an integer $\alpha>1$, the following theorem characterizes the performance of the PML-plug-in estimator. For any $p\in \Delta_\cX$, $\varepsilon\in(0, 1)$, and a sample $X^n\sim p$ with profile $\varphi$,
\setcounter{Theorem}{3} 
\begin{Theorem}
If $n=\Omega_\alpha(\size^{1-1/\alpha} (\varepsilon^{2}|\log{\varepsilon}|)^{-(1+\alpha)})$ and $H_{\alpha} (p)\leq (\log n)/4$, 
\[
\Pr(|H_{\alpha} (p_\varphi)-H_{\alpha} (p)|\geq \varepsilon)\leq {1}/{3}.
\]
\end{Theorem}

Due to the lower bounds in~\citep{AO17}, for all possible values of $\alpha$, the sample complexity of the PML plug-in estimator has the optimal dependency in $\size$. 
The remaining section is devoted to proving the above theorem. Note that estimating the R\'enyi entropy $H_{\alpha} (p)$ to an additive error is equivalent to estimating the power sum $P_{\alpha}(p)$ to a corresponding multiplicative error. Given this fact, we consider the estimator $\hat{P}_\alpha$ in~\citep{AO17} that maps each sequence $x^n\in \cX^*$ to 
\[
\hat{P}_\alpha(x^n):= \sum_x \frac{\mu_x(x^n)^{\underline{\alpha}}}{n^{\underline{\alpha}}},
\]
where for any real number $z$, the expression $z^{\underline{\alpha}}$ denotes the falling factorial of $z$ to the power $\alpha$. 
For a sample $X^n\sim p$, we have $\EE[\hat{P}_\alpha(X^n)]=P_{\alpha}(p)$. The following lemma~\citep{OS17, AO17} states that 
$\hat{P}_\alpha(X^n)$ often estimates $P_{\alpha}(p)$ to a small multiplicative error when $n$ is large.
\begin{Lemma}
Under the above conditions, for any $\varepsilon, n> 0$, 
\[
\Pr\Paren{|\hat{P}_\alpha(X^n)-P_\alpha(p)|\geq \varepsilon P_\alpha(p)}= 
\mathcal{O}_\alpha( \varepsilon^{-2}n^{-1}(P_\alpha(p))^{-1/\alpha}). 
\]

For sufficiently large $n= \Omega_\alpha(\size^{(\alpha-1)/\alpha})$, this inequality together with $P_\alpha(p)\leq \size^{1-\alpha}$ implies that
\[
\Pr\Paren{|\hat{P}_\alpha(X^n)-P_\alpha(p)|\geq \frac{1}{2}\cdot P_\alpha(p)}\leq \frac{1}{4}.
\]
\end{Lemma}
The following corollary is a consequence of the above lemma, the sufficiency of profiles, and the standard median trick. 
\begin{Corollary}\label{cor:2}
Under the above conditions, there is an estimator $\hat{P}_\alpha^\star$ such that for any $\varepsilon, n> 0$,
\[
\Pr\Paren{|\hat{P}_\alpha^\star(X^n)-P_\alpha(p)|\geq\varepsilon P_\alpha(p)}\leq 
2\exp\Paren{-\Omega_\alpha(\varepsilon^{2}n(P_\alpha(p))^{1/\alpha})}. 
\]
In addition, the estimator $\hat{P}_\alpha^\star$ is profile-based.
\end{Corollary}

For simplicity, suppress $X^n$ in $p_{\mu}(X^n)$. Since the profile probability $p(\varphi)$ is invariant to symbol permutation, for our purpose, we can assume that $p_{\mu}(y)\leq p_{\mu}(z)$ iff $p_{\varphi}(x)\leq p_{\varphi}(y)$, for all $x, y\in \cX$.
Under this assumption, the following lemma~\citep{O11,A17} relates $p_{\varphi}$ to $p_{\mu}$. 
\begin{Lemma}\label{lem:pmlcon}
For a distribution $p$ and sample $X^n\sim p$ with profile $\varphi$,
\[
\Pr\Paren{\max_x|p_{\varphi}(x)-p_\mu(x)|>\frac{2\log n}{n^{1/4}}}= \mathcal{O}\Paren{\frac{1}{n}}.
\]
\end{Lemma}
Consider $\varepsilon\in(0,1/2)$ and $x^n$ satisfying $|\hat{P}_\alpha^\star(x^n)-P_\alpha(p)|\leq \varepsilon P_\alpha(p)$. If we further have $P_\alpha(p)\geq 2(n^{1/4}(4\log n)^{-1})^{1-\alpha}$ and $\max_y|p_\varphi(y)-p_\mu(y)|\leq {2(\log n)n^{-1/4}}$, then,
\begin{align*}
\frac{P_\alpha(p)}{2}
\overset{(a)}{\leq} \hat{P}_\alpha(x^n)
\overset{(b)}{\leq}  P_\alpha(p_\mu)
\overset{(c)}{\leq}  2^{1+\alpha} P_\alpha(p_\varphi),
\end{align*}
where $(a)$ follows from the above assumptions; $(b)$ follows from $A^{\underline{B}}\leq A^B$ for any $A, B\geq 0$; 
and $(c)$ follows from the reasoning below. 
\begin{itemize}
\item Let $S$ denote the the collection of symbols $x$ such that $p_\mu(x)\leq {4(\log n)n^{-1/4}}$. Then a convexity argument yields $\sum_{x\in S}  \Paren{p_\mu(x)}^{\alpha}\leq (n^{1/4}(4\log n)^{-1})^{1-\alpha}$. 
\item Using $(a)$, $(b)$, and $P_\alpha(p)\geq 4(n^{1/4}(4\log n)^{-1})^{1-\alpha}$, we immediately obtain 
$P_\alpha(p_\mu)\geq 2(n^{1/4}(4\log n)^{-1})^{1-\alpha}$ and thus 
$2\sum_{x\in S}  \Paren{p_\mu(x)}^{\alpha}\leq P_\alpha(p_\mu) \leq 2\sum_{x\not\in S}  \Paren{p_\mu(x)}^{\alpha}$.
\item For any symbol $x\not\in S$, we have $p_\mu(x)> {4(\log n)n^{-1/4}}$. This together with the assumption that $\max_x|p_\varphi(x)-p_\mu(x)|\leq {2(\log n)n^{-1/4}}$ implies $p_\mu(x) \leq 2 p_\varphi(x)$. 
\item Therefore, the inequality  $\sum_{x\not \in S}  \Paren{p_\mu(x)}^{\alpha}\leq 2^{\alpha}\sum_{x\not \in S} (p_\varphi(x))^{\alpha}\leq 2^{\alpha} P_\alpha(p_\varphi)$ holds.
\item Consequently, we establish $P_\alpha(p_\mu(x))\leq 2\sum_{x\not\in S}  \Paren{p_\mu(x)}^{\alpha}\leq 2^{1+\alpha}P_\alpha(p_\varphi)$.
\end{itemize} 
By the inequality ${P_\alpha(p)}/{2}\leq 2^{1+\alpha} P_\alpha(p_\varphi)$ and Corollary~\ref{cor:2}, if $|\hat{P}_\alpha^\star(x^n)-P_\alpha(p_\varphi)|\geq \varepsilon  P_\alpha(p_\varphi)$, 
\[
p_\varphi(\varphi)
\leq 2\exp\Paren{-\Omega_\alpha(\varepsilon^{2}n(P_\alpha(p_\varphi))^{1/\alpha})}
\leq 2\exp\Paren{-\Omega_\alpha(\varepsilon^{2}n(P_\alpha(p))^{1/\alpha})}.
\]
Let $\delta_p$ denote the quantity on the right-hand side. 
If we further have $p(\varphi)>\delta_p$, then by definition, $p_\varphi(\varphi)\geq p(\varphi)>\delta_p$. 
Hence for any $x^n$ with profile $\varphi$ satisfying both $p(\varphi)>\delta_p$ and $|\hat{P}_\alpha^\star(x^n)-P_\alpha(p)|\leq \varepsilon P_\alpha(p)$, we must have $|\hat{P}_\alpha^\star(x^n)-P_\alpha(p_\varphi)|\leq \varepsilon P_\alpha(p_\varphi)$. Simple algebra combines the last two inequalities and yields
\[
|P_\alpha(p)-P_\alpha(p_\varphi)|\leq 4\varepsilon P_\alpha(p).
\]
On the other hand, for a sample $X^n\sim p$ with profile $\varphi'$, the probability that we have both $p(\varphi')\leq \delta_p$ and $|\hat{P}_\alpha^\star(X^n)-P_\alpha(p)|\leq \varepsilon P_\alpha(p)$ is at most $\delta_p$ 
times the cardinality of the set $\Phi_{\alpha, \varepsilon}^n(p):=\{\varphi(x^n): x^n\in \cX^n \text{ and } |\hat{P}_\alpha^\star(x^n)-P_\alpha(p)|\leq \varepsilon P_\alpha(p)\}$. Below we complete this argument by finding a tight upper bound on $|\Phi_{\alpha, \varepsilon}^n(p)|$ in terms of its parameters. 

For any sequence $x^n$ such that $\varphi(x^n)\in \Phi_{\alpha, \varepsilon}^n(p)$, let $N_\varphi(x^n)$ denote the number of prevalences $\varphi_j(x^n)$ that are non-zero. Then by definition, we obtain
\begin{align*}
\sum_{j=0}^{N_\varphi(x^n)} \frac{j^{\underline{\alpha}}}{n^{\underline{\alpha}}}
\leq \sum_{j} \frac{j^{\underline{\alpha}}}{n^{\underline{\alpha}}}\cdot \varphi_j(x^n)
=\hat{P}_\alpha^\star(x^n)
\leq \frac{3}{2} P_\alpha(p).
\end{align*}
Using the standard falling-factorial identity $((j+1)^{\underline{1+\alpha}}-j^{\underline{1+\alpha}})/{(1+\alpha)}=j^{\underline{\alpha}}$, we can further simplify the expression on the left-hand side:
\[
\sum_{j=0}^{N_\varphi(x^n)} \frac{j^{\underline{\alpha}}}{n^{\underline{\alpha}}}
=\frac{(N_\varphi(x^n)+1)^{\underline{1+\alpha}}}{(1+\alpha) n^{\underline{\alpha}}}.
\]
This together with the inequality above yields $N_\varphi(x^n)\leq T^n_{\alpha}(p):=(3(1+\alpha)n^\alpha \cdot P_\alpha(p)/2)^{1/(1+\alpha)}$.
 Further note that each prevalence in $\varphi(x^n) = (\varphi_1(x^n),\ldots, \varphi_n(x^n))$ can only take values in $\lceil n\rfloor:=\{0,1,\ldots, n\}$. Therefore, $|\Phi_{\alpha, \varepsilon}^n(p)|$ is at most the number of $T^n_{\alpha}(p)$-sparse vectors over $\lceil n\rfloor^n$, which admits the following upper bound
\[
\binom{n}{T^n_{\alpha}(p)} \Abs{\lceil n\rfloor}^{T^n_{\alpha}(p)}\leq (n+1)^{2T^n_{\alpha}(p)}.
\]
Therefore, for $\delta_p\cdot |\Phi_{\alpha, \varepsilon}^n(p)|$ to be small, it suffices to have
\[
\Omega_\alpha(\varepsilon^{2}n(P_\alpha(p))^{1/\alpha})\gg 2T^n_{\alpha}(p)\log (n+1)=2(3(1+\alpha)n^\alpha \cdot P_\alpha(p)/2)^{1/(1+\alpha)} \log (n+1),
\]
which in turn simplifies to
\[
\varepsilon^{2}n^{1/(1+\alpha)}(P_\alpha(p))^{1/(\alpha(1+\alpha))}\gg \Theta_\alpha(\log n).
\]
Following this and $P_\alpha(p)\geq 4(n^{1/4}(4\log n)^{-1})^{1-\alpha}$, we obtain the following lower bound on $n$. 
\[
n\gg \Theta_\alpha((\varepsilon^{2}|\log{\varepsilon}|)^{-(1+\alpha)} (P_\alpha(p))^{-1/\alpha}).
\]
In this case, the probability bound $\delta_p\cdot |\Phi_{\alpha, \varepsilon}^n(p)|$ is no larger than $1/6$.

Finally, let $C$ denote the collection of sequences $x^n$ with profile $\varphi$ that do not satisfy $|\hat{P}_\alpha^\star(x^n)-P_\alpha(p)|\leq \varepsilon P_\alpha(p)$ or $\max_x|p_\varphi(x)-\mu_x(x^n)/n|\leq {2(\log n)n^{-1/4}}$. By Corollary~\ref{cor:2}, Lemma~\ref{lem:pmlcon}, and the union bound, 
\[
\Pr_{X^n\sim p}(X^n\in C)\leq 2\exp\Paren{-\Omega_\alpha(\varepsilon^{2}n(P_\alpha(p))^{1/\alpha})}+\mathcal{O}\Paren{\frac{1}{n}}.
\]
For $n$ satisfying the lower-bound inequality above, the right-hand side is again no larger than $1/6$.
This completes the proof of the theorem.

\section{Distribution estimation}\label{sec:distproof}
\subsection{Sorted $\boldsymbol{\ell_1}$ distance and Wasserstein duality}\label{sec:Wassproof}
For convenience, we first restate the theorem.
\begin{Theorem}
If $n=\Omega( n(\varepsilon))=\Omega\Paren{{\size}/{(\varepsilon^2\log \size)}}$ and $\varepsilon\geq n^{-c}$, 
\[
\Pr(\ell_1^{\text{\tiny{<}}}(p_\varphi,p)\geq \varepsilon)\leq \exp(-\Omega(n^{1/11})).
\] 
\end{Theorem}
In this section, we relate the estimation of sorted distributions to that of distribution properties through a dual definition of the $1$-Wasserstein distance. 

Recall that we let $\{p\}$ denote the multiset of probability values of 
a distribution $p\in\Delta_\cX$. 
The sorted $\ell_1$ distance between two distributions 
$p,q\in\Delta_\cX$ is
\[
\ell_1^{\text{\tiny{<}}}(p,q):=\min_{q'\in\Delta_\cX: \{q'\}=\{q\}}\norm{p-q'}_1,
\]
which is invariant under domain-symbol permutations on either $p$ or $q$.

For two distributions $\omega, \nu$ over the unit interval $[0,1]$, 
let $\Gamma'_{\omega, \nu}$ be the collection of distributions over
$[0,1]\times[0,1]$ with marginals $\omega$ and $\nu$ on the first and second factors respectively. 
The \emph{$1$-Wasserstein distance}, also known as the \emph{earth-mover distance}, between $\omega$ and
$\nu$ is
\[
\mathcal{W}_1(\omega, \nu)
:=
\inf_{\gamma\in\Gamma'_{\omega, \nu}}\;\Exp_{(X,Y)\sim\gamma}
\Abs{X-Y}.
\]
Equivalently, let $\mathcal{L}_1$ denote the collection of real functions that are $1$-Lipschitz on $[0,1]$. 
Through duality, one can also define the $1$-Wasserstein distance~\cite{duality} as
\[
\mathcal{W}_1(\omega, \nu)= \sup_{f\in\mathcal{L}_1} \Paren{\Exp_{X\sim \omega} f(X)-\Exp_{Y\sim \nu} f(Y)}.
\]
For any $p\in\Delta_\cX$, let $u_{\{p\}}$ denote the distribution induced by the uniform measure on $\{p\}$. 
For any distributions $p, q\in\Delta_\cX$, one can verify~\cite{instdist, remd, jnew} that
\[
\ell_1^{\text{\tiny{<}}}(p,q) =\size\cdot \mathcal{W}_1(u_{\{p\}}, u_{\{q\}})\leq R(p,q). 
\]
Combining this with the dual definition of $\mathcal{W}_1$, we obtain
\[
\ell_1^{\text{\tiny{<}}}(p,q) = \size\cdot \sup_{f\in\mathcal{L}_1} \Paren{\Exp_{X\sim u_{\{p\}}} f(X)-\Exp_{Y\sim u_{\{q\}}} f(Y)}=\sup_{f\in\mathcal{L}_1} \Paren{\sum_x f(p(x))-\sum_x f(q(x))}.
\]

\subsection{Proof of Theorem~\ref{thm:dist}}\label{sec:hardtrun}
For a real function $f\in \mathcal{L}_1$, we denote by $f(p):=\sum_x f(p(x))$  the corresponding additive symmetric property. 
The previous reasoning also shows that for any $p, q\in\Delta_\cX$, 
\[
R(p,q)\geq\ell_1^{\text{\tiny{<}}}(p,q)\geq |f(p)-f(q)|.
\]
Therefore, property $f$ is {$1$-Lipschitz} on $(\Delta_\cX, R)$. 

Set $n:=\sup_{f\in\mathcal{L}_1} n_f(\varepsilon)$. The results in~\cite{jnew} imply that if $\varepsilon> n^{-0.3}$, 
\[
n=\Theta\Paren{\frac{\size}{\varepsilon^2\log \size}}.
\]
Clearly, we only need to consider $\varepsilon\le 2$, implying $k=\mathcal{O}(n\log n)$. Let~$\alpha, \gamma$ be absolute constants in $[1/100,1/6)$ and $\varepsilon>0$ be an error parameter. 

By the proof of Theorem~\ref{thm:est} in Section~\ref{sec:trueproof}, for any distribution $p\in\Delta_\cX$ and $X^{n/\alpha}\sim p$, with probability at least $1-2\exp\Paren{-4n^{1-2\gamma}}$, the PML (or APML) plug-in estimator will satisfy
\[
|f(p)-f(p_{\varphi(X^{n/\alpha})})|< \varepsilon \Paren{2+o(1)}+\mathcal{O}(n^{-c_1/2}\log^2 n)+4 n^{1-\gamma} \tau(n),
\]
where $c_1\in(1/100,1/32]$, $c_2=1/2+6c_1$, and $\tau(n)=\mathcal{O}\Paren{n^{\alpha c_2+ (2-\alpha) c_1-1}\log^3 n}$.
Additionally, in the previous section, we have proved that
\[
\ell_1^{\text{\tiny{<}}}(p,q) = \sup_{f\in\mathcal{L}_1} \Paren{f(p)- f(q)}=\sup_{f\in\mathcal{L}_1} |f(p)- f(q)|.
\]
Though it seems that the above inequality and equation imply the optimality of PML (since $f$ is chosen arbitrarily), such direct  implication actually does not hold. The reason is a little bit subtle: The inequality on $|f(p)- f(p_{\varphi(X^{n/\alpha})})|$ holds for any fixed function $f$ and $p\in\Delta_\cX$, while the function that achieves the corresponding supremum in
\[
 \sup_{f\in\mathcal{L}_1} \Abs{f(p)- f(p_{\varphi(X^{n/\alpha})})}=\ell_1^{\text{\tiny{<}}}\Paren{p,p_{\varphi(X^{n/\alpha})}}
\]
depends on both $p$ and $X^{n/\alpha}$, and hence is a random function. To address this discrepancy, we provide a more involved argument below. 

Let $f$ be a function in $\mathcal{L}_1$. Without loss of generality, we also assume that $f(0)=0$.  Let $\eta\in(0,1)$ be a threshold parameter to be determined later. An \emph{$\eta$-truncation} of $f$ is a function
\[
f_\eta(z):= f(z)\indic_{z\leq \eta}+f(\eta) \indic_{z> \eta}. 
\]
One can easily verify that $f_\eta\in\mathcal{L}_1$. 
Next, we find a finite subset of $\mathcal{L}_1$ so that the $\eta$-truncation of any $f\in\mathcal{L}_1$ is close to at least one of the functions in this subset. 

For a parameter $s>3$ to be chosen later. Partition the interval $[0,\eta]$ into $s$ disjoint sub-intervals of equal length, and define the sequence of end points as $z_j:=\eta\cdot j/s, j\in \lceil s\rfloor$ where $\lceil s\rfloor:=\{0,1,\ldots, s\}$. Then, for each $j\in\lceil s\rfloor$, 
we find the integer $j'$ such that $|f_\eta(z_j)-z_{j'}|$ is minimized and denote it by $j^*$. \vspace{-0.1em} Since $f_\eta$ is 1-Lipschitz, we must have $|j^*|\in\lceil j\rfloor$. Finally, we connect the points $Z_j:=(z_j, z_{j^*})$ sequentially. This curve is continuous and corresponds to a particular $\eta$-truncation $\tilde{f}_\eta\in \mathcal{L}_1$, which we refer to as the \emph{discretized $\eta$-truncation} of $f$. Intuitively, we have constructed an $(s+1)\times(s+1)$ grid and ``discretized'' function $f$ by finding its closest approximation in $\mathcal{L}_1$ whose curve only consists of edges and diagonals of the grid cells. By construction, 
\[
\max_{z\in[0,1]}|f_\eta(z)-\tilde{f}_\eta(z)|\leq{\eta}/{s}.
\]\par\vspace{-1em}
Therefore, for any $p\in\Delta_\cX$, the corresponding properties of $f_\eta$ and $\tilde{f}_\eta$ satisfy
\[
|f_\eta(p)-\tilde{f}_\eta(p)|\leq k\cdot{\eta}/{s}.
\]
Note that $|j^*|\in\lceil j\rfloor$ for all $j\in\lceil s\rfloor$, and $\tilde{f}_\eta(z)=z_{s^*}$ for $z\geq \eta$. While there are infinitely many $\eta$-truncations, the cardinality of the discretized $\eta$-truncations of functions in $\mathcal{L}_1$ is at most 
\[
\prod_{j=0}^{s}(2j+1)=(s+1) \prod_{j=0}^{s-1}(2j+1)(2s-2j+1)\leq {(s+1)}^{2s+1}=e^{(2s+1)\log (s+1)}\leq e^{3s\log s}.
\]
Consider any $p\in\Delta_\cX$ and $X^{n/\alpha}\sim p$ with a profile $\varphi$. Consolidate the previous results, and apply the union bound and triangle inequality. With probability at least $1-2\exp\Paren{3s\log s-4n^{1-2\gamma}}$, the PML plug-in estimator will satisfy
\begin{align*}
|f_\eta(p)-f_\eta(p_{\varphi})|
&\leq  |f_\eta(p)-\tilde{f}_\eta(p)|+|\tilde{f}_\eta(p)-\tilde{f}_\eta(p_{\varphi})|+ |\tilde{f}_\eta(p_{\varphi})-f_\eta(p_{\varphi})|\\
&\leq 2k\cdot{\eta}/{s}+\varepsilon \Paren{2+o(1)}+\mathcal{O}(n^{-c_1/2}\log^2 n)+4 n^{1-\gamma} \tau(n),
\end{align*}
for \emph{all} functions $f$ in $\mathcal{L}_1$.

Next we consider the ``second part'' of a function $f\in\mathcal{L}_1$, namely,
\[
\bar{f}_{\eta}(z):=f(z)-f_{\eta}(z)=(f(z)-f(\eta)) \indic_{z> \eta}.
\]
Again, we can verify that $\bar{f}_{\gamma}\in \mathcal{L}_1$. To establish the corresponding guarantees, we make use of the following result. Since the profile probability $p(\varphi)$ is invariant to symbol permutation, for our purpose, we can assume that $p(y)\leq p(z)$ iff $p_{\varphi}(x)\leq p_{\varphi}(y)$, for all $x, y\in \cX$. Under this assumption, the following lemma, which follows from the consistency results in~\citep{O11,A17}, relates~$p_{\varphi}$~to~$p$. Let~$\gamma'\in(0,1/4)$ be an absolute constant to be determined later. Then, 
\begin{Lemma}\label{lem:pmlcon}
For any distribution $p$ and sample $X^m\sim p$ with profile $\varphi$,
\[
\Pr\Paren{\max_x|p_{\varphi}(x)-p(x)|>m^{\gamma'-1/4}}= \mathcal{O}\Paren{m^{1/4}\exp(-\Omega(m^{1/2+2\gamma'}))}.
\]
\end{Lemma}\vspace{-0.75em}
Simply follow the proofs in~\citep{O11,A17}, we obtain: Changing $1/4$ to any (fixed) number greater than $1/6$, the above lemma also holds for APML with $m^{1/2+2\gamma'}$ replaced by $m^{2/3+2\gamma'}$.

Set $m=n/\alpha$ in this lemma. With probability at least $1-\mathcal{O}\Paren{(n/\alpha)^{1/4}\exp(-\Omega((n/\alpha)^{1/2+2\gamma'}))}$,
\begin{align*}
|\bar{f}_\eta(p)-\bar{f}_\eta(p_{\varphi})|
&=|\sum_{x}\bar{f}_\eta(p(x))-\bar{f}_\eta(p_{\varphi}(x))|\\
&\leq \sum_{x: p(x)>\eta \text{ or } p_{\varphi}(x)>\eta}|\bar{f}_\eta(p(x))-\bar{f}_\eta(p_{\varphi}(x))|\\
&\leq \sum_{x: p(x)>\eta \text{ or } p_{\varphi}(x)>\eta}|p(x)-p_{\varphi}(x)|\\
&\leq ({2}/{\eta}) (n/\alpha)^{\gamma'-1/4},
\end{align*}
for \emph{all} functions $f$ in $\mathcal{L}_1$.

Consolidate the previous results. By the triangle inequality and the union bound, with probability at least $1-2\exp\Paren{3s\log s-4n^{1-2\gamma}}-\mathcal{O}\Paren{(n/\alpha)^{1/4}\exp(-\Omega((n/\alpha)^{1/2+2\gamma'}))}$,
\begin{align*}
|f(p)-f(p_{\varphi})|
&\leq |f_\eta(p)-f_\eta(p_{\varphi})|+|\bar{f}_\eta(p)-\bar{f}_\eta(p_{\varphi})|\\
&\leq 2k{\eta}/{s}+\varepsilon \Paren{2+o(1)}+\mathcal{O}(n^{-c_1/2}\log^2 n)+4 n^{1-\gamma} \tau(n)+({2}/{\eta}) (n/\alpha)^{\gamma'-1/4},
\end{align*}
for \emph{all} functions $f$ in $\mathcal{L}_1$. Now we can conclude that $\ell_1^{\text{\tiny{<}}}\Paren{p,p_{\varphi}}$ is also at most the error bound on the right-hand side. The reason is straightforward: Since with high probability, the above guarantee holds for all functions in $\mathcal{L}_1$, it must also hold for the function that achieves the supremum in
\[
 \sup_{f\in\mathcal{L}_1} \Abs{f(p)- f(p_{\varphi})}=\ell_1^{\text{\tiny{<}}}\Paren{p,p_{\varphi}}.
\]
It remains to make sure that all the quantities in the error bound except $\varepsilon \Paren{2+o(1)}$ vanish with $n$, and the probability bound converges to $1$ as $n$ increases. Recall that $k=\mathcal{O}(n\log n)$, $c_1\in(1/100,1/25]$, $c_2=1/2+6c_1$, and $\tau(n)=\mathcal{O}\Paren{n^{\alpha c_2+ (2-\alpha) c_1-1}\log^3 n}$.

By direct computation, we can choose $\alpha=1/100$, $c_1=1/26$, $\gamma'=1/200$, $\gamma=(5/2+5\alpha)c_1+\alpha/2$, 
$s=n^{\gamma'+3/4+c_1}$, and $\eta=n^{\gamma'-1/4+c_1/2}$. Note that this is just one possible set of parameters. 
Given this choice, we have
\[
\ell_1^{\text{\tiny{<}}}\Paren{p,p_{\varphi}}\leq \varepsilon \Paren{2+o(1)} +\mathcal{O}(n^{-c_1/2}\log^3 n),
\]
with probability at least $1-\exp(-\Omega(n^{1/2}))$. Additionally, the equation 
\[
\sup_{f\in\mathcal{L}_1} \Abs{f(p)- f(p_{\varphi})}=\ell_1^{\text{\tiny{<}}}\Paren{p,p_{\varphi}}
\]
clearly yields that $n(\varepsilon)\geq \sup_{f\in\mathcal{L}_1} n_f(\varepsilon)$. Hence for $\varepsilon\geq \mathcal{O}(n^{-c_1/2}\log^4 n)$, 
\[
n(p_\varphi, (2+o(1))\varepsilon)\leq 100n(\varepsilon).
\]
\section{Truncated PML}\label{sec:TPML}

The idea appearing in the last section also applies to other tasks. One of the extensions is to compute a \emph{truncated/partial PML} and use the corresponding plug-in estimator to approximate certain properties. 

Recall that the \emph{profile} of a sequence $x^n$ is $\varphi(x^n)=(\varphi_1(x^n), \ldots, \varphi_n(x^n))$,
the vector of all the positive prevalences. 
We naturally define the \emph{$t$-truncated profile} of $x^n$ as 
\[
\varphi^t(x^n):=(\varphi_1(x^n), \ldots, \varphi_t(x^n)),
\]
the profile vector truncated at location $t$. 
Analogous to the definition of profile probability, for a distribution $p$, we define the \emph{probability of a truncated profile $\varphi^t$} as 
\[
p(\varphi^t):=\sum_{y^n: \varphi^t(y^n)=\varphi^t}p(y^n),
\]
the probability of observing a size-$n$ sample from $p$ with truncated profile $\varphi^t$.
For a set $\cP\subseteq \Delta_\cX$, the \emph{truncated profile maximum likelihood (TPML)} estimator over $\cP$ maps each $\varphi^t$ to a distribution 
\[
p_{\varphi^t}:=\arg\max_{p\in \cP} p(\varphi^t)
\]
that maximizes the truncated profile probability. In the subsequent discussion, we will assume that $\cP=\Delta_\cX$ unless otherwise specified. The following lemma states that the TPML plug-in estimator is competitive to other truncated-profile-based estimators. 
\begin{Lemma}\label{lem:tpmlcomp}
Let $f$ be a symmetric distribution property. If for samples of size $n$, there exists an estimator $\hat{f}$ 
over $t$-truncated profiles such that for any $p\in \cP$ and $\varphi^t\sim p$,
\[
\Pr(|f(p)-\hat{f}(\varphi^t)|>\varepsilon)<\delta,
\]
then
\[
\Pr(|f(p)-f(p_{\varphi^t})|>2\varepsilon)<\delta\cdot e n^t.
\]
\end{Lemma}
The proof essentially follows from Theorem 3 in~\citep{mmcover}. Note that the term $e n^t$ in the upper bound is sub-optimal for large $t$ values. For $t=\Omega(\sqrt{n}/\log n)$, one should replace $e n^t$ by $\exp(\sqrt{3n})$.

\subsection{TPML and Shannon-entropy estimation}
Below we consider Shannon entropy estimation using the TPML estimator.

Letting $h(z):=-z\log z$, the Shannon entropy of a distribution $p$ is
\[
H(p):=\sum_{x} h(p(x)).\vspace{-0.4em}
\]
Following the derivations in Section~\ref{sec:hardtrun}, we partition $H(p)$ into two parts: One part corresponds to the partial entropy of small probabilities, and the other corresponds to that of large ones. 

For simplicity of consecutive arguments, we assume that $n$ is an even integer. Let $c_1,c_2, c_3$, and $c_4$ be positive absolute constants to be determined later. 

Since $p$ is unknown, we perform a ``soft truncation'' (instead of the ``hard truncation'' performed in Section~\ref{sec:hardtrun}) and partition $H(p)$ into
\[
H^{s}(p):=\sum_{x} h(p(x)) \cdot \Pr_{Z^{n/2}\sim p}\Paren{\mu_x(Z^{\frac{n}{2}})\leq c_1\log \frac{n}{2}}
\]
and 
\[
H^{\ell}(p):=H(p)-H^{s}(p).
\]
To estimate $H^{s}(p)$, we make use of an estimator similar to that in~\citep{mmentro}. Let $d:=c_2\log n$ be a degree parameter. Let $g(z):=\sum_{i=0}^{d}a_i\cdot z^{i}$ denote the degree-$d$ min-max polynomial approximation of $h(z)$ over $I_n:=[0, c_3(\log n)/n]$.
For a sample $X^n$ from $p$, denote by $X^{n/2}$ and $X_{n/2}^{n}$ its first and second halves. Denote by $A^{\underline{B}}$ the order-$B$ falling factorial of $A$.  
Consider the following estimator.
\[
\hat{H}^{s}(X^n):=\sum_{x} \Paren{\sum_{i=0}^{d}a_i\cdot \frac{\mu_x(X^{n/2})^{\underline{i}}}{(n/2)^{\underline{i}}}}  
\indic_{\mu_x(X^{\frac n2})\leq c_4\log \frac n2}\cdot \indic_{\mu_x(X_{n/2}^{n})\leq c_1\log \frac{n}{2}}.
\]

Choose $c_4\gg c_3\gg c_1$ and $1\gg c_2$. Following the derivations in~\cite{mmentro}, and utilizing the Chernoff bound
and $\max_{z\in I_n}|g(z)-h(z)|=\mathcal{O}\Paren{{1}/{(n\log n)}}$, we bound the bias of $\hat{H}^{s}$ by $\mathcal{O}(k/(n\log n))$.
Furthermore, since $|a_i|=\mathcal{O}(2^{3d}(n/\log n)^{i-1})$, for any absolute constant $\lambda\in(0,1/2)$, we can choose a sufficiently small $c_2$ so that the $n$-sensitivity of $\hat{H}^{s}$ is at most $\mathcal{O}(n^\lambda/n)$.  
.

Estimator $\hat{H}^{s}$ is not a profile-based estimator as the sample partitioning creates asymmetry. 
Therefore Lemma~\ref{lem:tpmlcomp} does not directly apply here.  To close this gap, we present two different approaches:
 one is to modify the definition of TPML and redefine it as the probability-maximizing distribution for sequence partitions, 
 the other is to modify the estimator so that it is profile-based without changing the estimator's bias and sensitivity too much. 
 Below we present the first approach.

For any sequence pair $(x^m,y^m)$, define the prevalence $\mu_{i,j}(x^m, y^m)$ of an integer pair $(i,j)$ as the number of symbols $x$ satisfying both $\mu_x(x^m)=i$ and $\mu_x(y^m)=j$. 
We re-define the $t$-truncated profile of $(x^m, y^m)$ as the $t\times t$ matrix 
\[
\varphi^t(x^m, y^m):=[\mu_{i,j}(x^m, y^m)]_{i,j\in[t]}. 
\]
In the same way we define the TPML estimator and derive a result similar to Lemma~\ref{lem:tpmlcomp}. 
\begin{Lemma}\label{lem:tpmlcomp1}
Let $f$ be a symmetric distribution property. If for samples of size $2m$, there exists an estimator $\hat{f}$ 
over $t$-truncated profiles such that for any $p\in \Delta_\cX$ and $\varphi^t\sim p$,
\[
\Pr(|f(p)-\hat{f}(\varphi^t)|>\varepsilon)<\delta,
\]
then
\[
\Pr(|f(p)-f(p_{\varphi^t})|>2\varepsilon)<\delta\cdot (m+1)^{2t^2}.
\]
\end{Lemma}
This creates a new version of TPML but does not change the nature of the approach. 
Later in this section, we provide an \emph{alternative argument} employing the original TPML.

Due to the two indicator functions in the definition of $\hat{H}^{s}$, 
we can view $\hat{H}^{s}(X^n)$ as an estimator over $(c_4\log (n/2))$-truncated profiles. 
Then for any $\tau\geq 0$, together with the $n$-sensitivity 
bound $\mathcal{O}(n^\lambda/n)$ for $\hat{H}^{s}$, Lemma~\ref{lem:sensitivitybound} yields that
\[
 \Pr\Paren{\Abs{\hat{H}^{s}(X^n)-\EE[\hat{H}^{s}(X^n)]}\geq \tau}\leq 2\exp(-2\tau^2\cdot \Omega(n^{1-2\lambda})).
\]
The triangle inequality combines this with the previous bias bound, 
\[
 \Pr\Paren{|\hat{H}^{s}(X^n)-H^s(p)|\geq \tau+\mathcal{O}\Paren{\frac{k}{n\log n}}}\leq 2\exp(-2\tau^2\cdot \Omega(n^{1-2\lambda})).
\]
Applying Lemma~\ref{lem:tpmlcomp1} to $\hat{H}^{s}$ with $t=c_4\log (n/2)$ further implies that 
\[
\Pr\Paren{|H^{s}(p)-H^{s}(p_{\varphi^t})|\ge 2\tau+\mathcal{O}\Paren{\frac{k}{n\log n}}}\le 2\exp(-2\tau^2\cdot \Omega(n^{1-2\lambda}))\cdot (n/2+1)^{2t^2}.
\]
The right-hand side vanishes as fast as $2\exp(-\log^3 n)$ for $\tau=\Omega((\log n)^{1.5} /n^{1/2-\lambda})$. 

It remains to estimate the partial entropy of the large probabilities:
\[
H^{\ell}(p)=H(p)-H^{s}(p)=\sum_{x} h(p(x)) \cdot \Exp_{X^{n/2}\sim p}\left[\indic_{\mu_x(Y^\frac{n}{2})> c_1\log \frac{n}{2}}\right].
\]
We can estimate $H^{\ell}(p)$ by a simple variation of the Miller-Mallow estimator~\citep{emiller}:
\[
\hat{H}^{\ell}(X^n):=\sum_{x} \Paren{h\Paren{\frac{2\mu_x(X^{\frac n2})}{n}}+\frac{1}{n}} \cdot \Paren{1-\indic_{\mu_x(X^{\frac n2})\leq c_4\log \frac n2}\cdot \indic_{\mu_x(X_{n/2}^{n})\leq c_1\log \frac n2}}.
\]
For $c_4\gg c_1\gg 1$, derivations in~\citep{mmentro} bound the estimator's bias as
\[
\Abs{\EE[\hat{H}^{\ell}(X^n)]-H^{\ell}(p)|}\le \mathcal{O}\Paren{\frac{k}{n\log n}}.
\]
The $n$-sensitivity of $\hat{H}^{\ell}$ is $\mathcal{O}((\log^2 n)/n)$. 
The same rationale as the previous argument yields
\[
 \Pr\Paren{|\hat{H}^{\ell}(X^n)-H^\ell(p)|\geq \mathcal{O}\Paren{\frac{\log^3 n}{\sqrt{n}}+\frac{k}{n\log n}}}\leq 2\exp(-\log^2 n).
\]
Shown in~\cite{mmentro}, for $k=\tilde{\Omega}(n^{1/2})$, the sample complexity of estimating $H$ is $n_H(\varepsilon)=\Theta(k/(\varepsilon\log k))$. 
Under this condition, the following theorem summarizes our results.
\setcounter{Theorem}{6} 
\begin{Theorem}\label{thm:entro}
Entropy estimator $H^{s}(p_{\varphi_t})+\hat{H}^{\ell}$
is sample-optimal for $\varepsilon= \tilde{\Omega}(1/n^{1/2-\lambda})$.
\end{Theorem}
Note that we hide the estimator's dependence on $n$. Since $\lambda$ is an arbitrary absolute constant in $(0,1/2)$, the range of $\varepsilon$ where the estimator is sample-optimal is near-optimal (e.g., set $\lambda=0.01$) and better than the $\varepsilon= \Omega(n^{-0.2})$ range established in~\cite{mmcover} for the PML plug-in estimator. 

We can view the estimator in Theorem~\ref{thm:entro} as a joint plug-in estimator of two distribution estimates: $p_{\varphi_t}$ and $p_\mu$. Effectively, we decompose the original property into smooth and non-smooth parts. As~is the case with PML and APML, for $\beta\in(0,1)$, we can define the \emph{$\beta$-approximate TPML}
estimator as a mapping from each truncated profile $\varphi^t$ to a distribution
$p_{\varphi^t}^\beta$ satisfying
${p}^\beta_{\varphi^t}({\varphi^t})\ge \beta\cdot p_{\varphi^t}({\varphi^t})$. Via the same reasoning, one can verify that Theorem~\ref{thm:entro} also holds for any $\exp(-\polylog n)$-approximate TPML, which we refer to as ATPML.

\paragraph{Alternative argument} The above derivation utilizes a modified TPML. We sketch an alternative argument~\citep{percomm} using the original version by modifying the estimator $\hat{H}^s$ instead of TPML.

For a sample $X^n\sim p$, consider all its permuted versions. Applying $\hat{H}^{s}$ to each permutation of $X^n$ yields an estimate. We define $\hat{H}^{S}$ as an estimator that maps $X^n$ to the \emph{average} of all such estimates. Averaging explicitly removes the estimator's dependency on the ordering of sample points and makes it profile-based. In fact, this new estimator is over $(c_1+c_2)\log (n/2)$-truncated profiles due to the two indicator functions in the definition of $\hat{H}^s$. 

By symmetry and the linearity of expectation, the bias of $\hat{H}^S$ in estimating $H^s(p)$ is exactly equal to that of $\hat{H}^s$. In addition, any bounds on the  sensitivity of $\hat{H}^s$ also applies to $\hat{H}^S$. In particular, for any absolute constant $\lambda\in(0,1/2)$, we can choose a sufficiently small $c_2$ so that the $n$-sensitivity of $\hat{H}^{S}$ is at most $\mathcal{O}(n^\lambda/n)$.  
Utilizing Lemma~\ref{lem:sensitivitybound}, Lemma~\ref{lem:tpmlcomp}, and the same rationale as the previous argument, we establish Theorem~\ref{thm:entro} for the original version of TPML. 

\vfill

\subsection{TPML and support- coverage and size estimation}

We can apply TPML and ATPML to approximate other symmetric properties having smoothness attributes similar to those of Shannon entropy. \vspace{-0.5em}

\paragraph{Normalized support coverage} For example, consider estimating the normalized support coverage $\tilde{C}_m(p)=\sum_{x}(1-(1-p(x))^m)/m$ of an unknown distribution $p\in \Delta_{\cX}$. Similar to the previous argument, for a positive absolute constant $c_1$ to be determined, we can partition $\tilde{C}_m(p)$ into
\[
\tilde{C}_m^{s}(p):=\frac{1}{m}\cdot \sum_{x} (1-(1-p(x))^m)\cdot\Pr_{Z^n\sim p}(\mu_x(Z^n)\leq c_1\log n)
\]
and 
\[
\tilde{C}_m^{\ell}(p):=\tilde{C}_m(p)-\tilde{C}_m^{s}(p).
\]
Let $X^n$ and $Y^n$ be two independent samples from $p$, and denote $c_m(p(x)):=(1-(1-p(x))^m)/m$.

By the results in~\citep{mmcover}, for any positive absolute constant $\alpha$, error parameter $\varepsilon\ge {6n^\alpha}/{n}$, and parameters $m, n$ such that  
$2n\le m\le \alpha\frac{ n\log (n/2^{1/\alpha})}{\log (3/\varepsilon)}$, there is a linear estimator $\hat{C}_m:=\sum_{i\geq 1} \ell_i\cdot \varphi_i$ satisfying 
\[
\Abs{\Exp_{X^n\sim p}[\hat{C}_m(X^n)]-\tilde{C}_m(p)}\le \sum_{x}\Abs{\EE[\ell_{\mu_x(X^n)}]-c_m(p(x))}\le \frac{3n^\alpha}{m}+ \frac{\varepsilon}{3}\cdot\frac{\min\{m,k\}}{m} 
\]
and $\max_{i\ge 1} |\ell_i| \le {n^\alpha}/{n}$. Utilizing $Y^n$ and letting $c_2:=10 c_1$, we estimate $\tilde{C}_m^{s}(p)$ by 
\[
\hat{C}_m^{s}(X^n, Y^n) := \sum_x \sum_{i=1}^{c_2 \log n}\ell_{\mu_x(X^n)=i}\cdot \indic_{\mu_x(Y^n)\le c_1\log n}.
\]
We bound the bias of this estimator as follows. 
\begin{align*}
\Abs{\EE[\hat{C}_m^{s}(X^n, Y^n)]-\tilde{C}_m^{s}(p)}
& \le \sum_{x}\Abs{\EE[\ell_{\mu_x(X^n)}]-c_m(p(x))}\cdot \EE[\indic_{\mu_x(Y^n)\le c_1\log n}]\\
& + \sum_{x}\Abs{\EE\left[\sum_{i>c_2 \log n}\ell_{\mu_x(X^n)=i}\right]\cdot \EE[\indic_{\mu_x(Y^n)\le c_1\log n}]}\\
&\le \frac{3n^\alpha}{m}+ \frac{\varepsilon}{3}\cdot\frac{\min\{m,k\}}{m} +\max_{i\ge 1} |\ell_i|\cdot \sum_{x}\EE[\indic_{\mu_x(X^n)> c_2\log n}]\cdot\EE[\indic_{\mu_x(Y^n)\le c_1\log n}]\\
&\le \frac{3n^\alpha}{n}+ \frac{\varepsilon}{3}  + \frac{n^\alpha}{n}\cdot \sum_{x}np(x)\cdot \EE[\indic_{\mu_x(X^{n-1})\ge c_2\log n}]\cdot\EE[\indic_{\mu_x(Y^n)\le c_1\log n}]\\
&\le \varepsilon,
\end{align*}
where in the last step, we assumed that $c_1$ is sufficiently large and applied the Chernoff bound for binomial random variables. 
Also note that changing one point in $X^n$ or $Y^n$ changes the value of $\hat{C}_m^{s}(X^n, Y^n)$ by at most $4n^\alpha/n$. Viewing $Z^{2n}:=(X^n, Y^n)$ as a single sample, we can apply $\hat{C}_m^{s}$ to all the equal-size partitions of $Z^{2n}$ and denote by $\hat{C}_m^{S}(Z^{2n})$ the average of all the corresponding estimates. The resulting estimator $\hat{C}_m^{S}$ is over $(c_1+c_2)\log n$-truncated profiles,  and has the same bias and sensitivity bound as $\hat{C}_m^{s}$.  Finally, we substitute $2n$ with $n$. 

For any $\tau\geq 0$, Lemma~\ref{lem:sensitivitybound} and the $n$-sensitivity bound $\mathcal{O}(n^\alpha/n)$ for $\hat{C}_m^{S}$ yield that
\[
 \Pr_{Z^n\sim p}\Paren{|\hat{C}_m^{S}(Z^n)-\EE[\hat{C}_m^{S}(Z^n)]|\geq \tau}\leq 2\exp(-2\tau^2\cdot \Omega(n^{1-2\alpha})).
\] 
Applying Lemma~\ref{lem:tpmlcomp} and letting $t=(c_1+c_2)\log (n/2)$, we establish a similar guarantee for the TPML plug-in estimator.
\[
\Pr\Paren{|\tilde{C}_m^{s}(p)-\tilde{C}_m^{s}(p_{\varphi^t})|\ge 2\tau+2\varepsilon}\le 2\exp(-2\tau^2\cdot \Omega(n^{1-2\alpha}))\cdot en^t.
\]
The right-hand side vanishes as fast as $2\exp(-\log^2 n)$ for $\tau=\Omega((\log n) /n^{1/2-\alpha})$. 

Next we construct an estimator for
\[
\tilde{C}_m^{\ell}(p)=\tilde{C}_m(p)-\tilde{C}_m^{s}(p)=\sum_{x} c_m(p(x))\cdot\Exp_{Z^{n/2}\sim p}\left[\indic_{\mu_x(Z^{\frac n2})> c_1\log\frac n2}\right]
.
\]
We simply split the sample $Z^n$ into two parts of equal size, and refer to the first and second parts as $Z^{\frac n2}$ and $Z_{n/2}^n$, respectively. Then, we estimate $\tilde{C}_m^{\ell}(p)$ by 
\[
\hat{C}_m^{\ell}(Z^n) := 
\frac{1}{m}\sum_{x}\indic_{\mu_x(Z^{\frac n2})> 0} \cdot\indic_{\mu_x(Z_{n/2}^n)> c_1\log\frac n2}.
\]
The bias of this estimator satisfies
\begin{align*}
\Abs{\Exp[\hat{C}_m^{\ell}(Z^{n})]-\tilde{C}_m^{\ell}(p)}
&=\Abs{\frac{1}{m}\sum_x((1-p(x))^{\frac n2}-(1-p(x))^{m})\cdot\Exp_{Z^{n/2}\sim p}\left[\indic_{\mu_x(Z^{\frac n2})> c_1\log\frac n2}\right]}\\
&\le \sum_x p(x)(1-p(x))^{\frac n2}\cdot\Exp_{Z^{n/2}\sim p}\left[\indic_{\mu_x(Z^{\frac n2-1})\ge c_1\log\frac n2}\right]\\
&\le  \sum_{x: np(x)< c_1\log \frac{n}{2}} p(x)\cdot\Exp_{Z^{n/2}\sim p}\left[\indic_{\mu_x(Z^{\frac n2-1})\ge c_1\log\frac n2}\right]\\
&+ \sum_{x: np(x)\ge c_1\log \frac{n}{2}} p(x)\Paren{1-\frac{c_1\log \frac{n}{2}}{n}}^{\frac n2}\\
&\le 2\exp(-\Omega(c_1\log n)),
\end{align*}
where the last step follows from the Chernoff bound. The bias is $\mathcal{O}(1/n)$ for sufficiently large $c_1$. Furthermore, the $n$-sensitivity of $\hat{C}_m^{\ell}$ is exactly $1/m<1/n$.

By the McDiarmid's inequality, with probability at least $1-2\exp(-\log^2 n)$, 
\[
|\hat{C}_m^{\ell}(Z^n)-\tilde{C}_m^{\ell}(p)|\le \frac{\log n}{\sqrt{n}}.
\]
Consolidating the previous results yields
\begin{Theorem}\label{thm:suppc}
Support-coverage estimator $\tilde{C}_m^{s}(p_{\varphi^t})+\hat{C}_m^{\ell}$
is sample-optimal for $\varepsilon= \tilde{\Omega}(1/n^{1/2-\alpha})$.
\end{Theorem}

Note that $t=\Theta(\log n)$ and we replaced $n$ with $n/2$ in the definition of $\tilde{C}_m^{s}$. As in the case of entropy estimation, the range of $\varepsilon$ where the estimator is sample-optimal is again near-optimal (e.g., set $\alpha=0.01$) and better than the $\varepsilon= \Omega(n^{-0.2})$ range established in~\cite{mmcover} for the PML plug-in estimator. 

\paragraph{Normalized support size} Following the previous discussion, we consider estimating the normalized support size $\tilde{S}(p)=\sum_x \indic_{p(x)>0}/k$ of an unknown distribution $p\in \Delta_{\ge 1/k}$. Again, for a positive absolute constant $c_1$ to be determined, we can partition $\tilde{S}(p)$ into \vspace{-0.2em}
\[
\tilde{S}^{s}(p):=\frac{1}{k}\cdot \sum_{x} \indic_{p(x)>0} \cdot\Pr_{Z^{n/2}\sim p}\Paren{\mu_x(Z^{\frac n2})\leq c_1\log \frac n2}
\]
and 
\[
\tilde{S}^{\ell}(p):=\tilde{S}(p)-\tilde{S}^{s}(p).
\]
We proceed by relating $\tilde{S}^{s}(p)$ to $\tilde{C}_{m}^{s}(p)$. Note that we replaced $2n$ with $n$ in $\tilde{C}_{m}^{s}(p)$. For any error parameter $\varepsilon$, choose $m=k\log (1/\varepsilon)$, then
\begin{align*}
|\tilde{C}_{m}^{s}(p)\cdot \log(1/\varepsilon)-\tilde{S}^{s}(p)|
& = \Abs{\frac{1}{k} \sum_{x} (\indic_{p(x)>0}-(1-(1-p(x))^m)) \cdot\Pr_{Z^{n/2}\sim p}\Paren{\mu_x(Z^{\frac n2})\leq c_1\log \frac n2}}\\
& \le \frac{1}{k} \sum_{x} (1-p(x))^m \le \frac{1}{k}\sum_{x} \Paren{1-\frac 1k}^{k\log \frac{1}{\varepsilon}}\le \varepsilon.
\end{align*}
Hence by the previous results, for $Z^n\sim p$, $\varepsilon\ge 12n^\alpha/n$, and $k$, $n$ such that $n\le k\log (1/\varepsilon)\le \alpha\frac{ n\log (n/2^{1+1/\alpha})}{2\log (3/\varepsilon)}$, the bias of $\hat{C}_m^{S}(Z^{n})\cdot \log (1/\varepsilon)$ in estimating $\tilde{S}^{s}(p)$ satisfies
\begin{align*}
|\EE[\hat{C}_m^{S}(Z^{n})]\cdot \log (1/\varepsilon)-\tilde{S}^{s}(p)|
&\le |\EE[\hat{C}_m^{S}(Z^{n})]-\tilde{C}_{m}^{s}(p)|\cdot \log(1/\varepsilon)+|\tilde{C}_{m}^{s}(p)\cdot \log(1/\varepsilon)-\tilde{S}^{s}(p)|\\
&\le \Paren{\frac{3n^\alpha}{m}+ \frac{\varepsilon}{3}\cdot\frac{\min\{m,k\}}{m}} \cdot \log(1/\varepsilon)+\varepsilon\le \frac{n^\alpha}{n}\log n+\frac{4\varepsilon}{3}.
\end{align*}

In addition, changing one sample point in $Z^n$ modifies $\hat{C}_m^{S}(Z^{n})\cdot \log (1/\varepsilon)$ by at most $(8n^\alpha \log n)/n$. Hence by Lemma~\ref{lem:sensitivitybound}, for any $\tau\geq 0$, 
\[
 \Pr_{Z^n\sim p}\Paren{\Abs{\hat{C}_m^{S}(Z^n)\cdot \log (1/\varepsilon)-\EE[\hat{C}_m^{S}(Z^n)\cdot \log (1/\varepsilon)]}\geq \tau}\leq 2\exp(-2\tau^2\cdot \Omega(n^{1-2\alpha}/\log^2 n)).
\] 
Applying Lemma~\ref{lem:tpmlcomp} and letting $t=(c_1+c_2)\log (n/2)$, we obtain
\[
\Pr\Paren{|\tilde{C}_m^{s}(p)-\tilde{C}_m^{s}(p_{\varphi^t})|\ge 2\tau+\frac{8\varepsilon}{3}+\frac{2n^\alpha}{n}\log n}\le 2\exp(-2\tau^2\cdot \Omega(n^{1-2\alpha}/\log^2 n))\cdot en^t,
\]
where the TPML estimate $p_{\varphi^t}$ is computed over $\cP = \Delta_{\ge 1/k}$. 
For any $\tau=\Omega((\log n)^2 /n^{1/2-\alpha})$,
the right-hand side vanishes as fast as $2\exp(-\log^2 n)$. It remains to construct an estimator for 
\[
\tilde{S}^{\ell}(p)=\tilde{S}(p)-\tilde{S}^{s}(p)=\frac{1}{k}\cdot \sum_{x} \Pr_{Z^{n/2}\sim p}\Paren{\mu_x(Z^{\frac n2})> c_1\log \frac n2}
.
\]
A natural choice is the unbiased estimator 
\[
\hat{S}^{\ell}(Z^n):=\frac{1}{k}\cdot \sum_{x} \indic_{\mu_x(Z^{\frac n2})> c_1\log \frac n2}
.
\]
The $n$-sensitivity of this estimator is exactly $1/k\le (\log n)/n$. 
Hence by the McDiarmid's inequality, with probability at least $1-2\exp(-2\log^2 n)$, 
\[
|\hat{S}^{\ell}(Z^n)-\tilde{S}^{\ell}(p)|\le \frac{\log^2 n}{\sqrt{n}}.
\]
Consolidating the previous results yields
\begin{Theorem}\label{thm:supp}
Support-size estimator $\tilde{S}^{s}(p_{\varphi^t})+\hat{S}^{\ell}$
is sample-optimal for $\varepsilon= \tilde{\Omega}(1/n^{1/2-\alpha})$.
\end{Theorem}
The estimator's optimality follows from $k\log \frac 1\varepsilon \le \alpha\frac{ n\log (n/2^{1+1/\alpha})}{2\log (3/\varepsilon)}$, which matches\vspace{-0.25em} with the tight~\citep{W19} lower bound $n=\Omega\Paren{\frac{k}{\log k}\log^2 \frac 1\varepsilon }$. Note that\vspace{-0.1em} we compute the TPML estimate $p_{\varphi^t}$ over $\cP=\Delta_{\ge 1/k}$, where $t=\Theta(\log n)$. As in the case of support-coverage estimation, the range of $\varepsilon$ where the estimator is sample-optimal is again near-optimal (e.g., set $\alpha=0.01$) and better than the $\varepsilon= \Omega(n^{-0.2})$ range established in~\cite{mmcover} for the PML plug-in estimator. 

\subsection{TPML and distribution estimation}\label{sec:TPMLdist}
This section revisits distribution estimation. Write $\max\{a, b\}$ as $a\lor b$. For any $\tau\in[0,1]$, the \emph{$\tau$-truncated relative earth-mover distance}~\citep{instdist}, between $p$ and
$q$ is
\[
R_\tau(p,q)
:=
\inf_{\gamma\in\Gamma_{p,q}}\;\Exp_{(X,Y)\sim\gamma}
\Abs{\log\frac{p(X)\lor \tau}{q(Y)\lor \tau}}.
\]

Define $\alpha_n:=n^{.03}+n^{.01}$, $\beta_n:=n^{.03}+2n^{.01}$,  and $\gamma_n:=\alpha_n/ n$. 
The distribution estimator $\hat{p}_{\text{\scalebox{.7}{TPML}}}$ shown in Figure~\ref{fig:111} is a simple combination of the TPML and empirical estimators, and satisfies
\begin{Theorem}\label{TPMLdist}
For any discrete distribution $p$, draw a sample $X^n\sim p$ and denote its profile by $\varphi$. Then, with probability at least $1-\exp(-n^{\Omega(1)})$ and for any $w\in[1,\log n]$,
\[
R_{\frac{w}{n\log n}}(\hat{p}_{\text{\scalebox{.7}{TPML}}}(X^n),p) = \mathcal{O}\Paren{\frac{1}{\sqrt{w}}}.
\]
\end{Theorem}

\begin{figure}[ht]
\begin{center}
\boxed{
\begin{aligned}
& \text{ Compute } p_{\varphi_{\alpha_n}} \text{and replace the entries} >\beta_n \text{ by those of } p_\mu \\
& \text{ While the total value} <1:\text{append an entry of value } \gamma_n \\
& \text{ While the total value} >1:\text{remove the largest entry}\le \gamma_n\\
& \text{ Append one entry to make the total value}=1
\end{aligned}
}
\end{center}\vspace{-0.3em}
\caption{Distribution estimator $\hat{p}_{\text{\protect\scalebox{.7}{TPML}}}$}
\label{fig:111}
\end{figure}

\subsubsection*{Comparisons and implications}
The estimator's guarantee stated in Theorem~\ref{TPMLdist} is essentially the same as that presented in~\citep{instdist}. The algorithms are different as our estimator is based on TPML, while the estimator in~\citep{instdist} mainly relies on a linear program. Unlike the latter, our approach additionally has the following desired attribute. For numerous symmetric properties, the single TPML estimator yields estimators that are sample-optimal over nearly all ranges of accuracy parameters. On the other hand, even just for support coverage, the method in~\citep{instdist} is known to offer sample-optimal estimators only when the desired accuracy is a constant. 

Theorem~\ref{TPMLdist} provides an estimation guarantee stronger than those appear in~\citep{VV11,ventro}, since the latter results degrade as the alphabet size increases. It is also of interest to derive a result similar to Theorem~\ref{thm:dist}, which shows that both PML and APML are sample-optimal for learning sorted distributions. For any $\tau\in[0,1]$, define the \emph{$\tau$-truncated sorted $\ell_1$ distance} between two distributions~$p,q\in\Delta_\cX$ as
\[
\ell_{\tau}^{\text{\tiny{<}}}(p,q):=\min_{p'\in\Delta_\cX: \{p'\}=\{p\}}\sum_x \Abs{p'(x)\lor \tau- q(x)\lor \tau}.
\]
By Fact 1 in~\citep{instdist}, given distributions $p, q\in\Delta_\cX$,
 $
 \ell_{\tau}^{\text{\tiny{<}}}(p,q)\le 2R_\tau(p, q),
 $
implying 
\begin{Corollary}
Under the same conditions as Theorem~\ref{TPMLdist},
\[
\ell_{\tau}^{\text{\tiny{<}}}(\hat{p}_{\text{\scalebox{.7}{TPML}}}(X^n),p) = \mathcal{O}\Paren{\frac{1}{\sqrt{w}}}.
\]
\end{Corollary}
\subsection{Proof of Theorem~\ref{TPMLdist}}
The proof essentially follows the proof of Theorem~2 in~\citep{instdist}. 
The original reasoning is not sufficient for our purpose as the error probability derived 
is too large to invoke the competitiveness of TPML.  
To~address this issue, we slightly modify the linear program used in the paper, carefully separate the analysis 
of the estimators for large and small probabilities, and provide a finer analysis with tighter probability bounds by reducing the use of the union bound.  
To proceed, we first define histograms and the relative earth-moving cost, and give an operational meaning to $R_\tau$. 

For a distribution $p$, the \emph{histogram} of a multiset $\mathcal{A}\subseteq \{p\}$ is a mapping, denoted by $h_{{}_{\mathcal{A}}}:\! (0,1]\rightarrow \mathbb{Z}_{\ge 0}$, that maps each number $y\in (0,1]$ to the number of times it appears in $\mathcal{A}$. Note that every $y$ corresponds to a probability mass of $y\cdot h_{{}_{\mathcal{A}}}(y)$. More generally, we also allow \emph{generalized histograms} $h$ with non-integral values $h(y)\in \mathbb{R}_{\ge 0}$. For any $y_1,y_2\in (0,1]$, generalized histogram $h$, and nonnegative $m<y_1\cdot h(y_1)$, we can \emph{move} a probability mass from location $y_1$ to $y_2$ by reassigning $h(y_1)-m/y_1$ to $y_1$, and $h(y_2)+m/y_2$ to $y_2$. 
Given $\tau\in[0,1]$, 
we define the \emph{cost} associated with this operation as 
\[
c_{\tau, m}(y_1, y_2):=m\cdot \Abs{\log \frac{y_1\lor \tau}{y_2\lor \tau}},
\]
and term it as \emph{$\tau$-truncated earth-moving cost}. The cost of multiple operations is additive. 
Under such formulation, $R_\tau(p,q)$ is exactly the minimal total $\tau$-truncated earth-moving cost associated with any operation schemes of moving $h_{\{p\}}$ to yield $h_{\{q\}}$. One can verify that $c_{\tau, m}(y_1, y_2)=c_{\tau, m}(y_2, y_1)$ and $R_\tau(p,q)=R_\tau(q,p)$, for any $y_1,y_2\in (0,1]$ and $p,q\in\Delta_\cX$, respectively.

For notational convenience, denote the binomial- and Poisson-type probabilities by $\Bin(n, x, i):=\binom{n}{i}x^i (1-x)^{n-i}$ and $\Poi(mu, j):=e^{-\mu} \frac{\mu^j}{j!}$, and suppress $X^n$ in $\varphi_i(X^n)$ and $\mu_s(X^n)$, 

For any absolute constants $B$ and $C$ satisfying $0.1>B>C>\frac B2>0$, define $x_n:=\frac{n^{B}+n^{C}}{n}$ and $S := \{\frac{1}{n^2}, \frac{2}{n^2}, \ldots, x_n\}$.
Consider the following linear program. 

\begin{figure}[h]
\begin{center}
\boxed{
\begin{aligned}
&\text{For each } x\in S, \text{ define the associated variable } v_x\\
& \text{Minimize }\sum_{i=1}^{n^B}\Abs{\varphi_{i} - \sum_{x\in S} \Bin(n, x, i)\cdot v_x}\\
&\text{s.t. } \sum_{x\in S} x\cdot v_x = \sum_{i\le n^B+2n^C} \frac{i}{n}\cdot \varphi_{i}\\ 
&\text{and }\forall x\in S, v_x\geq 0
\end{aligned}
}
\end{center}\vspace{-0.5em}
\caption{Linear program (LP)}
\label{fig:1111}
\end{figure}

\subsubsection*{Existence of a good feasible point}
 
Let $p$ be the underlying distribution and $h$ be its histogram. First we show that with \emph{high} probability, the linear program LP has a feasible point $(v_x)$ that is \emph{good} in the following sense: 1)  the corresponding objective value is relatively \emph{small}; 2) for $\tau\geq n^{-3/2}$, the generalized histogram $h_0: x\rightarrow v_x$ is \emph{close} to $h_n: y \rightarrow h(y)\cdot \indic_{y\leq x_n}$ under the $\tau$-truncated earth-mover cost.  

For each $y\le x_n$ satisfying $h(y)>0$, find $x = \min\{x' \in S : x' \ge y\}$ and set $v_{x}=h(y)\cdot \frac{y}{x}$. 

Denote $\mathcal{F}:=\sum_{i\le n^B+2n^C} \frac{i}{n}\cdot \varphi_{i}$. By construction, 
\begin{align*}
T_n(h):= \sum_{y:\ y\le x_n, h(y)>0} h(y)\cdot y =\sum_x x\cdot v_x.
\end{align*}
By the Chernoff bound, the expectation of estimator $\mathcal{F}$ satisfies
\begin{align*}
\EE[\mathcal{F}]
 = \sum_{i\le n^B+2n^C} \frac{i}{n}\cdot \EE[\varphi_{i}]
 \ge T_n(h) -\exp(-\Omega(n^{2C-B})).
\end{align*}
Since changing one observation changes the estimator's value by at most $n^{-1}$, we bound its tail probability using the McDiamid's inequality, 
\begin{align*}
\Pr(|\mathcal{F}-\EE[\mathcal{F}]|> n^{-0.4} )&\leq 2\exp(-2 n^{0.2} ).
\end{align*}
Henceforth we assume $|\mathcal{F}-\EE[\mathcal{F}]|\le  n^{-0.4}$, which holds with probability at least $1-2\exp(-2 n^{0.2} )$. To ensure that $(v_x)$ is a feasible point of the linear program LP, we may need to modify its entries. 

For $y\in(0,1]$, let $f_{i}(y) := \frac{\Bin(n, y, i)}{y}$. For $i\geq 1$, we can verify that $|f_{i}(y)|\leq n$ and $|f_{i}'(y)|\leq n^2$.

Without any modifications, for $i\leq n^B$, the difference between $\EE[\varphi_i]=\sum_{y: h(y)>0} \Bin(n, y, i)\cdot h(y)$ and $\sum_{x\in S} \Bin(n, x, i)\cdot v_x$ is at most $n^{-2} \cdot \sup_{y\in[0,1]} |f_{i}'(y)| +n\exp(-\Omega(n^{2C-B}))=\mathcal{O}(1)$. Furthermore, by the McDiarmid's inequality, 
\[
\Pr(|\varphi_i-\EE[\varphi_i]|\ge n^{0.6})\le 2\exp(-2n^{0.2}).
\]
Define $m=\mathcal{F}(X^n)- \sum_x x\cdot v_x$ and consider two cases. If $m>0$, we choose $x=x_n$ and increase $v_x$ by $m/x$. For any $i$ satisfying $1\le i \le n^B$, this modifies the value of $\sum_{x\in S} \Bin(n, x, i)\cdot v_x$ by at most $\Bin(n, x_n, n^B)\cdot x_n^{-1}\leq \exp(-\Omega(n^{2C-B}))$. 

By the assumption that $|\mathcal{F}-\EE[\mathcal{F}]|\le  n^{-0.4}$, 
\[
\mathcal{F}
\ge \sum_x x\cdot v_x-n^{-0.4}-\exp(-\Omega(n^{2C-B}))
\ge \sum_x x\cdot v_x-\mathcal{O}(n^{-0.4}).\vspace{-0.25em}
\]
If $m<0$, we remove a total probability mass of at most $\mathcal{O}(n^{-0.4})$ by decreasing the entries of $(v_x)$.
Since $|f_{i}(y)|\leq n$, this operation modifies the value of $\sum_{x\in S} \Bin(n, x, i)\cdot v_x$ by at most $\mathcal{O}(n^{0.6})$. 

By the union bound, with probability at least $1-\exp(-n^{0.2})$, the objective value of the feasible point $(v_x)$ is at most $n^B\cdot \mathcal{O}(n^{0.6}+1)=\mathcal{O}(n^{B+0.6})$. 

Finally, for any $\tau\geq n^{-3/2}$, the minimal $\tau$-truncated earth-moving cost of 
moving the generalized histogram $h_0$ corresponding to $(v_x)$, and the 
histogram $h_n: y \rightarrow h(y)\cdot \indic_{y\leq x_n}$,  
so that they differ from each other only at $x=x_n$, 
is at most 
\[
\log \left(\frac{n^{-3/2}+n^{-2}}{n^{-3/2}}\right) + \mathcal{O}\Paren{\frac{\log n}{n^{0.4}}} =\mathcal{O}(n^{-0.3}).
\]

\subsubsection*{All solutions are good solutions}

Let $(v_x)$ be the solution described above. We show that for any solution $(v_x')$ to LP whose objective value is $\mathcal{O}(n^{B+0.6})$, the generalized histogram $h_1$ corresponding to 
$(v_x')$ is close to $h_0$. 

Consider the earth-moving scheme described in~\citep{instdist} that moves all the probability mass to a sequence $\{c_i\}$ of center points satisfying $c_i= \Omega(1/(n\log n))$. We apply this scheme to $h_0$ and $h_1$ with the following modification: For any probability mass that should be moved to a center $c_i$ with $c_i>x_n$ under the original earth-moving scheme, we move it to $x_n$. Since $x_n = \max S$, this modification only reduces the cost of the scheme. By Proposition 5 in~\citep{instdist}, for any $w\in[1,\log n]$ and $\tau = \frac w{n\log n}$,  the corresponding $\tau$-truncated earth-moving cost is at most $\mathcal{O}(1/{\sqrt{w}})$ . 

We first consider $h_0$. After applying the modified earth-moving scheme, the probability mass at each center $c_i<x_n$ is 
$\sum_{j\ge 0} \alpha_{i, j}\sum_{x\in S} \Poi(nx, j) x v_x$ for some set of coefficients $\{\alpha_{i, j}\}$ satisfying:  $\sum_{j\ge 0} |\alpha_{i, j}|\le 2n^{0.3}$ for all $i$; $\alpha_{i, j}=0$ for $i\le 0.2\log n\le j/2$; and $\alpha_{i, j}=\indic_{i-1=j}$ for $i> 0.2\log n$. As for $h_1$, the probability mass at each center $c_i<x_n$ is $\sum_{j\ge 0} \alpha_{i, j}\sum_{x\in S} \Poi(nx, j) x v_x'$, which differs from that of $h_0$ by 
\begin{align*}
\Abs{\sum_{j\ge 0} \alpha_{i, j}\sum_{x\in S} \Poi(nx, j) x (v_x'-v_x)}
&\le \sum_{j \ge 1 } \alpha_{i, j-1}\frac{j}{n}\Abs{\sum_{x\in S} \Poi(nx, j) (v_x'-v_x)}
.
\end{align*}
By our assumption on the corresponding objective values of LP, for any positive integer $i\le n^B$, 
\[
\Abs{\varphi_{i} - \sum_{x\in S} \Bin(n, x, i)\cdot v_x}\text{\large$\lor$} \Abs{\varphi_{i} - \sum_{x\in S} \Bin(n, x, i)\cdot v_x'} = \mathcal{O}(n^{B+0.6}),
\]
which, together with the inequality $|\Poi(nx, j)-\Bin(n,x, j)|\le 2x$ from~\citep{BH84}, implies that
\begin{align*}
 \sum_{j \ge 1 } \alpha_{i, j-1}\frac{j}{n}\Abs{\sum_{x\in S} \Poi(nx, j) (v_x'-v_x)}
 &\le \sum_{j \ge 1 } \alpha_{i, j-1}\frac{j}{n}\Paren{\Abs{\sum_{x\in S} \Bin(n,x, j) (v_x'-v_x)}+\Abs{\sum_{x\in S} 2x |v_x'-v_x|}}\\
& \le 2n^{0.3}\cdot \frac{n^B}{n} \cdot \Paren{\mathcal{O}(n^{B+0.6})+4}\\
& = \mathcal{O}(n^{2B-0.1}),
\end{align*}
where $n$ is assumed to be sufficiently large to yield $n^B>0.4\log n$. 

Therefore, for $\tau\ge 1/(n\log n)$, the minimal $\tau$-truncated earth-moving cost of moving $h_0$ and $h_1$ so that they differ only at $x_n$, is at most 
\[
n^B\cdot \mathcal{O}(n^{2B-0.1})\cdot 2\log n + 
\log \Paren{\frac{n^B+n^C}{n^B}}
=\mathcal{O}(n^{3B-0.1}\log n+n^{C-B}).
\] 
The right-hand side is at most $\mathcal{O}(1/\sqrt{\log n})$ for $B=0.03$ and $C=0.02$. We consolidate the previous results.  For $w\in[1, \log n]$ and $\tau= w/(n\log n)$, with probability at least $1-\exp(-\Omega(n^{0.2}))$, the solution to LP will yield a generalized histogram $h_1$, such that the minimal $\tau$-truncated earth-moving cost of moving $h_1$ and $h_n$ so that they differ only at $x_n$, is $\mathcal{O}(1/\sqrt{w})$.

\subsubsection*{Competitiveness of TPML}

The linear program LP estimates small probabilities and takes 
as input the $(n^B+2n^C)$-truncated profile of a given sample. 
For the TPML distribution associated with this truncated profile, denote by $h_2$ the histogram corresponding to its entries that are at most $x_n$.

Since $n^B+2n^C\leq 3n^B$, the number of such truncated 
profiles is bounded from above by $en^{3n^B}$. 
Utilizing the same rationale as in Section~\ref{sec:TPML}, for any $w\in[1, \log n]$ and $\tau= w/(n\log n)$, with probability at least $1-en^{3n^B}\cdot \exp(-\Omega(n^{0.2}))$, the minimal $\tau$-truncated earth-moving cost of moving $h_2$ and $h_n$ so that they differ only at $x_n$, is $\mathcal{O}(1/\sqrt{w})$. 
Note that the error probability bound $en^{3n^B}\cdot \exp(-\Omega(n^{0.2}))=\exp(-\Omega(n^{0.2}))$ vanishes quickly as $n$ increases.

\subsubsection*{Appending empirical estimates to TPML} 

Below we show that if we modify the TPML estimate $h_2$ properly and append the empirical probabilities of the frequent symbols, the resulting histogram is an accurate estimate of the actual histogram~$h$.

Assume that $h_2$ satisfies the conditions described in the last paragraph. Further assume that $\mathcal{F} \ge T_n(h)-\mathcal{O}(n^{-0.4})$, which also holds with probability at least $1-\exp(-\Omega(n^{0.2}))$. 
As in the case of $(v_x)$, we modify $h_2$ so that its total probability mass is exactly $\mathcal{F}$. 
If $T_n(h_2)<\mathcal{F}$, we increase $h_2(x_n)$ by ${(\mathcal{F} -T_n(h_2))}/{x_n}$; otherwise, we greedily decrease the values of $h_2(y)$'s while maintaining their non-negativity, starting from $y\le x_n$ closer to $x_n$. After this modification, there will be at most one location $y\le x_n$ satisfying $h_2(y)\not\in\mathbb{Z}$. If such a $y$ exists, decrease $h_2(y)$ by $h_2(y)-\lfloor h_2(y) \rfloor$ and move the corresponding probability mass to location $h_2(y)-\lfloor h_2(y) \rfloor$. The $1/(n\log n)$-truncated earth-moving cost of this step is at most $x_n\cdot \log n=\mathcal{O}(1/\log n)$. 
Let $h_2'$ be the resulting \emph{histogram}. 

By the previous analysis, for any $w\in[1, \log n]$ and $\tau= w/(n\log n)$, 
there is an earth-moving scheme on $h_2'$ having the following three properties: 
1)~the scheme moves no probability mass to a location $y>x_n$;
2) the $\tau$-truncated cost of the scheme is at most $\mathcal{O}(1/\sqrt{w})$;
3) the total discrepancy between the resulting generalized histogram and $h_n$ at all locations $y<x_n$ is at most $\mathcal{O}(n^{-0.4})$. 
For the case where $T_n(h_2)\ge \mathcal{F}$, we make use of the fact that $c_{\tau, m}(y_1, y_2)\le c_{\tau, m}(y_1, y_3)$ for any $m>0$ and $y_1\le y_2\le y_3\in(0,1]$.

For all $i> n^B+2n^C$, increase $h_2'(i/n)$ by $ \varphi_i$ and denote by $h_3$ the resulting generalized histogram, which has a total probability mass of $1$. By the Chernoff bound, for any symbol $s$,   
\[
\Pr\Paren{|n\cdot p(s)-\mu_s|\ge \mu_s^{3/4},\ \mu_{s}>n^B+2n^C}\leq 2 n p(s)\exp(-\Omega(n^{2C-B})),
\]
and
\[
\Pr\Paren{n\cdot p(s)\ge n^B+4n^C,\ \mu_{s}\le n^B+2n^C}\leq 2\exp(-\Omega(n^{2C-B})).
\]
Hence, we further assume that $|n\cdot p(s)-\mu_s|< \mu_s^{3/4}$ for all symbols $s$ appearing more than $n^B+2n^C$ times, and any symbol $s$ with probability $p(s)\ge (n^B+4n^C)/n$ appears more than 
$n^B+2n^C$ times. By the union bound, we will be correct with probability at least $1-4n\exp(-\Omega(n^{2C-B}))$. 
Under these assumptions, if for each symbol $s$ satisfying $\mu_s\ge n^B+2n^C$ times, we move a $\mu_s/n$ probability mass of $h_3$ from $\mu_s/n$ to $p(s)$, then at all locations $y\ge (n^B+4n^C)/n$, the total discrepancy between the resulting generalized histogram and the true histogram $h$ is at most $1/n$ multiplied by
\[
\sum_{j> n^B+2n^C}\varphi_{j} j^{\frac34}
=\sum_{j> n^B+2n^C}\varphi_{j}^{\frac14} (\varphi_{j}j)^{\frac34}
\le \Paren{\sum_{j> n^B+2n^C}\varphi_{j}}^{\frac14}\Paren{\sum_{j> n^B+2n^C} \varphi_{j}j}^{\frac34}
\le n^{1-\frac B4}
,
\] 
where the second last step follows from the H\"{o}lder's inequality. 
In addition, the associated total earth-moving cost is bounded from above by  
\[
\sum_{j> n^B+2n^C}\varphi_{j}\frac{j}{n} \log \Abs{\frac{j}{j\pm j^{\frac34}}}
\le \frac{1}{n} \sum_{j> n^B+2n^C}\varphi_{j}j^{\frac34}
\le n^{-\frac B4}
.
\] 
We consolidate the previous results. For any $w\in[1, \log n]$ and $\tau= w/(n\log n)$, 
there is an earth-moving scheme on $h_3$ having the following two properties: 1) the total $\tau$-truncated earth-moving cost of the scheme is $\mathcal{O}(1/\sqrt{w})$; 2) the 
total discrepancy between the resulting generalized histogram and $h$ at all locations $y\not\in I_n:= ((n^B+n^C)/n, (n^B+4n^C)/n)$ is at most $\mathcal{O}(n^{-0.4}+n^{-B/4})$.

Finally, note that the cost of moving a unit mass within $I_n$ is at most $3n^{B-C}$, implying that with probability at least $1-\exp(-n^{\Omega(1)})$, \vspace{-0.25em}
\[
R_\tau(h_3, h) =\mathcal{O}\Paren{\frac{1}{\sqrt{w}}}.
\]\par\vspace{-0.75em}

\section{Uniformity testing}\label{sec:uniproof}
\subsection{PML-based tester}
Let $\varepsilon$ be an arbitrary accuracy parameter and $\cX$ be a finite set. Let $p_u$ denote the uniform distribution over $\cX$. Given sample access to an unknown distribution $p\in\Delta_\cX$,  the uniformity testing distinguishes between the null hypothesis 
\[
H_0: p = p_u
\]
and the alternative hypothesis 
\[
H_1: \norm{p-p_u}_1\geq \varepsilon.
\] 
After a sequence of research works~\citep{GO00, bffkrw, Pan08, CompUni13, ChanD14, Val17, DKane15, Ach15, DK16, DGT18}, it is shown that to achieve a $k^{-\Theta(1)}$ bound on the error probability, this task requires a worst-case sample size of order $\sqrt{\size\log \size}/\varepsilon^2$. The uniformity tester $T_{\text{\tiny PML}}(X^n)$ in Figure~\ref{fig:11} is purely based on PML, and takes as input parameters $\size$~and~$\varepsilon$,  and a sample $X^n\sim p$. 

\begin{figure}[ht]
\begin{center}
\boxed{
\begin{aligned}
&\text{\bf Input:}\ \ \text{parameters }\size, \varepsilon, \text{and a sample }X^n\sim p \text{ with profile } \varphi.\\
&\text{1. If } {\max}_x \mu_x(X^n)\geq 3\max\{1, n/\size\}\log \size \text{, return } 1\text{;}\\
&\text{2. Elif } \norm{p_{\varphi}-p_u}_2\geq  3\varepsilon/ (4\sqrt{\size}) \text{, return }1\text{;}\\ 
&\text{3. Else return }0\text{.}
\end{aligned}
}
\end{center}\vspace{-0.3em}
\caption{Uniformity tester $T_{\text{\tiny PML}}$}
\label{fig:11}
\end{figure}

In the rest of this section, we establish the following theorem.
\setcounter{Theorem}{5} 
\begin{Theorem}
If $\varepsilon= \tilde{\Omega}(\size^{-1/4})$ and $n=\tilde{\Omega}(\sqrt{\size}/\varepsilon^2)$, then the tester $T_{\text{\tiny PML}}(X^n)$ will be correct with probability at least $1-\size^{-2}$. The tester also distinguishes between $p=p_u$ and $\norm{p-p_u}_2\geq \varepsilon/\sqrt{k}$. 
\end{Theorem}

\subsection{Proof of Theorem~\ref{thm:test}}
Assume that $\varepsilon\geq  (\log \size)/\size^{1/4}$. For a sample $X^n\sim p_u$, the multiplicity of each symbol $x$ follows a binomial distribution $\Bin(n, \size^{-1})$ with mean $n/\size$. The following lemma~\citep{C81} bounds the tail probability of a binomial random variable. 
\begin{Lemma}\label{lem:bin_con}
For a binomial random variable $Y$ with mean $M$ and any $t\geq 1$, 
\[
\Pr(Y\geq{(1+t)M})\leq{\exp(-t(2/t+2/3)^{-1} M )}.\vspace{-0.5em}
\]
\end{Lemma}
Applying the above lemma to $Y=\mu_x(X^n)$ and $t=3\max\{\size /n, 1\}\log \size$ immediately yields that $\Pr(\mu_x(X^n)\geq (1+t)n/\size)\leq \size^{-3}$. By symmetry and the union bound, we then have 
$
\Pr\Paren{\max_x \mu_x(X^n)\geq (1+t)n/\size}\leq \size^{-2}.
$
In the subsequent discussion, we denote by $\Phi^n_\cX$ the profile set $\{\varphi(x^n): x^n\in \cX^n\text{ and } \max_x \mu_x(x^n)<(1+t)n/\size\}$. 

Consider the problem of estimating the $\ell_2$-distance between an unknown distribution and the uniform distribution $p_u$, for which we have the following result~\citep{G17}. 
\begin{Lemma}\label{lem:tester}
There is a profile-based estimator $\hat{\ell}_2$ such that for any $\varepsilon_0 \leq \size^{-1/2}$, $n=\Omega(\size^{-1/2}/\varepsilon_0^2)$, $p\in \Delta_\cX$ satisfying $P_2(p)=\mathcal{O}(\size^{-1})$, and $X^n\sim p$,  \vspace{-0.25em}
\begin{itemize} 
\item if $\norm{p-p_u}_2>\varepsilon_0$, then $\hat{\ell}_2(X^n)\geq 0.9 \varepsilon_0$,  \vspace{-0.25em}
\item if $\norm{p-p_u}_2<\varepsilon_0/2$, then $\hat{\ell}_2(X^n)\leq 0.6 \varepsilon_0$,
\end{itemize}
with probability at least $2/3$. 
\end{Lemma}

Set $\varepsilon_0=\varepsilon/\sqrt{\size}$ in the above lemma. Then, by the sufficiency of profiles and the standard median trick, there exists another profile-based estimator $\hat{\ell}_2^\star$ that under the same conditions, provides the estimation guarantees stated above, with probability at least $1-\delta$ for $\delta:=2\exp(-\Omega(n\varepsilon^2/\sqrt{\size}))$. Scaling $\varepsilon_0$ by positive absolute constant factors yields: If $\norm{p-p_u}_2>0.67\varepsilon_0$, then $\hat{\ell}_2(X^n)\leq 0.6 \varepsilon_0$ with probability at most $\delta$; if $\norm{p-p_u}_2<0.75\varepsilon_0$, then $\hat{\ell}_2(X^n)\geq 0.9 \varepsilon_0$ with probability at most $\delta$.  

Let $\varphi'$ be a profile.  If we further have $p(\varphi')>\delta$, then by definition, $p_{\varphi'}(\varphi')\geq p(\varphi')>\delta$. 
Hence for any $x^n$ with profile $\varphi'$, if $\norm{p-p_u}_2>\varepsilon_0$, we must have both $\hat{\ell}_2(x^n)\geq 0.9 \varepsilon_0$ and $\norm{p_{\varphi'}-p_u}_2\geq 0.75\varepsilon_0$; if $\norm{p-p_u}_2<\varepsilon_0/2$, we must have both $\hat{\ell}_2(x^n)\leq 0.6 \varepsilon_0$ and $\norm{p_{\varphi'}-p_u}_2\leq 0.67\varepsilon_0$. 

On the other hand, for a sample $X^n\sim p$ with profile $\varphi$, the probability that we have both $p(\varphi)\leq \delta$ and $\varphi\in \Phi^n_\cX$ is at most $\delta$ 
times the cardinality of the set $\Phi^n_\cX$. By definition, if $\varphi\in \Phi^n_\cX$, then $\varphi_i =0$ for $i\geq  (1+t)n/\size$. In addition, each $\varphi_i$ can only take values in $\lceil \size\rfloor=\{0,1,\ldots, \size\}$, implying that $|\Phi^n_\cX|\leq |\lceil \size \rfloor|^{(1+t)n/\size}\leq \exp(6\max\{n/\size, 1\}\log^2 \size)$. Therefore, we obtain the following upper bound on the probability of interest: $\delta\cdot |\Phi^n_\cX|\leq 2\exp(-\Omega(n\varepsilon^2/\sqrt{\size})+6\max\{n/\size, 1\}\log^2 \size)$. In~order to make the probability bound vanish, we need to consider two cases: $n\leq \size$ and $n> \size$. If $n\leq \size$, it suffices to have $n\gg (\log^2 \size)\sqrt{\size}/\varepsilon^2$; If $n> \size$, it suffices to have $\varepsilon\gg (\log \size)/\size^{1/4}$. In both cases, the probability bound is at most $\exp(-\log^2 \size)$.

Next, consider estimating the power sum $P_2(p)$, which is at least $\size^{-1/2}$ for $p\in \Delta_\cX$. By Corollary~\ref{cor:2}, there is a profile-based estimator $\hat{P}_2^\star$ such that 
$
\Pr_{ X^n\sim p}(|\hat{P}_2^\star(X^n)-P_2(p)|\geq (\varepsilon/8) \cdot P_2(p))\leq 2\exp(-\Omega(n\varepsilon^2/\sqrt{\size})) =\delta.
$
Following the same derivations as above and in Section~\ref{sec:thm4proof} with $\Phi_{\alpha, \varepsilon}^n(p)$ replaced by $\Phi^n_\cX$, we establish that
\[
\Pr\Paren{|P_2(p_{\varphi})-P_2(p)|> P_2(p)/2\text{ and } \varphi\in \Phi^n_\cX}\leq \delta\cdot |\Phi^n_\cX|\leq \exp(-\log^2 \size).
\]
Now we are ready to characterize the performance of the tester $T_{\text{\tiny PML}}(X^n)$. For clarity, we divide our analysis into two parts based on which hypothesis is true. 
\begin{itemize}
\item {\bf Case 1:} The null hypothesis $H_0$ is true, i.e., $p=p_u$. 
\begin{itemize}
\item {\bf Step 1:} By Lemma~\ref{lem:bin_con} and its implications, given $p=p_u$, the probability of failure at this step is at most $\Pr_{X^n\sim p_u}(\exists x\in \cX\text{ s.t. }\mu_x(X^n)\geq (1+t)n/\size)\leq \size^{-2}$. 
\item {\bf Step 2:} Note that $P_2(p)=\size^{-1}$ and $\norm{p-p_u}_2=0$, and recall that $\varphi=\varphi(X^n)$. The tester accepts $H_1$ in this step iff $\varphi \in \Phi^n_\cX$ and $\norm{p_{\varphi}-p_u}_2\geq 0.75\varepsilon_0$. By Lemma~\ref{lem:tester} and the subsequent arguments, this happens with probability at most $\exp(-\log^2 \size)$. 
\item {\bf Step 3:} The tester always accepts $H_0$ in this step. Hence by the union bound, if the null hypothesis $H_0$ is true, then the tester succeeds with probability at least $1-\size^{-2}$.
\end{itemize}
\item {\bf Case 2:} The alternative hypothesis $H_1$ is true, i.e., $\norm{p-p_u}_1\ge \varepsilon$. 
\begin{itemize}
\item {\bf Step 1 to 2:} The tester accepts $H_1$ if the conditions in either Step 1 or Step 2 are satisfied, and hence incurs no error.
\item {\bf Step 3:} According to the value of $P_2(p)$, we further divide our analysis into two parts:
\begin{itemize}
\item If $P_2(p)\geq 10\size^{-1}$, then $\norm{p_{\varphi}-p_u}_2< 0.75 \varepsilon/\sqrt{\size}$ implies that $P_2(p_{\varphi})<1.6\size^{-1}$ and $|P_2(p_{\varphi})-P_2(p)|> P_2(p)/2$. Hence, the tester accepts $H_0$ only if 
both $|P_2(p_{\varphi})-P_2(p)|> P_2(p)/2$ and $\varphi \in \Phi^n_\cX$ happen, whose probability, 
by the above disscusion, is at most  $\exp(-\log^2 \size)$. 
\item If $P_2(p)< 10\size^{-1}$, then all the conditions in Lemma~\ref{lem:tester} are satisfied. 
In addition, by the Cauchy-Schwarz inequality, we have $\norm{p-p_u}_2
\ge  \norm{p-p_u}_1\cdot \size^{-1/2}\ge \varepsilon\cdot \size^{-1/2}$. 
The tester accepts $H_0$ iff both $\norm{p_{\varphi}-p_u}_2< 0.75 \varepsilon\cdot \size^{-1/2}$ 
and $\varphi\in \Phi^n_\cX$ hold, which happen, by Lemma~\ref{lem:tester} and the subsequent arguments, 
with probability at most $\exp(-\log^2 \size)$. 
\end{itemize}
\end{itemize}
\end{itemize}
This completes the proof of the theorem.

\section{Conclusion and future directions}\label{sec:conclusion}\vspace{-0.25em}
We studied three fundamental problems in statistical learning: distribution estimation, property estimation, and property testing. 
We established the profile maximum likelihood (PML) as the first universally sample-optimal approach for several important learning tasks: distribution estimation under the sorted $\ell_1$ distance, additive property and R\'enyi entropy estimation, and identity testing. 
We~proposed~the truncated PML (TPML) and showed that simply combining the TPML and empirical estimates yields estimators for 
distributions and their properties enjoying even stronger guarantees.  

Several future directions are promising. We believe that neither the factor of $4$ in the sample size in Theorem~\ref{thm:est}, nor the lower bounds on $\varepsilon$ in Theorem~\ref{thm:est},~\ref{thm:dist}, and~\ref{thm:test} are necessary. In other words, the original PML approach is universally sample-optimal for these tasks in all ranges of parameters. It~is~also of interest to extend the PML's optimality to estimating symmetric properties not covered by Theorem~\ref{thm:est} to~\ref{thm:renyi3}, such as \emph{generalized distance to uniformity}~\citep{B17, H18a}, the $\ell_1$ distance between the unknown distribution and the closest uniform distribution over an arbitrary subset of $\cX$.
Besides the competitiveness we established for the PML-type estimators under the min-max estimation framework, ~\citet{H18} and~\citet{H19} recently proposed and studied a different formulation of competitive property estimation that aims to emulate the instance-by-instance performance of the widely used empirical plug-in estimator, using a smaller sample size. It is also meaningful to investigate the performance of  PML-based techniques through this new formulation. 

\appendix
\section{Proof of Lemma~\ref{lem:bound}}\label{sec:lemproof}
The proof closely follows that of Proposition 6.19 in~\cite{VDoc12} (page 131--136), which we refer to as \emph{the proposition's proof}. Note that in the work~\cite{VDoc12}, the definitions of $k$
 and $n$ are swapped, i.e., $k$ stands for the sample size, and $n$ denotes the alphabet size. For consistency, we still keep our 
 notation. 

Recall that we set $t_n:=2n^{-c_1}\log n$ and $\alpha\in(0, 1)$, and define 
 \[
\beta_i:=(1-e^{-t_n\alpha i})f\Paren{\frac{(i+1)\alpha}{n}}\frac{n}{(i+1)\alpha}+\sum_{\ell=0}^i z_\ell (1-t_n)^\ell\alpha^\ell (1-\alpha)^{i-\ell}\binom{i}{\ell}.
\]
for any $i\leq n$, and $\beta_i:=\beta_n$ for $i>n$. Let $w(i)$ denote the first quantity on the right-hand side, and $w:=(w(0),w(1),\ldots)$ be the corresponding vector. Similarly, let $\tilde{z}_\alpha(i)$ denote the second quantity on the right-hand side, and $\tilde{z}_\alpha$ be the corresponding vector. Assume that $v\leq \log^2 n$.

First part of the proposition's proof remains unchanged, which corresponds to the content from page~131 to the second last paragraph on page 132, showing that
\[
\sqrt{\alpha} \norm{\tilde{z}_\alpha}_2=\mathcal{O}(n^{\alpha c_2+(1-\alpha) c_1}\cdot\log^3 n).
\]
The assumption that $\alpha\in[1/100, 1)$ implies $\sqrt{\alpha}\geq 1/10$, and hence we have $|\tilde{z}_\alpha(i)|\leq \norm{\tilde{z}_\alpha}_2=\mathcal{O}(n^{\alpha c_2+(1-\alpha) c_1}\cdot\log^3 n)$. Recall that for lemma~\ref{lem:1} to hold, the coefficients $\beta_i$ must satisfy the following two conditions,
\begin{enumerate}
\item $|\varepsilon(y)|\leq a'+{b'}/{y}$,
\item $|\beta_{j}^\star-\beta_{\ell}^\star|\leq c' \sqrt{{j}/{n}}$ for any $j$ and $\ell$ such that $|j- \ell|\leq \sqrt{j}\log n$, 
\end{enumerate}
where $\varepsilon(y):={f(y)}/{y}-e^{-ny}\sum_{i\geq 0} \beta_i \cdot (ny)^i/i!$, and $\beta_i^\star:=\beta_{i-1}\cdot {i} /n, \forall i\geq 1$, and $\beta_0^\star:=0$.

We first consider the second condition and find a proper parameter $c'$. 

Our objective is to find $c'>0$ such that $c'>\sqrt{n/j}\ |\beta_{j}^\star-\beta_{\ell}^\star|$. By the triangle inequality, 
\begin{align*}
\sqrt{\frac{n}{j}}|\beta_{j}^\star-\beta_{\ell}^\star|
&\le \sqrt{\frac{n}{j}}\Abs{\frac{j}{n}\tilde{z}_\alpha(j-1)-\frac{\ell}{n}\tilde{z}_\alpha(\ell-1)}
+\sqrt{\frac{n}{j}}\Abs{\frac{j}{n}w(j-1)-\frac{\ell}{n}w(\ell-1)}
\end{align*}
We bound the two quantities on the right-hand side separately and consider two cases for each. If~both~$j$ and $\ell$ are at most $400n^{c_1}$, then 
\begin{align*}
\sqrt{\frac{n}{j}}\Abs{\frac{j}{n}\tilde{z}_\alpha(j-1)-\frac{\ell}{n}\tilde{z}_\alpha(\ell-1)}
\le \mathcal{O}(n^{c_1/2-1/2})\cdot \max_i |z_\alpha(i)|
\le  \mathcal{O} (n^{\alpha c_2+(3/2-\alpha) c_1-1/2}\log^3n).
\end{align*}
Recall that $|z_\ell|\leq v \cdot n^{c_2},\forall \ell\ge 0$. If one of $j$ and $\ell$ is larger than $400n^{c_1}$, say $j>400n^{c_1}$, then
\begin{align*}
\sqrt{\frac{n}{j}}\Abs{\frac{j}{n}\tilde{z}_\alpha(j-1)}
&\leq \sqrt{\frac{j}{n}} \sum_{\ell=0}^{j-1} |z_\ell| (1-t_n)^\ell\alpha^\ell (1-\alpha)^{j-1-\ell}\binom{j-1}{\ell}\\
&\leq \sqrt{j}n^{c_2-1/2}(\log^2 n)\sum_{\ell=0}^{j-1} (1-t_n)^\ell\alpha^\ell (1-\alpha)^{j-1-\ell}\binom{j-1}{\ell}\\
&=\sqrt{j}n^{c_2-1/2}(\log^2 n)(1-t_n\alpha)^{j-1}\\
&\leq\sqrt{j}n^{c_2-1/2}(\log^2 n)(1-\log n/(50n^{c_1}))^{400n^{c_1}}\\
&\leq\sqrt{j}n^{c_2-1/2}(\log^2 n)n^{-8}.
\end{align*}
For $j<2n^2$, the last quantity is at most $n^{-1}$. 
For $j>2n^2$, we have $\ell>n^2$ and hence
\begin{align*}
\sqrt{\frac{n}{j}}\Abs{\frac{j}{n}\tilde{z}_\alpha(j-1)-\frac{\ell}{n}\tilde{z}_\alpha(\ell-1)}
=\sqrt{\frac{n}{j}}\Abs{j-\ell}\tilde{z}_\alpha(n-1)
\le \sqrt{n}(\log n)n^{-1}
= (\log n)n^{-1/2}.
\end{align*}
Similarly, we can bound the other quantity, i.e.,
\begin{align*}
\sqrt{\frac{n}{j}}\Abs{\frac{j}{n}w(j-1)-\frac{\ell}{n}w(\ell-1)}
=\sqrt{\frac{n}{\alpha^2 j}}\Abs{(1-e^{-t_n\alpha (j-1)})f\Paren{\frac{j\alpha}{n}}
-(1-e^{-t_n\alpha (\ell-1)})f\Paren{\frac{\ell\alpha}{n}}}.
\end{align*}
Since $f$ (the property) is $1$-Lipschitz on $(\Delta_{\cX}, R)$ and $f(p)=0$ if $p(x)=1$ for some $x\in\cX$, one can verify that $|f(x)|\le x|\log x|\leq e^{-1}$ and $|f(x)/x-f(y)/y|\le |\log(x/y)|$ for $x,y\in[0,1]$ (the corresponding real function). We consider two cases and bound the quantity of interest. If $j\geq \sqrt{n}$, 
\begin{align*}
\sqrt{\frac{n}{\alpha^2 j}}\Abs{(1-e^{-t_n\alpha (j-1)})f\Paren{\frac{j\alpha}{n}}}
\le \sqrt{\frac{n}{\alpha^2 j}}\Abs{f\Paren{\frac{j\alpha}{n}}}
\le \sqrt{\frac{n}{\alpha^2 j}}\frac{j\alpha}{n}\log \Paren{\frac{j\alpha}{n}}
\le  \mathcal{O}(n^{-1/4}\log n).
\end{align*}
The same bound also applies to the other term where $j$ is replaced by $\ell$. If $j> \sqrt{n}$, then 
$e^{-t_n\alpha (j-1)}\leq \exp{(-2\alpha(\log n) n^{1/2-c_1})}=\mathcal{O}(n^{-2})$. 
Analogously, the same upper bound holds for the other term $e^{-t_n\alpha (\ell-1)}$. Hence, we ignore these two terms and consider only 
\begin{align*}
\sqrt{\frac{n}{\alpha^2 j}}\Abs{f\Paren{\frac{j\alpha}{n}}
-f\Paren{\frac{\ell\alpha}{n}}}
&\leq \sqrt{\frac{j}{n}} \Abs{\frac{n}{j\alpha} f\Paren{\frac{j\alpha}{n}}
-\frac{n}{\ell \alpha}  f\Paren{\frac{\ell\alpha}{n}}}+\sqrt{\frac{j}{n}}\Abs{\frac{n}{j\alpha}
-\frac{n}{\ell \alpha}}f\Paren{\frac{\ell\alpha}{n}}\\
&\leq \sqrt{\frac{j}{n}} \Abs{\log\frac{j}{\ell}}+\sqrt{\frac{j}{n}}\Abs{\frac{n}{j\alpha}
-\frac{n}{\ell \alpha}}  f\Paren{\frac{\ell\alpha}{n}}\\
&\leq \sqrt{\frac{j}{n}} \frac{\Abs{j-\ell}}{j}+\frac{\sqrt{jn}}{\alpha}\frac{|j-\ell|}{j\ell}  f\Paren{\frac{\ell\alpha}{n}}\\
&\leq \sqrt{\frac{j}{n}} \frac{\Abs{j-\ell}}{j}+\sqrt{\frac{j}{n}}\frac{|j-\ell|}{j} \Abs{\log\Paren{\frac{\ell\alpha}{n}}}\\
&\leq \frac{\log n}{\sqrt{n}}+\frac{\log n}{n}\Abs{\log\Paren{\frac{\ell\alpha}{n}}}\\
&= \mathcal{O}(n^{-1/2}\log n).
\end{align*}
By the assumption that $\alpha c_2+(3/2-\alpha) c_1\leq 1/4$, we have $ \mathcal{O} (n^{\alpha c_2+(3/2-\alpha) c_1-1/2}\log^3n)=\mathcal{O}(n^{-1/4}\log^3 n)$. Hence, we can set the latter quantity to be $c'$. The above derivations also show that 
 \[
 |w(i)|=\Abs{(1-e^{-t_n\alpha i})f\Paren{\frac{(i+1)\alpha}{n}}\frac{n}{(i+1)\alpha}}\leq \Abs{\log \Paren{\frac{(i+1)\alpha}{n}}}=\mathcal{O}(\log n).
\]
Together with $\beta_i = w(i)+\tilde{z}_\alpha(i)$ and $|\tilde{z}_\alpha(i)|=\mathcal{O}(n^{\alpha c_2+(1-\alpha) c_1}\cdot\log^3 n)$, this inequality implies 
\[
|\beta_i|\leq \mathcal{O}(n^{\alpha c_2+ (1-\alpha) c_1}\log^3 n).
\]

It remains to analyze the first condition of Lemma~\ref{lem:1} and find proper values for $a'$ and $b'$. For this part, the corresponding proof in~\citep{VDoc12} also holds for $\alpha\in [1/100,1/2]$ (page 134 to the second last paragraph on page 135), hence no change is needed. One thing to note is that $1/\alpha$ and $1/\sqrt{\alpha}$ are both $\mathcal{O}(1)$. For some $a'',b''\geq 0$ such that $a''+b''\size\leq v$, we can set $a'=a''+\mathcal{O}(n^{-c_1/2}\log^2 n)$ and $b'=b''(1+\mathcal{O}(n^{-c_1}\log n))$. The proof of Lemma~\ref{lem:bound} is complete.

\bibliographystyle{plainnat}
\bibliography{refs}

\vfill
\pagebreak
\end{document}